%% file: robust_01.tex
\documentclass[final,12pt]{colt2020}

\input{preamble.tex}
\title[Adversarially Robust Classification Calibration]{Corrigendum to: \\Calibrated Surrogate Losses for Adversarially Robust Classification}
\usepackage{times}

\coltauthor{%
 \Name{Han Bao}\thanks{This work was performed while the first author was a visitor at University of Michigan.} \Email{tsutsumi@ms.k.u-tokyo.ac.jp}\\
 \addr The University of Tokyo \\ RIKEN AIP
 \AND
 \Name{Clayton Scott} \Email{clayscot@umich.edu}\\
 \addr University of Michigan
 \AND
 \Name{Masashi Sugiyama} \Email{sugi@k.u-tokyo.ac.jp}\\
 \addr RIKEN AIP \\ The University of Tokyo
}

\begin{document}

\maketitle

\definecolor{cream}{rgb}{1.0, 0.99, 0.82}

\begin{mdframed}[
  innerrightmargin=20pt,
  innerbottommargin=10pt,
  innerleftmargin=20pt,
  backgroundcolor=cream,
  frametitle={Author's note},
  frametitleaboveskip=10pt,
]
  \small
  This document is a corrigendum to \citet{Bao:2020b}, which used the wrong definition of calibration from \citet{Steinwart:2007}. This corrigendum uses the correct definition.
  Fortunately, all of the main results from the published version still hold, with occasional minor changes.
  This corrigendum also removes an erroneous statement from the published version, connecting calibration to consistency. In addition, our simulation results are updated and corrected based on the correct definition of calibration, and we have also added a new positive result for convex losses under Massart’s noise condition.

  Because the error affected several aspects of the original paper, this corrigendum is a complete rewrite of the published version. We thank the authors of \citet{Awasthi:2021, Awasthi:2021b} for calling our attention to issues with our original simulations, which lead to this revision, and point the interested reader to these papers for several additional insights on calibration in the context of adversarially robust learning.
\end{mdframed}

\vspace{20pt}

\begin{abstract}
  Adversarially robust classification seeks a classifier
  that is insensitive to adversarial perturbations of test patterns.
  This problem is often formulated via a minimax objective,
  where the target loss is the worst-case value of the 0-1 loss
  subject to a bound on the size of perturbation.
  Recent work has proposed convex surrogates for the adversarial 0-1 loss,
  in an effort to make optimization more tractable.
  A primary question is that of consistency, that is, whether minimization of the surrogate risk implies minimization of the adversarial 0-1 risk.
  In this work, we analyze this question through the lens of calibration,
  which is a pointwise notion of consistency.
  We show that no convex surrogate loss is calibrated with respect to the adversarial 0-1 loss
  when restricted to the class of linear models.
  We further introduce a class of nonconvex losses
  and offer necessary and sufficient conditions for losses in this class to be calibrated.
  We also show that if the underlying distribution satisfies Massart's noise condition,
  convex losses can also be calibrated in the adversarial setting.
\end{abstract}

\begin{keywords}
  surrogate loss,
  calibration,
  adversarial robustness
\end{keywords}

\section{Introduction}
\label{sec:introduction}

In conventional machine learning, training and testing instances are assumed to follow the same probability distribution.
In \emph{adversarially robust} machine learning,
test instances may be perturbed by an adversary before being presented to the predictor.
Recent work has shown that seemingly insignificant adversarial perturbations can lead to significant performance degradations of otherwise highly accurate classifiers~\citep{Goodfellow:2015}.
This has led to the development of a number of methods for learning predictors
with decreased sensitivity to adversarial perturbations~\citep{Xu:2009,Xu:2012,Goodfellow:2015,Cisse:2017,Wong:2018,Raghunathan:2018a,Tsuzuku:2018}.

Adversarially robust classification is typically formulated
as empirical risk minimization with an \emph{adversarial 0-1 loss},
which is the maximum of the usual 0-1 loss
over a set of possible perturbations of the test instance.
This minimax optimization problem is nonconvex, and recent work, reviewed in Section~\ref{sec:related},
has proposed several convex surrogate losses.
However, it is still unknown whether minimizing these convex surrogates leads to minimization of the adversarial 0-1 loss.

In this work, we examine the question of which surrogate losses are calibrated with respect to (wrt) the adversarial 0-1 loss. Calibration, defined precisely below, means that for each possible input $x$, minimization of the excess surrogate risk (over a specified class of decision functions) implies minimization of the excess target risk. Calibration thus ensures pointwise consistency, and this notion has been repeatedly used to prove \emph{consistency} of algorithms based on surrogate losses. Employing the calibration function perspective of~\citet{Steinwart:2007},
we show that no convex surrogate loss is calibrated wrt the adversarial 0-1 loss for general distributions
when restricted to the class of linear models (Section~\ref{sec:convex-surrogate}).
Intuitively, this is because convex losses prefer predictions close to the decision boundary on average when $\P(Y=+1|X) \approx \frac{1}{2}$,
while predictions that are too close to the decision boundary should be penalized in adversarially robust classification.
We also provide necessary and sufficient conditions for a certain class
of nonconvex losses to be calibrated wrt the adversarial 0-1 loss (Section~\ref{sec:calibrated-surrogate}).
These calibrated losses attain robustness by penalizing predictions that are too close to the decision boundary.
Finally, we show that under a certain type of low-noise condition~\citep{Massart:2006},
convex losses can be calibrated (Section~\ref{sec:low-noise}).

To our knowledge, this is the first work to formally analyze the adversarial 0-1 loss by calibration analysis.
Our analysis depends on the fact that the adversarially robust 0-1 loss equals the horizontally shifted (non-robust) 0-1 loss when restricted to linear models (Proposition~\ref{prop:margin-error-rate}).
In summary, we argue against the use of convex losses in adversarially robust classification (with linear models),
and calibrated nonconvex losses serve as good alternatives.

\begin{figure}[t]
  \scriptsize
  \centering
  \subfigure[Ramp loss ($\beta = 0.5$)][c]{
    \label{fig:twonorm-ramp}
    \includegraphics[width=0.35\columnwidth]{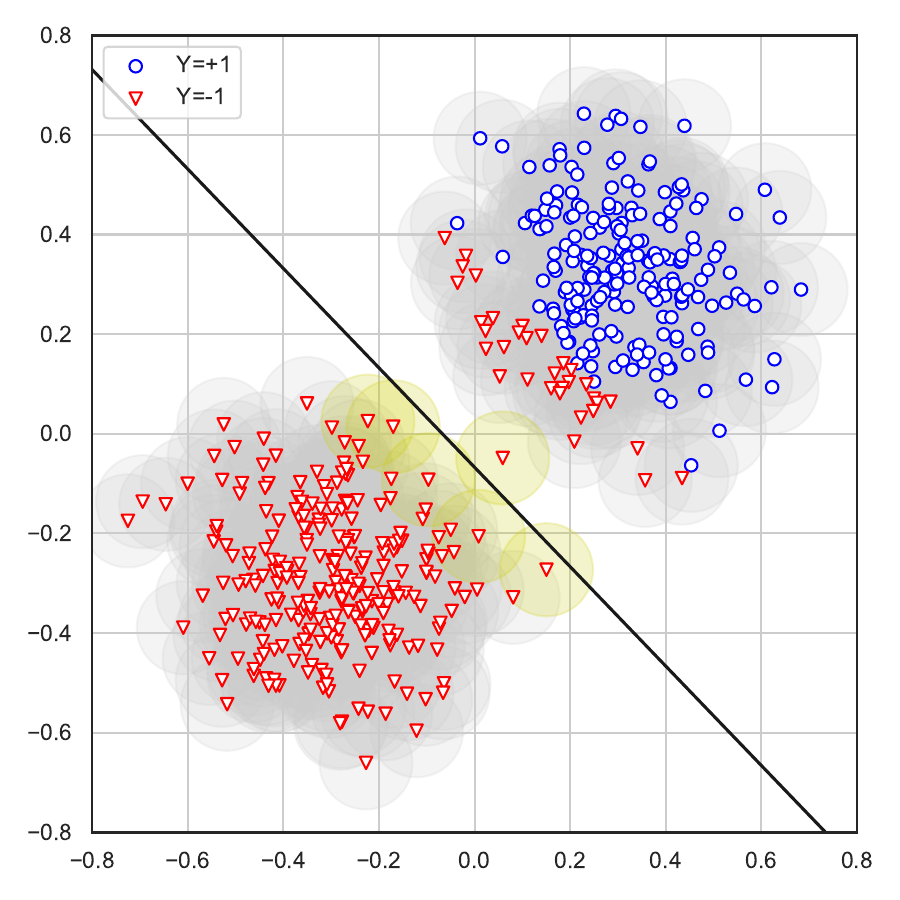}
  }
  \subfigure[Hinge loss ($\beta = 0.5$)][c]{
    \label{fig:twonorm-hinge}
    \includegraphics[width=0.35\columnwidth]{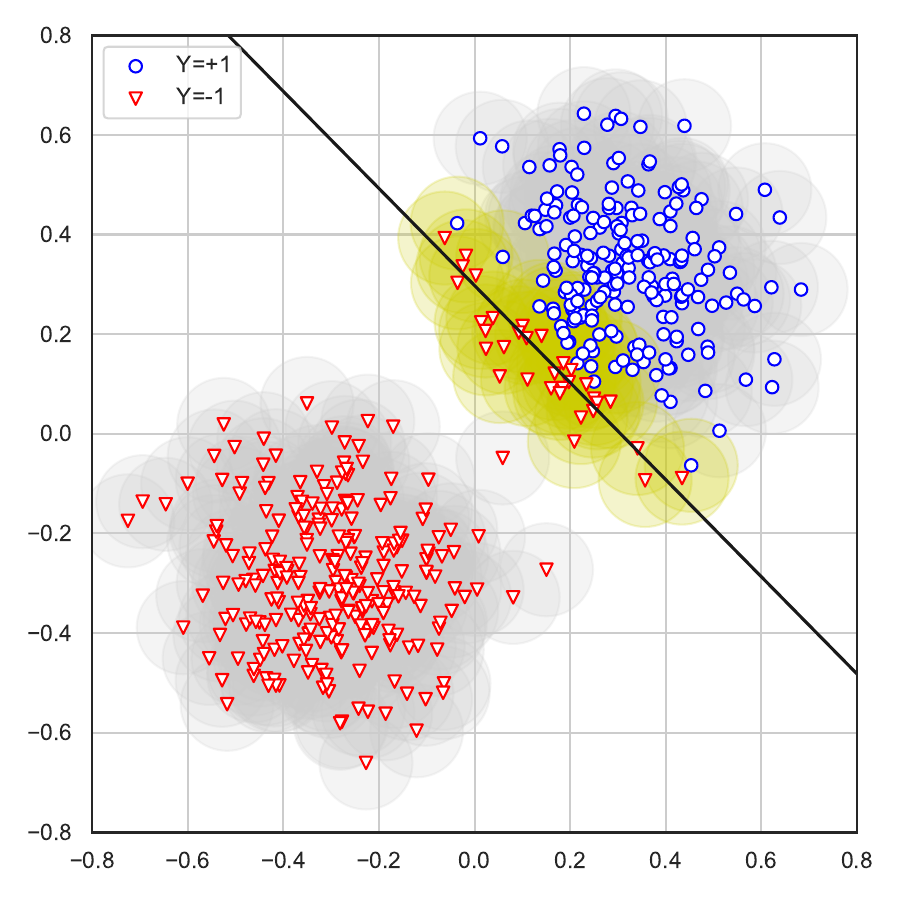}
  }
  \caption{
    The best linear classifier under each loss.
    The shift parameter $\beta$ for a surrogate loss is defined in Section~\ref{sec:examples}.
    The $\ell_2$-balls associated to each instance indicate adversarial perturbations with radii $0.1$.
    The yellow balls indicate instances vulnerable to perturbations,
    in that they are within $0.1$ of the decision boundary.
    In this example,
    $1.2$\% of instances are vulnerable under the ramp loss,
    while $10.4$\% of instances are vulnerable under the hinge loss.
  }
  \label{fig:twonorm}
\end{figure}

Our results demonstrate that adversarial robustness requires different surrogates than other notions of robustness.
For example, symmetric losses such as the sigmoid and ramp losses are robust to label noise~\citep{Ghosh:2015},
but not calibrated wrt the adversarial 0-1 loss.
Figure~\ref{fig:twonorm} illustrates the results of learning a linear classifier
with respect a \emph{shifted} ramp loss,
which is calibrated wrt the adversarial 0-1 loss,
and a shifted hinge loss, which is not (these losses are discussed in detail later).
While the hinge loss yields a classifier with smaller misclassification rate wrt the conventional 0-1 loss,
this classifier is quite sensitive to small perturbations of the test instances.
The classifier learned by the ramp loss, on the other hand,
makes fewer errors when subjected to adversarial perturbations.

The rest of this paper is organized as follows.
Section~\ref{sec:preliminary} formalizes notation and the problem.
Related work on robust learning and calibration analysis is reviewed in Section~\ref{sec:related}.
Technical details of calibration analysis are reviewed in Section~\ref{sec:calibration}.
Section~\ref{sec:convex-surrogate} describes the nonexistence of convex calibrated surrogate losses,
while Section~\ref{sec:calibrated-surrogate} presents general calibration conditions for a certain class of nonconvex losses.
Section~\ref{sec:examples} applies our theory to several convex and nonconvex losses for the calibrated nonconvex losses.
Calibration analysis under low-noise conditions is shown in Section~\ref{sec:low-noise}.
Section~\ref{sec:simulation} shows simulation results to verify that calibrated losses achieve target excess risk tending to zero under the robust 0-1 loss.
Conclusions are stated in Section~\ref{sec:conclusion}.

\begin{remark}
  Here is a summary of the main changes in this corrigendum relative to the published version~\citep{Bao:2020b}.
  \begin{enumerate}
    \item Eliminate an erroneous statement that calibration always implies consistency (Section \ref{sec:preliminary}).
    \item Use the correct definition of calibration function (Definition~\ref{def:calibrated-loss}) and the minimal $\phi$-CCR (Section~\ref{sec:preliminary}).
    \item Modify the proof of our negative result; the statement remains unchanged (Section~\ref{sec:convex-surrogate}). The proof is modified by straightforward adjustment of constants here and there in accordance with the correct definition of calibration function.
    \item Modify the statement and proof of our positive result (Section~\ref{sec:calibrated-surrogate}). The proof is modified in the same way as the negative result, and one technical assumption is also modified.
    \item Change constants in some of the examples (Section~\ref{sec:examples}).
    \item Introduce new positive results on convex losses under Massart's noise condition (Section~\ref{sec:low-noise}).
    \item Update the simulations (Section~\ref{sec:simulation}).
  \end{enumerate}
\end{remark}

\section{Notation}

Let $\|x\|_p$ for a vector $x \in \R^d$ be the $\ell_p$-norm, namely, $\|x\|_p = \sqrt[p]{\sum_{i=1}^d |x_i|^p}$.
Let $B_p(r) \define \{v \in \R^d \mid \|v\|_p \leq r\}$ be the $d$-dimensional closed $\ell_p$-ball with radius $r$,
and $B_p^\circ(r) \define \{v \in \R^d \mid \|v\|_p < r\}$ be the open $\ell_p$-ball.
The set $\{1, \dots, n\}$ is denoted by $[n]$.
The indicator function corresponding to an event $A$ is denoted by $\ind{A}$.
We define the infimum over the empty set as $+\infty$.
Denote $h \equiv c$ for a function $h: S \to \R$ and $c \in \R$ if $h(x) = c$ for all $x \in \dom(h)$,
where $\dom(h)$ denotes the domain of a function $h$,
and $h \not\equiv c$ otherwise.
For a function $h: S \to \R$, we write $h^{**}: S \to \R$ for the Fenchel-Legendre biconjugate of $h$,
characterized by $\epi(h^{**}) = \closconv\epi(h)$,
where $\closconv S$ is the closure of the convex hull of the set $S$,
and $\epi(h)$ is the epigraph of the function $h$:
$\epi(h) \define \{(x, t) \mid x \in S, h(x) \leq t\}$.
A function $h: S \to \R$ is said to be \emph{quasiconcave} if for all $x_1, x_2 \in S$ and $\lambda \in [0, 1]$,
$h(\lambda x_1 + (1 - \lambda) x_2) \geq \min\{h(x_1), h(x_2)\}$.

Let $\calX \define B_2(1)$ be the feature space,
$\calY \define \{\pm 1\}$ be the binary label space,
and $\calF \subseteq \R^\calX$ be a function class.
We consider \emph{symmetric} $\calF$, that is, $-f \in \calF$ for all $f \in \calF$.
We write $\calFall \subseteq \R^\calX$ for the space of all measurable functions.
Let $\ell: \calY \times \calX \times \calF \to \R_{\geq 0}$ be a loss function.
Then, we write $\calR_\ell(f) \define \E_{(X,Y)}[\ell(Y,X,f)]$ for the \emph{$\ell$-risk} of $f \in \calF$,
where $(X, Y) \in \calX \times \calY$ are random variables jointly distributed following the underlying distribution $\P(X, Y)$.
Subsequently, $\P(X)$ and $\P(Y|X)$ denote the $X$-marginal and the posterior distributions, respectively.
If $\ell$ can be represented by $\ell(y, x, f) = \phi(yf(x))$ with some $\phi: \R \to \R_{\geq 0}$
for any $y \in \calY$, $x \in \calX$, and $f \in \calF$,
$\phi$ is called a \emph{margin-based} loss function.
We define the \emph{$\phi$-risk} of $f \in \calF$ for a margin-based loss $\phi$ by
\begin{align}
  \calR_{\phi}(f)
  \define \E_{(X, Y)}[\phi(Yf(X))]
  = \E_X \E_{Y|X}[\phi(Yf(X))],
  \label{eq:phi-risk}
\end{align}
where $\E_X$ and $\E_{Y|X}$ mean the expectation over $\P(X)$ and $\P(Y|X)$, respectively.
We can rewrite \eqref{eq:phi-risk} as $\calR_{\phi}(f) = \E_X[\calC_\phi(f, \P(Y=+1|X), X)]$
with $\calC_\phi(f, \eta, x) \define \eta\phi(f(x)) + (1 - \eta)\phi(-f(x))$.
We call $\calC_\phi(f, \eta, x)$ the \emph{class-conditional $\phi$-risk}, or $\phi$-CCR.
The minimal $\phi$-risk (over a function class $\calF$) $\calR_{\phi,\calF}^* \define \inf_{f \in \calF} \calR_{\phi}(f)$ is called the \emph{Bayes ($\phi$, $\calF$)-risk},
and the minimal $\phi$-CCR on $\calF$ at $x$ is denoted by $\calC_{\phi,\calF}^*(\eta, x) \define \inf_{f \in \calF} \calC_{\phi}(f, \eta, x)$.
We refer to $\calR_{\phi}(f) - \calR_{\phi,\calF}^*$ as the $(\phi, \calF)$-excess risk.
We occasionally use the abbreviation $\Delta\calC_{\phi,\calF}(f, \eta, x) \define \calC_{\phi}(f, \eta, x) - \calC_{\phi,\calF}^*(\eta, x)$ to denote the excess $(\phi,\calF)$-CCR at $x$.
For non-margin-based loss function $\ell$,
we define the $\ell$-CCR $\calC_{\ell,\calF}(f, \eta, x)$, the minimal $\ell$-CCR $\calC_{\ell,\calF}^*(\eta, x)$, and $\Delta\calC_{\ell,\calF}(f, \eta, x)$ in the same manner.

\begin{remark}
  In our published version~\citep{Bao:2020b}, the minimal $\phi$-CCR is inappropriately defined as $\calC_{\phi,\calF}^*(\eta) \define \inf_{f \in \calF, x \in \calX} \calC_{\phi}(f, \eta, x)$.
  In this corrigendum, the minimal CCR is defined as $\calC_{\phi,\calF}^*(\eta, x) \define \inf_{f \in \calF} \calC_{\phi}(f, \eta, x)$,
  which depends on $x$ as well,
  leading to the correct definition of calibrated losses in Definition~\ref{def:calibrated-loss} later.
\end{remark}

\section{Surrogate Losses for Adversarial Robust Classification}
\label{sec:preliminary}

In supervised binary classification, a learner is asked to output a predictor $f: \calX \to \R$ that minimizes the classification error $\P\{Yf(X) \leq 0\}$,
where $\P$ is the unknown underlying distribution.
This can be equivalently interpreted as the minimization of the risk $\E_{(X, Y)}[\ell_{01}(Y, X, f)]$ wrt $f$,
where
\begin{align*}
  \ell_{01}(y, x, f) \define \begin{cases}
    1 & \text{if $y \ne \sign(f(x))$}, \\
    0 & \text{otherwise}
  \end{cases}
\end{align*}
is the 0-1 loss.
Here, we adopt the convention $\sign(0) \define +1$.
On the other hand, an adversarially robust learner is asked to output a predictor $f$
that minimizes the 0-1 loss while being tolerant to small perturbations to input data points.
Following existing literature~\citep{Xu:2009,Tsuzuku:2018,Bubeck:2019},
we consider $\ell_2$-ball perturbations and define the goal as the minimization of $\P\{\exists \Delta_x \in B_2(\gamma) \text{ s.t. } X + \Delta_x \in \calX \text{ and } Yf(X + \Delta_x) \leq 0\}$,
where $\Delta_x$ is a perturbation vector and $\gamma \in (0, 1)$ is a pre-defined perturbation budget.
Equivalently, the goal of adversarially robust classification is to minimize $\E_{(X, Y)}[\ell_\gamma(Y, X, f)]$ wrt $f$,
where
\begin{align}
  \ell_\gamma(y, x, f) \define \begin{cases}
    1 & \text{if $\exists \Delta_x \in B_2(\gamma)$ s.t. $x + \Delta_x \in \calX$ and $yf(x + \Delta_x) \leq 0$}, \\
    0 & \text{otherwise}.
  \end{cases}
  \nonumber
\end{align}
We call this loss function $\ell_\gamma$ the \emph{adversarially robust 0-1 loss}, or the \emph{robust 0-1 loss} for short.

The robust 0-1 loss is a margin-based loss when restricted to the class of linear models
$\calFlin \define \{x \mapsto \theta^\top x \mid \theta \in \R^d, \|\theta\|_2 = 1\} \subseteq \R^\calX$.
Note that $\calFlin$ is symmetric.
\begin{proposition}
  \label{prop:margin-error-rate}
  For any $x \in \calX$, $y \in \calY$, and $f \in \calFlin$,
  we have $\ell_\gamma(y, x, f) = \ind{yf(x) \leq \gamma}$.
\end{proposition}
We include the proof in Appendix~\ref{sec:proof} for completeness though it is mentioned as a fact by \citet{Diakonikolas:2019}.
Subsequently, when considering $\calFlin$, we work with the loss function $\phi_\gamma(\alpha) \define \ind{\alpha \leq \gamma}$
and call $\phi_\gamma$ the \emph{$\gamma$-robust 0-1 loss}.
We will study calibrated surrogates wrt $\phi_\gamma$ instead of $\ell_\gamma$,
and both are equivalent under the restricted function class $\calFlin$.

In many machine learning problems, there are often dichotomies between optimization (learning) and evaluation.
For instance, binary classification is evaluated by the 0-1 loss,
while common learning methods such as the support vector machine and logistic regression minimize surrogates to the 0-1 loss.
This dichotomy arises because minimizing the 0-1 loss directly is known to be NP-hard~\citep{Feldman:2012}.
Much research has investigated surrogates $\phi$ satisfying
\begin{align}
  \calR_\phi(f_i) - \calR_{\phi,\calF}^* \to 0
  \Longrightarrow
  \calR_\ell(f_i) - \calR_{\ell,\calF}^* \to 0,
  \label{eq:risk-convergence}
\end{align}
for all probability distributions and sequence of $\{f_i\}_{i \in \N} \subseteq \calF$.
When \eqref{eq:risk-convergence} is satisfied, the surrogate $\phi$ is said \emph{$(\psi,\calF)$-consistent}.

In this work, we study a pointwise form of consistency known as \emph{calibration},
which can be viewed as consistency of the excess $(\phi,\calF)$-CCR $\calC_{\phi,\calF}(f, \eta, x) - \calC_{\phi,\calF}^*(\eta, x)$ at each $x \in {\cal X}$
(formally defined in Section~\ref{sec:calibration}).
Since CCRs are defined in the pointwise manner, calibration analysis is easier than analyzing consistency directly, and has been used to prove consistency in a number of learning settings as we will see in Section~\ref{sec:related}. For example, calibration analysis has been performed in standard binary classification~\citep{Bartlett:2006},
where calibration is necessary~\citep[Theorem~3.3]{Steinwart:2007} and sufficient~\citep[Theorem~2.8]{Steinwart:2007} for consistency when $\calF = \calFall$.
When $\calF \ne \calFall$, calibration may not be sufficient for consistency, although it remains an important first step to analyze and understand consistency in standard classification~\citep{Long:2013,Zhang:2020}. This motivates the study of calibration in the context of adversarially robust classification.

\begin{remark}
  While it has been stated in our published version~\citep{Bao:2020b} that calibration of a surrogate loss is sufficient for consistency,
  this is not the case as \citet{Awasthi:2021} pointed out by means of a counterexample.
  For calibrated losses to imply consistency, the losses require an additional technical assumption, \emph{$\P$-minimizability},
  which does not necessarily hold in general unless $\calF = \calFall$~\citep[Theorem~3.2]{Steinwart:2007}.
  \citet[Theorem~25]{Awasthi:2021} provide sufficient conditions under which consistency does hold for $\calFlin$.
  Our simulations further indicate consistency for calibrated losses in settings (twonorm and advnorm datasets) where the sufficient condition of \citet{Awasthi:2021} does not hold.
  A full characterization of consistency for adversarially robust learning remains an open problem.
\end{remark}

\section{Related Work}
\label{sec:related}

From the viewpoint of robust optimization~\citep{Ben-Tal:2009,Bertsimas:2011},
adversarially robust binary classification can be formulated as
\begin{align}
  \min_{f \in \calF} \E_{(X, Y)} \left[ \max_{\tilde X \in \calU(X)} \ell(Y, \tilde X, f) \right],
  \label{eq:robust-classification}
\end{align}
where $\ell$ is a loss function
and $\calU(x)$ is a user-specified uncertainty set.
The optimization problem of adversarially robust classification $\min_{f \in \calF} \calR_{\ell_\gamma}(f)$ can be regarded as the special case
$\ell = \ell_{01}$ and $\calU(x) = x + B_2(\gamma)$.

Since the minimax problem \eqref{eq:robust-classification} is generally nonconvex,
it is traditionally tackled by minimizing a convex upper bound.
\citet{Lanckriet:2002} and \citet{Shivaswamy:2006} pick $\calU(x) = \{x \sim (\bar x, \Sigma_x)\}$ as an uncertainty set,
where $x \sim (\bar x, \Sigma_x)$ means that $x$ is drawn from a distribution that has prespecified mean $\bar x$, covariance $\Sigma_x$, and arbitrary higher moments.
\citet{Lanckriet:2002} and \citet{Shivaswamy:2006} convexified \eqref{eq:robust-classification} and obtained a second-order cone program.
\citet{Xu:2009} studied the relationship between robustness and regularization,
and showed that \eqref{eq:robust-classification} with the hinge loss and $\calU(x) = x + B_2(\gamma)$ is equivalent to $\ell_2$-regularized SVM.
Recently, \citet{Wong:2018}, \citet{Madry:2018}, \citet{Raghunathan:2018a}, \citet{Raghunathan:2018b}, and \citet{Khim:2019} examined \eqref{eq:robust-classification}
with the softmax cross entropy loss and $\calU(x) = x + B_\infty^d(\gamma)$ when $\calF$ is a set of deep nets,
and provided convex upper bounds of the worst-case loss in \eqref{eq:robust-classification}.
However, no work except \citet{Cranko:2019} studied whether the surrogate objectives minimize the robust 0-1 excess risk.
\citet{Cranko:2019} showed that no canonical proper loss~\citep{Reid:2010} can minimize the robust 0-1 loss.
Since canonical proper losses are convex,
this result aligns with ours.
We show more general results via calibration analysis for $\calU(x) = x + B_2(\gamma)$.

There are several other approaches to the robust classification
such as minimizing the Taylor approximation of the worst-case loss in \eqref{eq:robust-classification} \citep{Goodfellow:2015,Gu:2015,Shaham:2018},
regularization on the Lipschitz norm of models~\citep{Cisse:2017,Hein:2017,Tsuzuku:2018},
and injection of random noises to model parameters~\citep{Lecuyer:2019,Cohen:2019,Pinot:2019,Salman:2019}.
It is not known whether these methods imply the minimization of the robust 0-1 excess risk.

Other forms of robustness have also been considered in the literature.
A number of existing works considered the worst-case test distribution.
This line includes divergence-based methods~\citep{Namkoong:2016,Namkoong:2017,Hu:2018,Sinha:2018},
domain adaptation~\citep{Mansour:2009,Ben-David:2010,Germain:2013,Kuroki:2019,Zhang:2019},
and methods based on constraints on feature moments~\citep{Farnia:2016,Fathony:2016}.

In addition to adversarial robustness, it is worthwhile to mention outlier and label-noise robustness.
It is known that convex losses are vulnerable to outliers,
thus truncation making losses nonconvex is useful~\citep{Huber:2011}.
In the machine learning context, \citet{Masnadi-Shirazi:2009} and \citet{Holland:2019} designed nonconvex losses robust to outliers.
On the other hand, label-noise robustness, especially the random classification noise model, has been studied extensively~\citep{Angluin:1988},
where training labels are flipped with a fixed probability.
\citet{Long:2010} showed that there is no convex loss that is robust to label noises.
Later, \citet{Ghosh:2015}, \citet{vanRooyen:2015}, and \citet{Charoenphakdee:2019} discovered a certain class of nonconvex losses is a good alternative
for label-noise robustness.
In both outlier and label-noise robustness,
nonconvex loss functions play an important role as we see in adversarial robustness.

Calibration analysis has been formalized in \citet{Lin:2004}, \citet{Zhang:2004a}, \citet{Bartlett:2006}, and \citet{Steinwart:2007},
and employed to analyze not only binary classification,
but also complicated problems such as
multi-class classification~\citep{Zhang:2004b,Tewari:2007,Long:2013,Avila-Pires:2016,Ramaswamy:2016},
multi-label classification~\citep{Gao:2011,Dembczynski:2012},
cost-sensitive learning~\citep{Scott:2011,Scott:2012,Avila-Pires:2013},
ranking~\citep{Duchi:2010,Ravikumar:2011,Ramaswamy:2013},
structured prediction~\citep{Hazan:2010,Ramaswamy:2012,Osokin:2017,Blondel:2019},
AUC optimization~\citep{Gao:2015},
and optimization of non-decomposable metrics~\citep{Bao:2020}.
\citet{Zhang:2004a}, \citet{Ravikumar:2011}, and \citet{Gao:2015} figured out \emph{ad hoc} derivations of excess risk bounds,
while \citet{Bartlett:2006}, \citet{Steinwart:2007}, \citet{Scott:2012}, \citet{Avila-Pires:2013}, \citet{Avila-Pires:2016}, \citet{Osokin:2017}, and \citet{Blondel:2019}
used more systematic approaches.
As for adversarially robust classification,
\citet[Theorem 3.1]{Zhang:2019b} applied the classical result of calibration analysis on convex losses to upper bound the robust classification risk,
resulting in a term requiring numerical approximation in practice.

Finally, \citet{Awasthi:2021} contributed calibration analysis of adversarially robust classification by showing that realizability assumptions are sufficient for calibrated losses to imply consistency.
They showed that no \emph{continuous} margin-based losses are calibrated
and that some nonconvex and minimax-type losses are consistent wrt the robust 0-1 loss.
\citet{Awasthi:2021b} independently corrected our main results and extended them to more general function classes beyond $\calFlin$.

\section{Calibration Analysis}
\label{sec:calibration}

Calibration analysis is a tool to study the relationship between surrogate losses and target losses.
This section is devoted to explaining the calibration function introduced in \citet{Steinwart:2007} and specializing it to the current paper.\footnote{
  We import toolsets from \citet{Steinwart:2007} because of two reasons:
  (i) \citet{Steinwart:2007} formalized calibration analysis
  that is \emph{dependent} on user-specified function classes,
  which is useful for our analysis on $\calFlin$.
  (ii) \citet{Steinwart:2007} gave a general form of the calibration function~\eqref{eq:calibration-function},
  while most of literature focuses on specific target losses.
}

\begin{definition}
  \label{def:calibrated-loss}
  For a loss $\psi: \R \to \R_{\geq 0}$ and a function class $\calF$,
  we say a loss $\phi: \R \to \R_{\geq 0}$ is \emph{calibrated wrt $(\psi,\calF)$},
  or \emph{$(\psi,\calF)$-calibrated},
  if for any $\epsilon > 0$, $\eta \in [0, 1]$, and $x \in \calX$,
  there exists $\delta > 0$ such that for all $f \in \calF$, we have
  \begin{align}
    \calC_\phi(f, \eta, x) < \calC_{\phi, \calF}^*(\eta, x) + \delta
    \implies
    \calC_\psi(f, \eta, x) < \calC_{\psi, \calF}^*(\eta, x) + \epsilon.
    \label{eq:ccr-convergence}
  \end{align}
\end{definition}

If $\phi$ is $(\psi,\calF)$-calibrated, the condition \eqref{eq:risk-convergence}
holds for any probability distribution on $\calX \times \calY$ satisfying regularity conditions~\citep[Theorem 2.8]{Steinwart:2007}.\footnote{
  In order to imply $(\psi,\calF)$-consistency~\eqref{eq:risk-convergence},
  the two loss functions $\phi$ and $\psi$ are required to be \emph{$\P$-minimizable} for the underlying distribution $\P$---roughly meaning that their CCRs can be made arbitrarily small by a function in $\calF$.
  This ensures $\calR_{\phi,\calF}^* = \E_X[\calC_{\phi,\calF}^*(\P(Y=+1|X), X)]$.
  The precise statements and more details about $\P$-minimizability can be found in \citet[Definition~2.4 and Lemma~2.5]{Steinwart:2007}.
}

Next, we introduce the \emph{calibration function}~\citep[Lemma~2.9]{Steinwart:2007}.
\begin{definition}
  For a margin-based loss $\psi$ and $\phi$, and a function class $\calF$,
  the \emph{calibration function of $\phi$ wrt $(\psi,\calF)$}, or simply \emph{calibration function} if the context is clear,
  is defined as
  \begin{align}
    \bar\delta(\epsilon, \eta, x) = \inf_{f \in \calF} \calC_\phi(f, \eta, x) - \calC_{\phi,\calF}^*(\eta, x)
    \quad \text{s.t.} \quad
    \calC_\psi(f, \eta, x) - \calC_{\psi,\calF}^*(\eta, x) \geq \epsilon.
    \label{eq:calibration-function}
  \end{align}
\end{definition}
Note that $\bar\delta(\epsilon, \eta, x)$ is nondecreasing for $\epsilon > 0$.
The calibration function $\bar\delta(\epsilon, \eta, x)$ is the maximal $\delta$ satisfying the CCR condition \eqref{eq:ccr-convergence}.
\citet{Steinwart:2007} established the following important result to confirm if a surrogate is $(\psi,\calF)$-calibrated.
\begin{proposition}[\citet{Steinwart:2007}]
  \label{prop:calibration-condition}
  A surrogate loss $\phi$ is $(\psi,\calF)$-calibrated if and only if its calibration function $\delta$ satisfies $\bar\delta(\epsilon, \eta, x) > 0$ for all $\epsilon > 0$, $\eta \in [0, 1]$, and $x \in \calX$.
\end{proposition}

In order to see the relationship between $(\psi,\calF)$-excess risk and $(\phi,\calF)$-excess risk,
a stronger notion of calibrated losses than Definition~\ref{def:calibrated-loss} is necessary.
\begin{definition}
  \label{def:uniformly-calibrated-loss}
  For a loss $\psi: \R \to \R_{\geq 0}$ and a function class $\calF$,
  we say a loss $\phi: \R \to \R_{\geq 0}$ is \emph{uniformly $(\psi,\calF)$-calibrated},
  if for any $\epsilon > 0$, there exists $\delta > 0$ such that for all $\eta \in [0, 1]$, $f \in \calF$, and $x \in \calX$, we have
  \begin{align}
    \calC_\phi(f, \eta, x) < \calC_{\phi, \calF}^*(\eta, x) + \delta
    \implies
    \calC_\psi(f, \eta, x) < \calC_{\psi, \calF}^*(\eta, x) + \epsilon.
    \nonumber
  \end{align}
  The corresponding \emph{uniform calibration function} is defined as
  \begin{align*}
    \delta(\epsilon) = \inf_{\eta \in [0, 1]} \inf_{x \in \calX} \bar\delta(\epsilon, \eta, x).
  \end{align*}
\end{definition}
Note that Definition~\ref{def:uniformly-calibrated-loss} is slightly but substantially different from Definition~\ref{def:calibrated-loss}
in that the order of quantifiers of $\delta$ and $(\eta, x)$ is reversed.
With this notion, we can connect the surrogate excess risk to the target excess risk
as shown in the following statement.
\begin{proposition}[Theorem 2.13 in \citet{Steinwart:2007}]
  \label{prop:excess-risk-transform}
  Let $\delta: \R_{\geq 0} \to \R_{\geq 0}$ be the uniform calibration function of $\phi$ wrt $(\psi,\calF)$.
  Define $\check\delta: \R_{\geq 0} \to \R_{\geq 0}$ as $\check\delta(\epsilon) = \delta(\epsilon)$ if $\epsilon > 0$ and $\check\delta(0) = 0$.
  Suppose that $\phi$ and $\psi$ are $\P$-minimizable and $\calR_{\phi,\calF}^*, \calR_{\psi,\calF}^* < \infty$.
  Then, for all $f \in \calF$, we have
  \begin{align}
    \check\delta^{**}\left(\calR_\psi(f) - \calR_{\psi,\calF}^*\right) \leq \calR_\phi(f) - \calR_{\phi,\calF}^*,
    \label{eq:excess-risk-transform}
  \end{align}
  where $\check\delta^{**}$ denotes the Fenchel-Legendre biconjugate of $\check\delta$.
\end{proposition}
The relationship in \eqref{eq:excess-risk-transform} is called an excess risk transform.
The excess risk transform is invertible iff $\phi$ is uniformly $(\psi,\calF)$-calibrated~\citep[Remark~2.14]{Steinwart:2007}.
In this case, we obtain the excess risk bound $\calR_\psi(f) - \calR_{\psi,\calF}^* \leq (\check\delta^{**})^{-1}(\calR_\phi(f) - \calR_{\phi,\calF}^*)$.\footnote{
  In addition, it is known that a non-vacuous and distribution-independent excess risk transform is available
  only if a surrogate is uniformly calibrated provided that the biconjugate of the calibration function is invertible~\citep[Theorem~2.17]{Steinwart:2007}.
  Hence, uniform calibration is necessary to obtain an excess risk bound.
}
In the end, the calibration function can be used in two ways:
Proposition~\ref{prop:calibration-condition} enables us to check if a surrogate loss is calibrated,
and Proposition~\ref{prop:excess-risk-transform} gives us a quantitative relationship between the surrogate excess risk and the target excess risk.
Such an analysis has been carried out in a number of learning problems as we mention in Section~\ref{sec:related}.

We review an important result regarding convex surrogates for the non-robust 0-1 loss $\ell_{01}$.
\begin{proposition}[Theorem 6 in \citet{Bartlett:2006}]
  \label{prop:convex-loss}
  Let $\phi$ be a convex surrogate loss.
  Then, $\phi$ is uniformly calibrated wrt $(\ell_{01},\calFall)$ if and only if it is differentiable at $0$ and $\phi'(0) < 0$.
\end{proposition}
As a result of Proposition~\ref{prop:convex-loss}, we know that many surrogate losses used in practice such as the hinge loss, logistic loss, and squared loss are uniformly calibrated
wrt $(\ell_{01},\calFall)$. One of our objectives in this paper is to establish a general class of loss functions that are calibrated wrt the adversarial 0-1 loss, in analogy to Proposition \ref{prop:convex-loss}.

Before proceeding to our main results, we present two lemmas that facilitate our analysis.
All proofs are deferred to Appendix~\ref{sec:proof}.
\begin{lemma}
  \label{lem:calibration-iff-condition}
  Let $\tcalX_\rho \define \calX \setminus B_2^\circ(\gamma + \rho)$ and $\phi$ be a continuous surrogate loss.
  Denote
  \begin{align}
    \delta_\rho(\epsilon)
    = \inf_{\eta \in [0,1]} \inf_{x \in \tcalX_\rho} \inf_{f \in \calFlin} \calC_{\phi}(f, \eta, x) - \calC_{\phi,\calFlin}^*(\eta, x)
    \quad \text{s.t.} \quad
    \calC_{\phi_\gamma}(f, \eta, x) - \calC_{\phi_\gamma,\calFlin}^*(\eta, x) \geq \epsilon.
    \nonumber
  \end{align}
  Then, $\phi$ is $(\phi_\gamma,\calFlin)$-calibrated
  if and only if $\delta_\rho(\epsilon) > 0$ for all $\epsilon > 0$ and $\rho \in (0, 1-\gamma)$.
\end{lemma}
The calibration function with the restricted domain $\delta_\rho$ is easier to work with in the subsequent analyses. Finally, we characterize the calibration function of an arbitrary surrogate loss $\phi$ wrt $(\phi_\gamma,\calFlin)$.
\begin{lemma}
  \label{lem:mer-calibration-function}
  Let $\phi$ be a surrogate loss.
  Then, the $(\phi_\gamma,\calFlin)$-calibration function is
  \begin{align}
    \bar\delta(\epsilon, \eta, x) = \begin{cases}
      \infty & \text{if $\epsilon > \max\{\eta, 1 - \eta\}$}, \\
      \inf\limits_{f \in \calFlin: |f(x)| \leq \gamma} \Delta\calC_{\phi,\calFlin}(f, \eta, x) & \text{if $|2\eta - 1| < \epsilon \leq \max\{\eta, 1 - \eta\}$}, \\
      \inf\limits_{f \in \calFlin: (2\eta - 1)f(x) \leq 0 \text{ or } |f(x)| \leq \gamma} \Delta\calC_{\phi,\calFlin}(f, \eta, x) & \text{if $\epsilon \leq |2\eta - 1|$},
    \end{cases}
    \label{eq:inner-calibration-function}
  \end{align}
  when $\xnorm > \gamma$, and $\bar\delta(\epsilon, \eta, x) = \infty$ when $\xnorm \leq \gamma$.
\end{lemma}
Lemmas~\ref{lem:calibration-iff-condition} and \ref{lem:mer-calibration-function} are used in the proofs and examples below.

\section{Convex Surrogates are Not \texorpdfstring{($\phi_\gamma$, $\calFlin$)}{(phi(gamma), F(lin)}-calibrated}
\label{sec:convex-surrogate}

Our first result concerns calibration of convex surrogate losses wrt the $\gamma$-robust 0-1 loss.

\begin{theorem}
  \label{thm:no-convex-calibrated-surrogate}
  For any margin-based surrogate loss $\phi: \R \to \R_{\geq 0}$,
  if $\phi$ is convex, then $\phi$ is not calibrated wrt ($\phi_\gamma$,$\calFlin$).
\end{theorem}

\begin{proof} \textit{(Sketch)}
  In the non-robust setup, \citet{Bartlett:2006} showed that a surrogate loss is calibrated wrt ($\ell_{01}$,$\calFall$)
  iff $\inf_{(2\eta - 1)f(x) \leq 0} \calC_\phi(f, \eta, x)$ (the minimum $\phi$-risk over `wrong' predictions) is larger than $\inf_{f(x) \in \R} \calC_\phi(f, \eta, x)$ (the minimum $\phi$-risk over all predictions) for $\eta \ne \frac{1}{2}$.
  This means wrong predictions must be penalized more.
  In our robust setup, we must penalize not only wrong predictions but also predictions that fall in the $\gamma$-margin,
  i.e.,
  \begin{align}
    \inf_{f \in \calFlin: |f(x)| \leq \gamma} \calC_\phi(f, \eta, x) > \inf_{f \in \calFlin} \calC_\phi(f, \eta, x),
    \label{eq:penalize-gamma-margin}
  \end{align}
  which is an immediate corollary of Proposition~\ref{prop:calibration-condition} and Lemma~\ref{lem:mer-calibration-function}
  and stated in part~\ref{lem:ccr-property:expanded-calibration-condition} of Lemma~\ref{lem:ccr-property} in Appendix~\ref{sec:proof}.
  Condition~\eqref{eq:penalize-gamma-margin} becomes harder to satisfy as a data point gets more uncertain ($\eta \to \frac{1}{2}$).
  In the limit, we have $\inf_{|\alpha| \leq \gamma} \phi(\alpha) + \phi(-\alpha) > \inf_{\alpha \in \R} \phi(\alpha) + \phi(-\alpha)$,
  meaning that the even part of $\phi$ should take larger values in $|\alpha| \leq \gamma$ than in the rest of $\alpha$.
  However, $\phi(\alpha) + \phi(-\alpha)$ attains the infimum at $\alpha = 0$
  because $\phi(\alpha) + \phi(-\alpha)$ is convex and even as long as $\phi$ is convex.
  Therefore, the condition \eqref{eq:penalize-gamma-margin} would never be satisfied
  by convex surrogate $\phi$.
  This idea is illustrated in Figure~\ref{fig:convex-loss}.
\end{proof}

Hence, many popular surrogate losses such as the hinge, logistic, and squared error losses are not calibrated wrt ($\phi_\gamma$,$\calFlin$).
We defer all proofs to Appendix~\ref{sec:proof}.

\begin{figure}[t]
  \begin{minipage}[c]{0.45\columnwidth}
    \centering
    \scriptsize
    \def\varbeta{0.3}
    \def\vargamma{0.5}
    \includegraphics{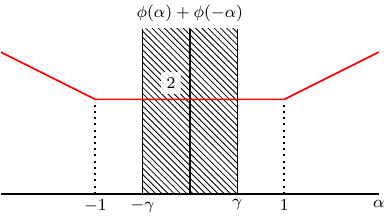}
    \caption{
      $\phi(\alpha) + \phi(-\alpha) = 2\calC_\phi\left(f, \frac{1}{2}, x\right)$ is illustrated with $\alpha = f(x)$,
      where $\phi$ is the hinge loss and $\gamma = 0.5$.
      $\phi(\alpha) + \phi(-\alpha)$ has the same minimizers in both $|\alpha| \leq \gamma$ and $|\alpha| \leq 1$.
    }
    \label{fig:convex-loss}
  \end{minipage}
  \hfill
  \begin{minipage}[c]{0.45\columnwidth}
    \centering
    \scriptsize
    \includegraphics{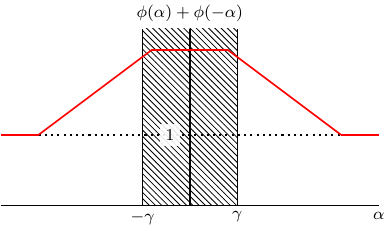}
    \caption{
      $\phi(\alpha) + \phi(-\alpha) = 2\calC_\phi\left(f, \frac{1}{2}, x\right)$ is illustrated with $\alpha = f(x)$,
      where $\phi$ is the ramp loss with $\beta = 0.6$ (defined in Section~\ref{sec:examples}) and $\gamma = 0.5$.
      The condition $\phi(\gamma) + \phi(-\gamma) > \phi(\alpha) + \phi(-\alpha)$ for $\alpha \in (\gamma, 1]$ reflects the idea
      that predictions falling into the shaded area ($|\alpha| \leq \gamma$) must be penalized more
      than the others.
    }
    \label{fig:qccv-even-loss}
  \end{minipage}
\end{figure}

Note that the definition of calibration makes no assumptions on the conditional distribution $\P(Y=+1|x)$.
If we additionally adopt the low noise assumption~\citep{Massart:2006}, then it is possible for a convex loss to be calibrated wrt ($\phi_\gamma$,$\calFlin$).
We will discuss the details later in Section~\ref{sec:low-noise}.

\section{Calibration Conditions for Nonconvex Surrogates}
\label{sec:calibrated-surrogate}

As seen in Section~\ref{sec:convex-surrogate},
convex surrogate losses that are calibrated wrt ($\phi_\gamma$,$\calFlin$) do not exist.
This motivates a search for nonconvex surrogate losses.
Nonconvex surrogates are used for outlier robustness~\citep{Collobert:2006,Masnadi-Shirazi:2009,Holland:2019} or label-noise robustness~\citep{Ghosh:2015,vanRooyen:2015,Charoenphakdee:2019}.
Bounded monotone surrogates such as the ramp loss and the sigmoid loss are simple and common choices for those purposes. In this section, we also look for good surrogates from bounded monotone losses.

The following assumption will be adopted.
\begin{assumption}
  \label{assump:qccv-ccr}
  For a margin-based loss function $\phi: \R \to \R_{\geq 0}$,
  $\phi(-\alpha) > \phi(\alpha)$ for $\alpha \in (\gamma, 1]$,
  and its $\phi$-CCR $\calC_\phi(\cdot, \eta)$ is quasiconcave for all $\eta \in [0, 1]$.
\end{assumption}
The assumption $\phi(-\alpha) > \phi(\alpha)$ for $\alpha \in (\gamma, 1]$ is naturally satisfied
by surrogates strictly decreasing in $[-\alpha_0, \alpha_0]$ with sufficiently large $\alpha_0 > 0$.

Next, we state our main positive result.
Its proof is included in Appendix~\ref{sec:proof}.
\begin{theorem}
  \label{thm:calibration-condition}
  Let $\phi: \R \to \R_{\geq 0}$ be a surrogate loss.
  Assume that $\phi$ is bounded, continuous, nonincreasing, and satisfies Assumption~\ref{assump:qccv-ccr}.
  Let $\calF = \calFlin$.
  Then,
  \begin{enumerate}
    \item \label{thm:calibration-condition:0/1}
    $\phi$ is $(\ell_{01},\calFlin)$-calibrated.

    \item \label{thm:calibration-condition:robust-0/1}
    $\phi$ is $(\phi_\gamma,\calFlin)$-calibrated
    if and only if $\phi(\gamma) + \phi(-\gamma) > \phi(\alpha) + \phi(-\alpha)$ for all $\alpha \in (\gamma, 1]$.
  \end{enumerate}
\end{theorem}

\begin{proof} \textit{(Sketch of \ref{thm:calibration-condition:robust-0/1})}
  As in the proof sketch of Theorem~\ref{thm:no-convex-calibrated-surrogate},
  \eqref{eq:penalize-gamma-margin} is needed for $(\phi_\gamma,\calFlin)$-calibration, and thus
  $\phi(\alpha) + \phi(-\alpha)$ should take larger values in $|\alpha| \leq \gamma$ than in the rest of $\alpha$.
  Quasiconcavity of $\phi(\alpha) + \phi(-\alpha)$ naturally implies this property with a non-strict inequality,
  and the condition $\phi(\gamma) + \phi(-\gamma) > \phi(\alpha) + \phi(-\alpha)$ (for all $\alpha > \gamma$) ensures a strict inequality.
  Figure~\ref{fig:qccv-even-loss} illustrates this idea with the ramp loss.
\end{proof}

To the best our knowledge, this is the first characterization
of losses calibrated to $\phi_\gamma$.

\begin{remark}
For all $\alpha > \gamma$,
$\phi(\gamma) + \phi(-\gamma) \geq \phi(\alpha) + \phi(-\alpha)$ always holds when
$\phi$ is bounded, continuous, nonincreasing, and satisfies Assumption~\ref{assump:qccv-ccr} (see
part~\ref{lem:ccr-qccv:even-decreasing} of Lemma~\ref{lem:ccr-qccv} in
Appendix~\ref{sec:proof}). The strict inequality $\phi(\gamma) +
\phi(-\gamma) > \phi(\alpha) + \phi(-\alpha)$ is then necessary and sufficient for
$(\phi_\gamma,\calFlin)$-calibration.
\end{remark}

\begin{remark}
The ramp loss and the sigmoid loss are $(\ell_{01},\calFall)$-calibrated~\citep{Bartlett:2006,Charoenphakdee:2019}.
Note that these two losses are bounded, continuous, nonincreasing, and satisfy Assumption~\ref{assump:qccv-ccr},
hence $(\ell_{01}$,$\calFlin)$-calibrated.
\end{remark}

\begin{remark}
  After modifying the definition of the calibration (Definition~\ref{def:calibrated-loss}),
  the condition in Theorem~\ref{thm:calibration-condition} (part~\ref{thm:calibration-condition:robust-0/1}) has been changed from $\phi(\gamma) + \phi(-\gamma) > \phi(1) + \phi(-1)$ in the published version~\citep{Bao:2020b}.

  In parallel, Assumption~\ref{assump:qccv-ccr} is newly introduced,
  which is stronger than \emph{quasiconcave even} losses assumed in the published version: $C_{\phi}\left(\cdot, \frac{1}{2}\right)$ is quasiconcave.
  Indeed, quasiconcave even losses do not satisfy Assumption~\ref{assump:qccv-ccr} in general as was asserted by \citet[Lemma~13 (part~1)]{Bao:2020b}
  because Lemma~15 used in the proof has an error.
  A counterexample
  \begin{align*}
    \phi(\alpha) \define \frac{e^{-\alpha^2} + 1}{2} + \max\{-5, \min\{5, -x\}\} + \num{4.5}
  \end{align*}
  verifies this---$\calC_{\phi}(\cdot, \num{0.6})$ is not quasiconcave even though $C_{\phi}\left(\cdot, \frac{1}{2}\right)$ is (see Appendix~\ref{sec:additional-plots}).
  Nevertheless, natural monotone losses such as the ramp and sigmoid losses satisfy the stronger assumption.
\end{remark}

\section{Examples}
\label{sec:examples}

\begin{figure}[t]
  \centering
  \includegraphics[width=0.6\columnwidth]{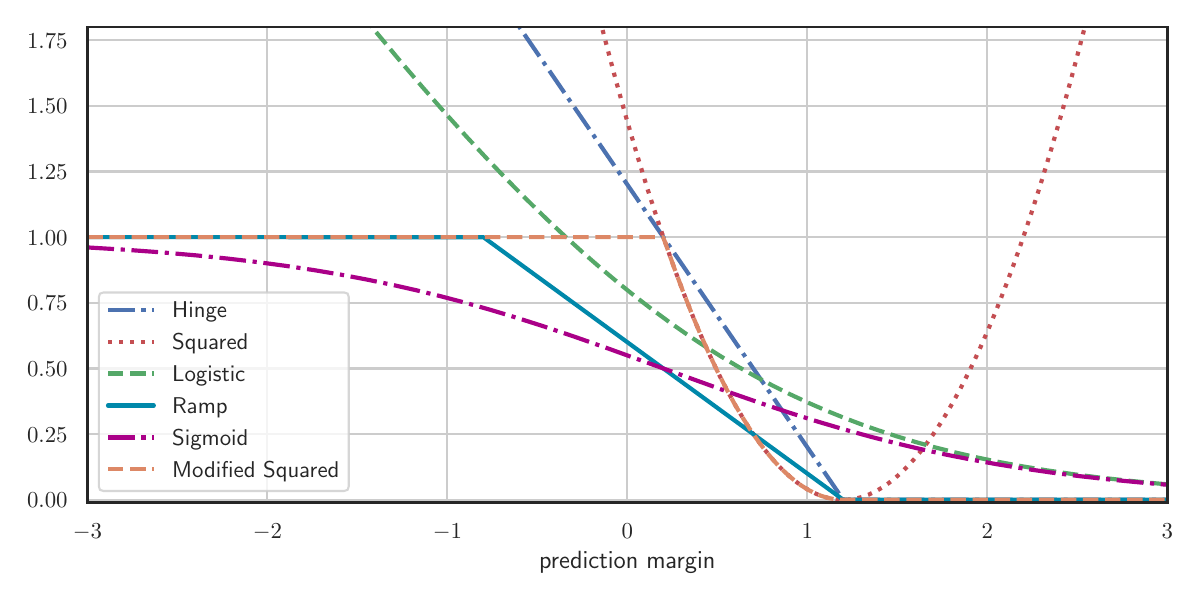}
  \caption{Surrogate losses. They are different from the traditional ones by horizontal translation of $+\beta$ ($\beta = 0.2$ here).}
  \label{fig:loss}
\end{figure}

Several examples of loss functions are shown in Figure~\ref{fig:loss}.
For each base surrogate $\phi$, we consider the shifted surrogate $\phi_\beta(\alpha) \define \phi(\alpha - \beta)$
with the horizontal shift parameter $\beta$.
The ramp, sigmoid, modified squared losses are examples of nonconvex losses satisfying Assumption~\ref{assump:qccv-ccr} when $\beta \geq 0$,
while the hinge, logistic, and squared losses are examples of convex losses.
We show $(\phi_\gamma,\calFlin)$-calibration functions in this subsection.\footnote{
  In this section, we call $\delta_\rho(\epsilon)$ defined in Lemma~\ref{lem:calibration-iff-condition} as the calibration function
  instead of $\bar\delta(\epsilon, \eta, x)$ with a slight abuse of terminology.
}
As a result, we will see that the ramp, sigmoid, and modified squared losses are calibrated with appropriate shift parameters.\footnote{
  After modifying the definition of calibration (Definition~\ref{def:calibrated-loss}),
  the admissible shift parameters slightly differ from the published version~\citep{Bao:2020b}.
  Specifically, the shift parameter range for the ramp loss has been changed from $0 < \beta <2$,
  and for the modified squared loss from $0 \leq \beta < 1$.
  It remains the same for the sigmoid loss.
}
Detailed derivations of the calibration functions
and the proofs of quasiconcavity are deferred to Appendix~\ref{sec:calibration-function}.

\subsection{Ramp Loss}

\begin{figure}[t]
  \centering
  \scriptsize
  \def\vargamma{0.2}
  \def\varrho{0.5}
  \def\varAz{(1+\vargamma+\varrho-\varbeta)}
  \def\varAo{(\varrho/2)}
  \def\vardeltaz{(\varrho/4)}
  \def\varepsz{(\vardeltaz/\varAz)}
  \subfigure[$0 \leq \beta < 1 - \gamma$][c]{
    \label{fig:calibration-function-ramp-loss:1}
    \def\varbeta{0.4}
    \def\varepsz{\varbeta/(2-\varbeta)}
    \includegraphics{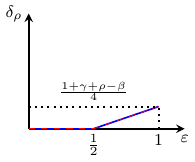}
  } \hfill
  \subfigure[$1 - \gamma \leq \beta < 1 + \gamma$][c]{
    \label{fig:calibration-function-ramp-loss:2}
    \def\varbeta{0.8}
    \includegraphics{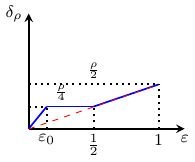}
  } \hfill
  \subfigure[$1 + \gamma \leq \beta < 1 + \gamma + \rho$][c]{
    \label{fig:calibration-function-ramp-loss:3}
    \def\varbeta{1.4}
    \includegraphics{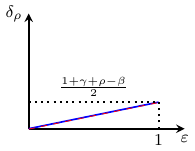}
  } \hfill
  \subfigure[$1+\gamma+\rho \leq \beta$][c]{
    \label{fig:calibration-function-ramp-loss:4}
    \includegraphics{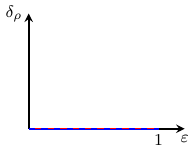}
  }
  \caption{
    The calibration function of the ramp loss.
    $\epsilon_0 \define \frac{\rho}{4(1+\gamma+\rho-\beta)}$.
    The dashed line is $\check\delta_\rho^{**}$.
  }
  \label{fig:calibration-function-ramp-loss}
\end{figure}

The ramp loss is $\phi(\alpha) = \min\left\{1, \max\left\{0, \frac{1-\alpha}{2}\right\}\right\}$.
We consider the shifted ramp loss: $\phi_\beta(\alpha) = \phi(\alpha - \beta) = \min\left\{1, \max\left\{0, \frac{1-\alpha+\beta}{2}\right\}\right\}$.
The $(\phi_\gamma,\calFlin)$-calibration function and its Fenchel-Legendre biconjugate of the ramp loss are plotted in Figure~\ref{fig:calibration-function-ramp-loss}.
We can see that the ramp loss is calibrated wrt ($\phi_\gamma$,$\calFlin$) when $1-\gamma < \beta < 1+\gamma$.
Since the ramp loss satisfies Assumption~\ref{assump:qccv-ccr} when $\beta \geq 0$,
we also observe that the ramp loss is not calibrated when $\beta = 0$
because it is symmetric loss~\citep{Charoenphakdee:2019}, that is, $\phi_0(\alpha) + \phi_0(-\alpha) = 1$ for all $\alpha \in \R$,
which does not satisfy the condition $\phi_0(\gamma) + \phi_0(-\gamma) > \phi_0(\alpha) + \phi_0(-\alpha)$ for all $\alpha \in (\gamma, 1]$ in Theorem~\ref{thm:calibration-condition}.

\subsection{Sigmoid Loss}

\begin{figure}[t]
  \centering
  \scriptsize
  \def\vargamma{0.2}
  \def\varrho{0.5}
  \def\varAz{(1/(1+exp(-\vargamma-\varrho-\varbeta))-1/(1+exp(\vargamma+\varrho-\varbeta)))}
  \def\varAo{(1/(1+exp(\vargamma-\varbeta))-1/(1+exp(-\vargamma-\varbeta))+1/(1+exp(-\vargamma-\varrho-\varbeta))-1/(1+exp(\vargamma+\varrho-\varbeta)))}
  \def\vardeltaz{(0.5*(1/(1+exp(\vargamma-\varbeta))+1/(1+exp(-\vargamma-\varbeta))-1/(1+exp(-\vargamma-\varrho-\varbeta))-1/(1+exp(\vargamma+\varrho-\varbeta))))}
  \def\varepsz{((\vardeltaz)/(\varAz))}
  \subfigure[$\beta = 0$][c]{
    \def\varbeta{0.0}
    \includegraphics{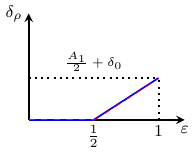}
  } \hfill
  \subfigure[$\beta = 1.0$][c]{
    \def\varbeta{1.0}
    \includegraphics{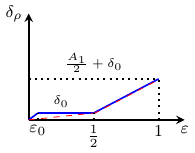}
  } \hfill
  \subfigure[$\beta = 2.0$][c]{
    \def\varbeta{2.0}
    \includegraphics{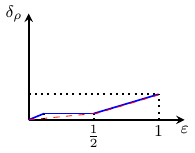}
  } \hfill
  \subfigure[$\beta = 3.0$][c]{
    \def\varbeta{3.0}
    \includegraphics{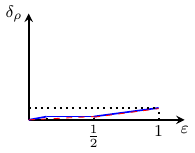}
  }
  \caption{
    The calibration function of the sigmoid loss.
    $A_0 \define \phi_\beta(-\gamma-\rho)-\phi_\beta(\gamma+\rho)$,
    $A_1 \define \phi_\beta(\gamma)-\phi_\beta(-\gamma)-\phi_\beta(\gamma+\rho)+\phi_\beta(-\gamma-\rho)$,
    $\delta_0 \define (\phi_\beta(\gamma)+\phi_\beta(-\gamma)-\phi_\beta(\gamma+\rho)-\phi_\beta(-\gamma-\rho))/2$,
    and $\epsilon_0 \define \frac{\delta_0}{A_0}$.
    The dashed line is $\check\delta_\rho^{**}$.
  }
  \label{fig:calibration-function-sigmoid-loss}
\end{figure}

The sigmoid loss is $\phi(\alpha) = \frac{1}{1 + e^{\alpha}}$.
We consider the shifted sigmoid loss: $\phi_{\beta}(\alpha) = \frac{1}{1 + e^{\alpha - \beta}}$ for $\beta > 0$.
The $(\phi_\gamma,\calFlin)$-calibration function is plotted in Figure~\ref{fig:calibration-function-sigmoid-loss}.
Thus, the sigmoid loss is $(\phi_\gamma,\calFlin)$-calibrated when $\delta_0 > 0$,
which is equivalent to $\beta > 0$.
Again, we observe that the sigmoid loss with $\beta = 0$ is not calibrated in the same way as the ramp loss
because it is symmetric.

\subsection{Modified Squared Loss}

\begin{figure}[t]
  \centering
  \scriptsize
  \def\vargamma{0.5}
  \def\varrho{0.2}
  \def\varAz{(\vargamma+\varrho-\varbeta)*(2+\varbeta-\vargamma-\varrho)}
  \def\varAo{(\varrho*(2+2*\varbeta-2*\vargamma-\varrho))}
  \def\vardeltaz{(\varAo)/2}
  \def\varepsz{(\vardeltaz)/(\varAz)}
  \subfigure[$\beta = 0$][c]{
    \label{fig:calibration-function-modified-squared-loss:0}
    \def\varbeta{0}
    \includegraphics{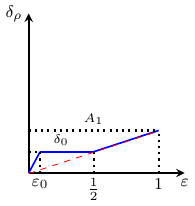}
  } \hfill
  \subfigure[$0 < \beta < \gamma$][c]{
    \label{fig:calibration-function-modified-squared-loss:1}
    \def\varbeta{0.2}
    \includegraphics{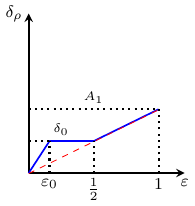}
  } \hfill
  \subfigure[$\gamma \leq \beta < \gamma + \rho$][c]{
    \label{fig:calibration-function-modified-squared-loss:2}
    \def\varbeta{0.6}
    \includegraphics{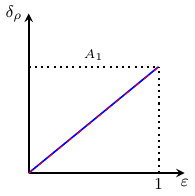}
  } \hfill
  \subfigure[$\gamma + \rho \leq \beta$][c]{
    \label{fig:calibration-function-modified-squared-loss:3}
    \includegraphics{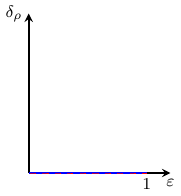}
  }
  \caption{
    The calibration function of the modified squared loss.
    The dashed line is $\check\delta_\rho^{**}$.
    $A_0 \define (\gamma+\rho-\beta)(2+\beta-\gamma-\rho)$,
    $A_1 \define \rho(2+2\beta-2\gamma-\rho)$,
    $\delta_0 \define \frac{A_1}{2}$, and
    $\epsilon_0 \define \frac{\delta_0}{A_0}$.
  }
  \label{fig:calibration-function-modified-squared-loss}
\end{figure}

We make a bounded monotone surrogate
$\phi(\alpha) = \mathrm{clip}_{[0,1]}(\max\{0, 1-\alpha\}^2)$
by modifying the squared loss,
where $\mathrm{clip}_{[a,b]}(\cdot)$ clips values outside the interval $[a,b]$,
and consider the shifted version $\phi_\beta(\alpha) \define \phi(\alpha - \beta)$.
The $(\phi_\gamma,\calFlin)$-calibration function and its Fenchel-Legendre biconjugate are plotted in Figure~\ref{fig:calibration-function-modified-squared-loss}.
We can deduce that the modified squared loss is calibrated wrt ($\phi_\gamma$,$\calFlin$) for all $0 \leq \beta \leq \gamma$.
In contrast to the proceeding examples, the modified squared loss is not symmetric.

\begin{figure}[t]
  \begin{minipage}[c]{0.6\columnwidth}
    \centering
    \scriptsize
    \def\varrho{0.2}
    \def\varAz{(\vargamma+\varrho-\varbeta)*(2+\varbeta-\vargamma-\varrho)}
    \def\varAo{(\varrho*(2+2*\varbeta-2*\vargamma-\varrho))}
    \def\vardeltaz{(\varAo)/2}
    \def\varepsz{(\vardeltaz)/(\varAz)}
    \subfigure[$\beta = -0.1$, $\gamma=0.2$]{
    \includegraphics{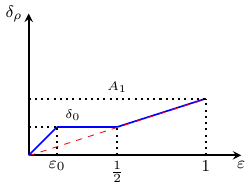}
    } \hfill
    \subfigure[$\beta = -0.2$, $\gamma = 0.2$]{
    \includegraphics{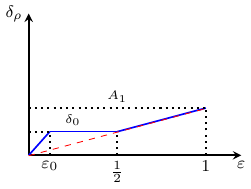}
    }
    \caption{
      The calibration function of the modified squared loss when $\beta<0$.
      $A_0$, $A_1$, $\delta_0$, $\epsilon_0$ are the same as in the caption of Figure~\ref{fig:calibration-function-modified-squared-loss}.
    }
    \label{fig:calibration-function-squared-loss:negative}
  \end{minipage} \hfill
  \begin{minipage}[c]{0.38\columnwidth}
    \centering
    \includegraphics{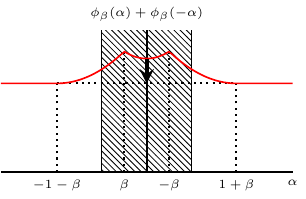}
    \caption{Illustration of $\phi_\beta(\alpha)+\phi_\beta(-\alpha)$ for the modified squared loss when $\gamma \leq 0.4$ and $-1-\gamma+\sqrt{1+2\gamma^2} < \beta < 0$. Here, $\beta = -0.2$ and $\gamma = 0.4$.}
    \label{fig:even-part-mod-squared-loss:negative}
  \end{minipage}
\end{figure}

Moreover, the modified squared loss is $(\phi_\gamma,\calFlin)$-calibrated
even if $\phi_\beta$ for $\beta < 0$ does not satisfy Assumption~\ref{assump:qccv-ccr}.\footnote{
  Indeed, its CCR is not necessarily quasiconcave. See Figure~\ref{fig:ccr-modified-squared-loss:non-qccv} in Appendix~\ref{sec:calibration-function:modified-squared}.
}
We plot two examples in Figure~\ref{fig:calibration-function-squared-loss:negative}.
As seen in the proof sketch of Theorem~\ref{thm:calibration-condition},
it is crucial that $\phi_\beta(\alpha) + \phi_\beta(-\alpha)$ takes higher values in $|\alpha| \leq \gamma$
than in $|\alpha| > \gamma$.
When $\gamma \leq \frac{2}{5}$,
the modified squared loss with $-1-\gamma+\sqrt{1+2\gamma^2} < \beta < 0$
satisfies this property (see Figure~\ref{fig:even-part-mod-squared-loss:negative}).

\subsection{Hinge Loss and Squared Loss}

\begin{figure}[t]
  \centering
  \scriptsize
  \subfigure[Hinge loss.]{
    \includegraphics{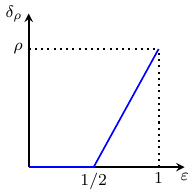}
    \label{fig:calibration-function-hinge-loss}
  }
  \subfigure[Squared loss.]{
    \includegraphics{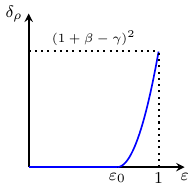}
    \label{fig:calibration-function-squared-loss}
  }
  \caption{
    The calibration functions of the hinge and squared loss.
    $\epsilon_0 \define \frac{1+\beta+\gamma}{2(1+\beta)}$.
  }
  \label{fig:calibration-function-convex-loss}
\end{figure}

Here we consider the shifted hinge loss $\phi_\beta(\alpha) = \max\{0, 1 - \alpha + \beta\}$,
and the shifted squared loss $\phi_\beta(\alpha) = (1 - \alpha + \beta)^2$
as examples of convex losses.
Their $(\phi_\gamma,\calFlin)$-calibration functions are shown in Figure~\ref{fig:calibration-function-convex-loss},
which tell us that the hinge and squared losses are not $(\phi_\gamma,\calFlin)$-calibrated.
This result aligns with Theorem~\ref{thm:no-convex-calibrated-surrogate}.

\section{Calibrated Losses under Low-noise Condition}
\label{sec:low-noise}

In Sections~\ref{sec:convex-surrogate} and \ref{sec:calibrated-surrogate},
we have seen that convex $\phi$ would not be $(\phi_\gamma,\calFlin)$-calibrated
while some nonconvex $\phi$ can be calibrated.
In this section, we will see that convex losses can be $(\phi_\gamma,\calFlin)$-calibrated
under a certain assumption on the conditional distribution.

\begin{assumption}
  \label{assump:massart}
  Let $\xi \in (0, 1)$.
  The conditional distribution satisfies $|2\P(Y=+1|X) - 1| \geq \xi$ almost surely.
\end{assumption}
This assumption is commonly known as Massart's noise condition and we sometimes refer to it as the $\xi$-Massart condition~\citep{Massart:2006}.
With the Massart condition, we further introduce a modified version of $(\phi_\gamma,\calFlin)$-calibrated losses and the calibration function.
\begin{definition}
  \label{def:massart-robustly-calibrated-loss}
  For the robust 0-1 loss $\phi_\gamma$ and a function class $\calF$,
  we say a loss $\phi: \R \to \R_{\geq 0}$ is \emph{$(\phi_\gamma,\calF)$-calibrated under $\xi$-Massart condition}
  if for any $\epsilon > 0$ and $\eta \in [0, 1]$ with $|2\eta-1| \geq \xi$,
  there exists $\delta > 0$ such that for all $f \in \calF$,
  we have
  \begin{align*}
    \calC_{\phi}(f, \eta, x) < \calC_{\phi,\calF}^*(\eta, x) + \delta
    \implies
    \calC_{\phi_\gamma}(f, \eta, x) < \calC_{\phi_\gamma,\calF}^*(\eta, x) + \epsilon.
  \end{align*}

  The corresponding $(\phi_\gamma,\calF)$-calibration function is defined as
  \begin{align*}
    \delta_{\xi}^{\mathrm{Massart}}(\epsilon) = \inf_{\substack{\eta \in [0, 1]\\|2\eta-1| \geq \xi}} \inf_{x \in \calX} \inf_{f \in \calF} \calC_{\phi}(f, \eta, x) - \calC_{\phi,\calF}^*(\eta, x)
    \quad \text{s.t.} \quad
    \calC_{\phi_\gamma}(f, \eta, x) - \calC_{\phi_\gamma,\calF}^*(\eta, x) \geq \epsilon.
  \end{align*}
\end{definition}
As seen in the case of $(\phi_\gamma,\calF)$-calibration (Proposition~\ref{prop:calibration-condition}),
it is necessary and sufficient to check $\delta_{\xi}^{\mathrm{Massart}}(\epsilon) > 0$ for all $\epsilon > 0$,
in order to check $(\phi_\gamma,\calF)$-calibration under $\xi$-Massart condition.

\begin{figure}[t]
  \centering
  \scriptsize
  \subfigure[Hinge loss ($\beta$ = 0).]{
    \includegraphics{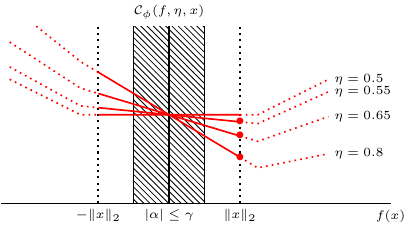}
    \label{fig:ccr-convex-loss:hinge}
  }
  \subfigure[Logistic loss.]{
    \includegraphics{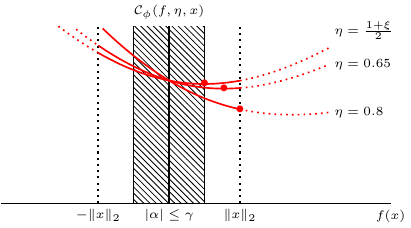}
    \label{fig:ccr-convex-loss:logistic}
  }
  \caption{
    $\phi$-CCR for the hinge and logistic loss with different $\eta$.
    The dots are minimizers of each line.
  }
  \label{fig:ccr-convex-loss}
\end{figure}

Then, we can obtain the positive result for convex losses under the Massart condition.
\begin{theorem}
  \label{thm:convex-loss-calibration}
  Under $\xi$-Massart condition,
  \begin{itemize}
    \item
    the shifted hinge loss $\phi(\alpha) = [1 - \alpha + \beta]_+$ with any shift $\beta \geq 0$ is $(\phi_\gamma,\calFlin)$-calibrated for any $\xi > 0$, and

    \item
    the logistic loss $\phi(\alpha) = \log(1 + e^{-\alpha})$ is $(\phi_\gamma,\calFlin)$-calibrated for $\xi > \tanh\left(\frac{\gamma}{2}\right)$.
  \end{itemize}
\end{theorem}

\begin{proof} \textit{(Sketch)}
  As we see in the proof sketches of Theorems~\ref{thm:no-convex-calibrated-surrogate} and \ref{thm:calibration-condition},
  the suboptimal predictions should be penalized strictly more than the optimal predictions.
  Under $\xi$-Massart condition,
  let us focus on predictions $f(x)$ for $\eta \geq \frac{1+\xi}{2}$.
  Since $f(x)$ spans $[-\xnorm, \xnorm]$ when $f \in \calFlin$ for a fixed $x$,
  the suboptimal predictions are obtained by the infimum of $\phi$-CCR in $f(x) \in [-\xnorm, \gamma]$ ($[-\xnorm, 0]$ is wrong predictions and $(0, \gamma]$ is predictions not robust),
  while the optimal predictions are obtained by the infimum in $f(x) \in [-\xnorm, \xnorm]$.
  Now take a look at Figure~\ref{fig:ccr-convex-loss}.
  Figure~\ref{fig:ccr-convex-loss:hinge} tells us
  that the optimal minimizers of the hinge loss is always $f(x) = \xnorm$ unless $\eta = \frac{1}{2}$.
  On the other hand, we can see that the optimal minimizers of the logistic loss satisfy $f(x) > \gamma$
  if $\eta > \frac{1+\xi}{2}$.
  They are strictly less penalized than the suboptimal minimizers
  thereby calibrated.
\end{proof}

Theorem~\ref{thm:convex-loss-calibration} shows
that surrogate losses could be $(\phi_\gamma,\calFlin)$-calibrated under the Massart condition
even if they are not calibrated for all distributions.

\begin{remark}
  \citet[Theorem~25]{Awasthi:2021} provide a sufficient condition for
  $(\ell_\gamma,\calFlin)$-consistency to hold for $(\phi_\gamma,\calFlin)$-calibrated surrogate loss.
  Their condition assumes $\calR_{\ell_\gamma,\calFlin}^* = 0$.
  Since $\calR_{\ell_{01},\calFlin} \leq \calR_{\ell_\gamma}$, this assumption immediately implies $\calR_{\ell_{01}}^* = 0$,
  which is equivalent to Assumption~\ref{assump:massart} with $\xi = 1$.
  Hence, convex losses lead to $(\ell_\gamma,\calFlin)$-consistency under the assumptions of \citet[Theorem~25]{Awasthi:2021}.
\end{remark}

\section{Simulation}
\label{sec:simulation}

\begin{table}[t]
  \centering
  \small
  \caption{
    The simulation results of the $\gamma$-adversarially robust 0-1 loss with $\gamma = 0.1$ and $\beta = 0.5$.
    50 trials are conducted for each pair of a method and dataset.
    Standard errors (multiplied by $10^4$) are shown in parentheses.
    Bold-faces indicate outperforming methods, chosen by one-sided t-test with the significant level 5\%.
  }
  \label{tab:cropped}
  \begin{tabular}{lllll}
    \toprule
        {} &                 Ramp &             Sigmoid &                Hinge &    Logistic \\
    \midrule
    0 vs 1 &            0.034 (3) &  \textbf{0.017 (2)} &           0.087 (12) &  0.321 (19) \\
    0 vs 2 &   \textbf{0.111 (7)} &          0.133 (10) &   \textbf{0.109 (8)} &  0.281 (19) \\
    0 vs 3 &   \textbf{0.107 (7)} &           0.126 (8) &            0.120 (9) &  0.307 (18) \\
    0 vs 4 &   \textbf{0.069 (6)} &          0.093 (12) &            0.072 (7) &  0.269 (21) \\
    0 vs 5 &  \textbf{0.233 (21)} &          0.340 (25) &  \textbf{0.233 (21)} &  0.269 (16) \\
    0 vs 6 &   \textbf{0.129 (8)} &          0.167 (13) &   \textbf{0.127 (8)} &  0.287 (22) \\
    0 vs 7 &   \textbf{0.067 (6)} &           0.073 (6) &            0.090 (9) &  0.302 (18) \\
    0 vs 8 &   \textbf{0.096 (7)} &          0.123 (12) &            0.100 (9) &  0.263 (20) \\
    0 vs 9 &   \textbf{0.082 (6)} &           0.101 (8) &            0.092 (8) &  0.279 (22) \\
    \bottomrule
  \end{tabular}
\end{table}

\begin{figure}[t]
  \centering
  \small
  \subfigure[Twonorm dataset ($\gamma=\num{0.1}, \beta=\num{0.2}$)]{
    \includegraphics[width=0.48\columnwidth]{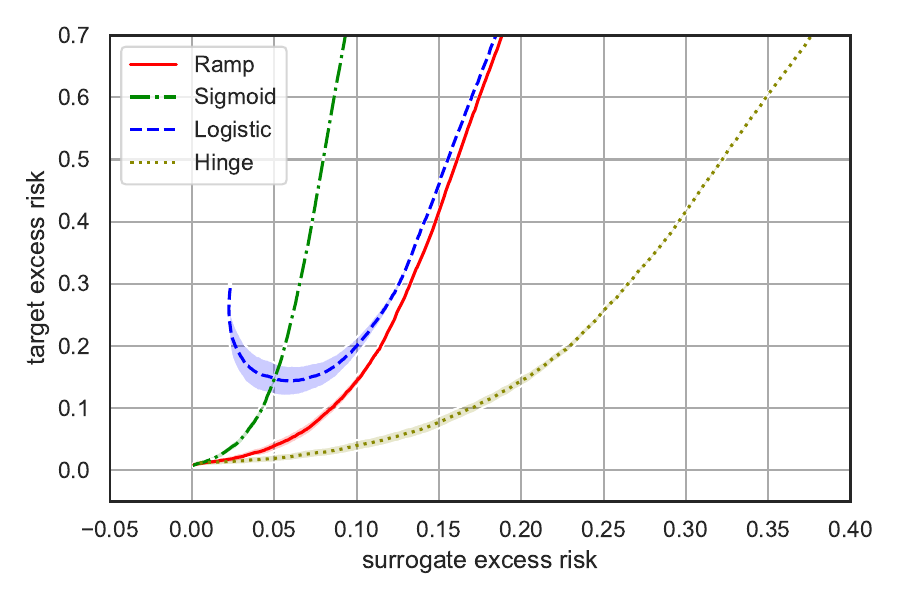}
    \label{fig:result:twonorm}
  }
  \subfigure[Advnorm dataset ($\gamma=\num{0.1}, \beta=\num{0.5}$)]{
    \includegraphics[width=0.48\columnwidth]{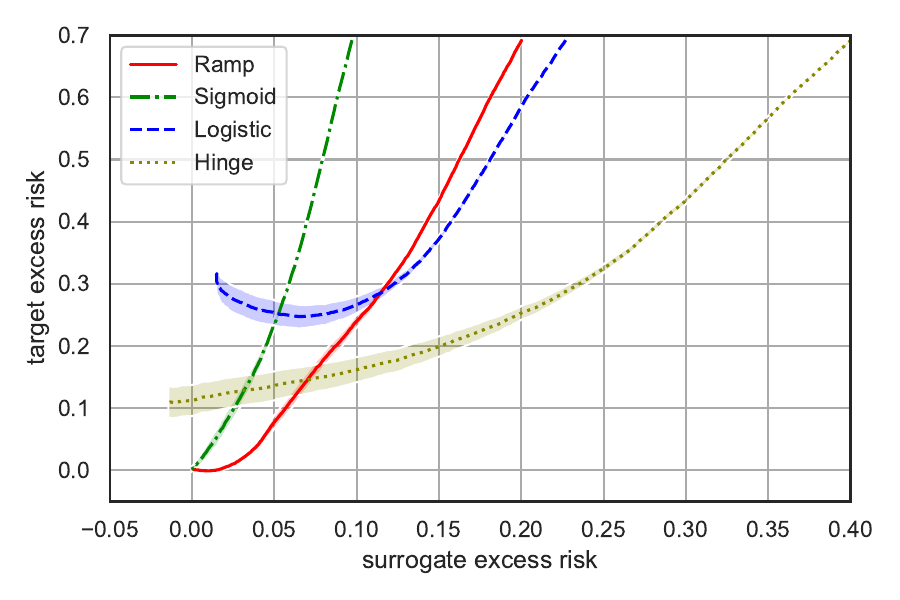}
    \label{fig:result:advnorm}
  }
  \caption{
    Optimization trajectories are shown.
    The horizontal (vertical, resp.) axis shows surrogate excess risk (excess risk of the robust 0-1 loss, resp.) on test data.
  }
  \label{fig:result}
\end{figure}

\begin{figure}[t]
  \centering
  \small
  \subfigure[Twonorm dataset]{
    \includegraphics[width=0.48\columnwidth]{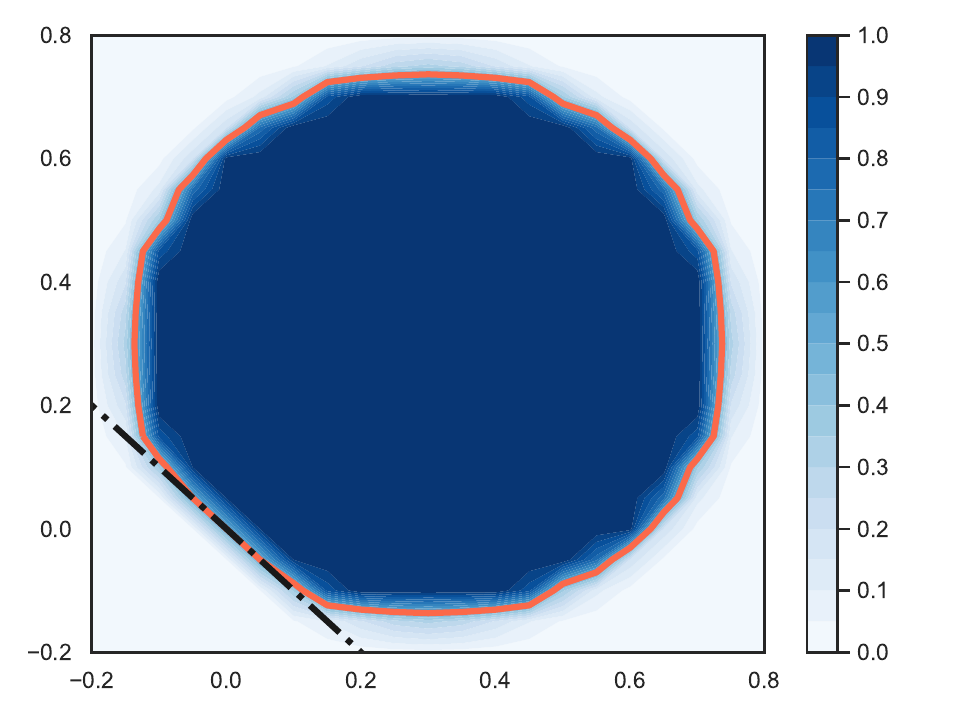}
    \label{fig:posterior:twonorm}
  }
  \subfigure[Advnorm dataset]{
    \includegraphics[width=0.48\columnwidth]{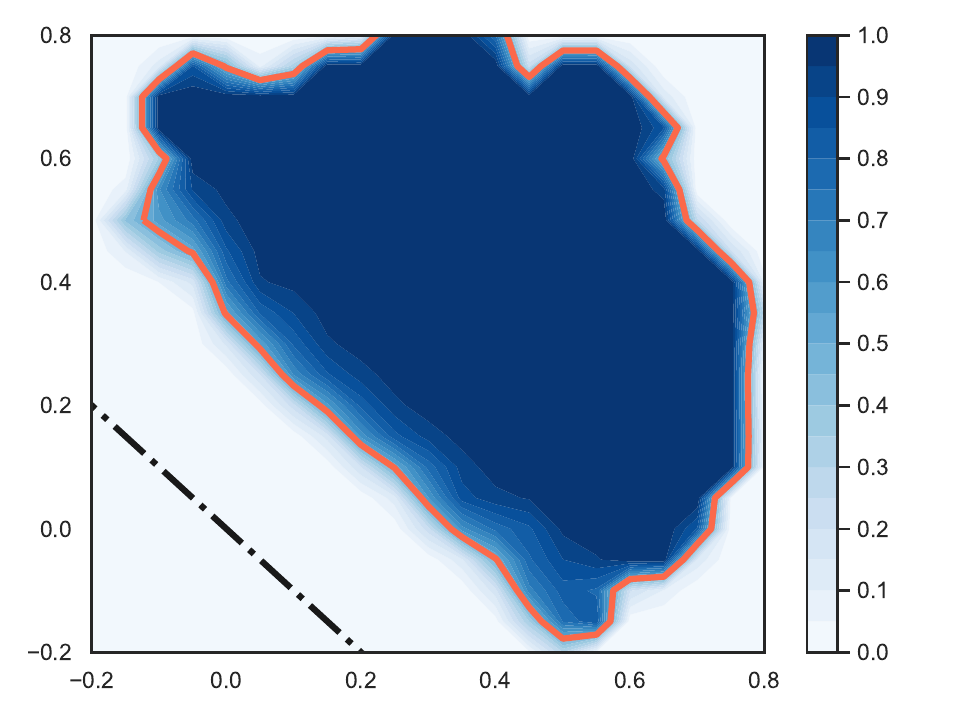}
    \label{fig:posterior:advnorm}
  }
  \caption{
    The estimated posterior distributions $\P(Y=+1|x)$ are plotted with the same scale.
    The estimation procedure is described in Appendix~\ref{sec:numerical-approximation-of-bayes-risk}.
    The black dashed line is the Bayes $(\phi_\gamma,\calFlin)$-classifier
    and the red solid line is the contour line for $\P(Y=+1|x) = \num{0.5}$.
    As can be seen, the posterior value changes more gradually around $\num{0.5}$ in Figure~\ref{fig:posterior:advnorm}.
  }
  \label{fig:posterior}
\end{figure}

\paragraph{Learning Curve on Synthetic Data.}

We use two synthetic datasets.
\begin{itemize}
  \item \textbf{Twonorm.}
  Positive data are generated from $\mathcal{N}([0.3\;0.3]^\top,0.1^2I_2)$
  and negative data are generated from $\mathcal{N}(-[0.3\;0.3]^\top,0.1^2I_2)$.
  The class ratio is $\num{0.5}$.
  All data points lie in the $\ell_2$ unit ball with high probability.
  The classifier $\theta = [1/\sqrt{2}\;1/\sqrt{2}]$ achieves $(\phi_\gamma,\calFlin)$-Bayes risk.

  \item \textbf{Advnorm.}
  First, clean positive data are generated from $\mathcal{N}([0.3\;0.3]^\top,0.142^2I_2)$
  and clean negative data are generated from $\mathcal{N}(-[0.3\;0.3]^\top,0.142^2I_2)$.
  Then, labels of $(x, y=+1)$ with $0 < x_1 + x_2 < 0.25$ are flipped to $y=-1$.
  All data points lie in the $\ell_2$ unit ball with high probability.
  The classifier $\theta = [1/\sqrt{2}\;1/\sqrt{2}]$ achieves $(\phi_\gamma,\calFlin)$-Bayes risk.
  This dataset is the same one as we use in the illustration of Figure~\ref{fig:twonorm}.
\end{itemize}
For each dataset, we generate $\num{500}$ training and $\num{500}$ test points.

Linear models $f(x) = \theta^\top x + \theta_0$ are used, where $\theta$ and $\theta_0$ are learnable parameters.
As surrogate losses, we use the ramp, sigmoid, logistic, and hinge losses.
Batch gradient descent with the fixed step size $0.01$ is used in optimization,
and $\num{3000}$ steps are run for each trial.
After every parameter update, the parameters are normalized to ensure $\|[\theta \; \theta_0]^\top\|_2 = 1$.

The robust 0-1 loss is used as the target loss.
The Bayes risk for each surrogate loss and the robust 0-1 loss is numerically computed,
which is used to compute the excess risk.
The detail of numerical approximation of the Bayes risks is explained in Appendix~\ref{sec:simulation-results}.
The surrogate and target excess risks are shown in Figure~\ref{fig:result}.
$\num{20}$ trials are run for each data realization.

As we can see from Figure~\ref{fig:result}, for both twonorm and advnorm, the
optimization trajectories of the calibrated surrogates (the ramp and sigmoid) have target excess risks tending to zero, while the logistic loss fails.
This observation agrees with our theoretical findings in Theorems~\ref{thm:no-convex-calibrated-surrogate} and \ref{thm:calibration-condition} for the logistic loss.
As for the hinge loss, we observe that it achieves the near-optimal target excess risk on twonorm.
This distribution does not satisfy Massart's condition for any $\xi > 0$, which suggests there might be a more general condition that guarantees calibration for certain convex losses.
For advnorm, which does not satisfy Massart's condition, hinge fails to converge to zero target excess risk, most likely because $\P(Y=+1|x)$ changes more smoothly around $\frac{1}{2}$ for advnorm compared to twonorm. (see Figure~\ref{fig:posterior}).

Note again that even if a surrogate loss $\phi$ is $(\phi_\gamma,\calFlin)$-calibrated,
it does not immediately imply $(\ell_\gamma,\calFlin)$-consistency as pointed out by \citet{Awasthi:2021}.
Nonetheless, nonconvex calibrated surrogate losses are useful in practice as illustrated above, and the hinge loss may also perform reasonably when there is not too much noise near the decision boundary.

\paragraph{Benchmark Data.}

We compare the ramp, sigmoid, hinge, and logistic losses on MNIST.
The results are shown in Table~\ref{tab:cropped},
where we see that nonconvex losses, especially the ramp loss, outperform convex losses
in terms of the robust 0-1 loss.
Details and full results appear in Appendix~\ref{sec:simulation-results}.

\section{Conclusion}
\label{sec:conclusion}

Calibration analysis was leveraged to analyze the adversarially robust 0-1 loss. Focusing on the class of linear classifiers, we found that no convex surrogate loss is calibrated wrt the adversarially robust 0-1 loss for general distributions. We also established necessary and sufficient conditions for a certain class of nonconvex surrogate losses to be calibrated wrt the adversarially robust 0-1 loss,
which includes shifted versions of the ramp and sigmoid losses.

\acks{
  HB was supported by JST ACT-I Grant Number JPMJPR18UI.
  CS was supported in part by NSF Grant Number 1838179.
  MS was supported by JST CREST Grant Number JPMJCR18A2.
}

\bibliography{bao}

\newpage
\appendix
\setcounter{footnote}{0}
\renewcommand{\thefootnote}{\roman{footnote}}
\begin{center}
  \fontsize{14.4pt}{20pt}
  \selectfont
  \rule[5pt]{\columnwidth}{1pt} \\
  \textbf{
    <<Appendix>> \\
    Calibrated Surrogate Losses for Adversarially Robust Classification
  } \\
  \rule[0pt]{\columnwidth}{1pt}
\end{center}

\section*{Index}

\begin{description}[itemsep=0pt,topsep=0pt,leftmargin=0pt,font=\normalfont]
  \item[Appendix~\ref{sec:convex-analysis}:] Overview of convex and quasiconvex analysis.
  \item[Appendix~\ref{sec:useful-lemmas}:] Technical lemmas used in Appendix~\ref{sec:proof}.
  \begin{description}[itemsep=0pt,topsep=0pt,itemindent=0pt,leftmargin=30pt,font=\normalfont]
    \item[Lemma~\ref{lem:ccr-property}:] Simplified necessary and sufficient conditions for calibrated losses.
    \item[Lemma~\ref{lem:ccr-qccv}:] Properties of class-conditional risk related to quasiconcavity.
  \end{description}
  \item[Appendix~\ref{sec:proof}:] Proofs.
  \begin{description}[itemsep=0pt,topsep=0pt,itemindent=0pt,leftmargin=30pt,font=\normalfont]
    \item[\S\ref{sec:proof:margin-error-rate}:] Simplification of the robust 0-1 loss $\ell_{\gamma}$ under $\calFlin$ (Proposition~\ref{prop:margin-error-rate}).
    \item[\S\ref{sec:proof:calibration-iff-condition}:] Necessary and sufficient condition of calibration with $\delta_\rho$ (Lemma~\ref{lem:calibration-iff-condition}).
    \item[\S\ref{sec:proof:mer-calibration-function}:] Detailed expression of $\delta_\rho$ (Lemma~\ref{lem:mer-calibration-function}).
    \item[\S\ref{sec:proof:no-convex-calibrated-surrogate}:] Nonexistence of convex calibrated losses (Theorem~\ref{thm:no-convex-calibrated-surrogate}).
    \item[\S\ref{sec:proof:calibration-condition}:] Characterization of calibrated losses via quasiconcavity (Theorem~\ref{thm:calibration-condition}).
    \item[\S\ref{sec:proof:convex-loss-calibration}:] Hinge and logistic losses are calibrated under Massart condition (Theorem~\ref{thm:convex-loss-calibration}).
  \end{description}
  \item[Appendix~\ref{sec:calibration-function}:] Derivation of calibration functions and analysis of quasiconcavity of each loss.
  \begin{description}[itemsep=0pt,topsep=0pt,itemindent=0pt,leftmargin=30pt,font=\normalfont]
    \item[\S\ref{sec:calibration-function:ramp}:] Ramp loss.
    \item[\S\ref{sec:calibration-function:sigmoid}:] Sigmoid loss.
    \item[\S\ref{sec:calibration-function:modified-squared}:] Modified squared loss.
    \item[\S\ref{sec:calibration-function:hinge}:] Hinge loss.
    \item[\S\ref{sec:calibration-function:squared}:] Squared loss.
  \end{description}
  \item[Appendix~\ref{sec:simulation-results}:] Simulation details and results.
  \item[Appendix~\ref{sec:additional-plots}:] A counterexample plot for the necessity of Assumption~\ref{assump:qccv-ccr}.
\end{description}

\section{Convex and Quasiconvex Analysis}
\label{sec:convex-analysis}

This section summarizes basic tools for convex and quasiconvex analysis.

\header{Quasiconvex function}
A function $h: S \to \R$ on a (finite-dimensional) vector space $S$ is said to be \emph{quasiconvex} if for all $x, y \in S$ and $\lambda \in [0, 1]$,
$h(\lambda x + (1 - \lambda) y) \leq \max\{h(x), h(y)\}$.
A function $h$ is said quasiconcave if $-h$ is quasiconvex:
For all $x, y \in S$ and $\lambda \in [0, 1]$, $h(\lambda x + (1 - \lambda)y) \geq \min\{h(x), h(y)\}$.
Intuitively, quasiconvexity relaxes convexity in that a function still preserves `unimodality'
though it loses definite curvature.
There is an equivalent definition (here we only show for quasiconcavity):
$h$ is quasiconcave if every superlevel set $\{x \mid h(x) \geq t\}$ for $t \in \R$ is a convex set~\citep{Boyd:2004}.

\header{Subderivative}
In order to analyze convexity and quasiconvexity, subderivative is a useful tool.
We adopt the Clarke definition of subderivative~\citep{Clarke:1990,Aussel:1994}.
Let $S^*$ be the dual space of $S$ and $\langle \cdot, \cdot \rangle$ be the dual pairing.\footnote{
  For two vector spaces $U$ and $V$ over the same field $F$
  and a bilinear map $\langle\cdot,\cdot,\rangle: U \times V \to F$,
  we say a triple $(U, V, \langle\cdot,\cdot,\rangle)$ is a dual pair
  if there exists $v \in V$ such that $\langle u, v \rangle \ne 0$ for all $u \in U$
  and there exists $u \in U$ such that $\langle u, v \rangle \ne 0$ for all $v \in V$.
  Here, $V$ is called a dual space of $V$,
  and $\langle\cdot,\cdot\rangle$ is called a dual pairing.
}
The (Clarke) subderivative of a lower semicontinuous function $h$ is the operator $\partial h: S \to S^*$ defined for each $x \in S$ such that
\begin{align}
  \partial h(x) \define \{x_* \in S^* \mid \langle x_*, x \rangle \leq h^\circ(x; v) \quad \forall v \in S \},
  \nonumber
\end{align}
where $h^\circ(x; v)$ is the Rockafellar directional derivative (see \citet{Clarke:1990} and \citet{Aussel:1994} for the formal definition).
When $h$ is locally Lipschitz at $x \in S$,
\citet{Clarke:1990} states that this is equivalent to $\partial h(x) = \conv\{ \lim \nabla f(x_i) \mid x_i \to x, x_i \not\in \Upsilon \cup \Omega_h \}$,
where $\conv$ is the convex hull, $\Upsilon$ is any set of measure zero, and $\Omega_h$ is the set of points where $h$ is non-differentiable.

\header{Properties of subderivative}
Several basic properties of subderivatives are shown in \citet[Section~2.3]{Clarke:1990} such as
$\partial(th)(x) = t\partial h(x) \define \{tx_* \mid x_* \in \partial h(x)\}$ (scalar multiples),
$\partial\left(\sum h_i\right)(x) \subseteq \sum \partial h_i(x) \define \left\{\sum x_{i,*} \mid x_{i,*} \in \partial h_i(x)\right\}$ (finite sums),
and $0 \in \partial h(x)$ if $h$ attains a local extrema at $x$.
When $h$ is locally Lipschitz, it clearly holds that $\partial h(x) = \{h'(x)\}$ if $h$ is differentiable at $x$.

\header{Operator monotonicity}
Convex smooth functions have monotonically nondecreasing derivatives.
This can be extended to non-smooth functions via subderivatives.
Let $h: S \to \R$ be a lower semicontinuous function.
Then $h$ is convex if and only if $\partial h: S \to S^*$ is a \emph{monotone} operator~\citep{Aussel:1994},
that is, $\langle y_* - x_*, y - x \rangle \geq 0$ for all $x, y \in \dom(h)$ and $x_* \in \partial h(x), y_* \in \partial h(y)$.
In addition, $h$ is quasiconvex if and only if $\partial h$ is a \emph{quasimonotone} operator~\citep{Aussel:1994},
that is, $\langle x_*, y - x \rangle > 0 \Longrightarrow \langle y_*, y - x \rangle \geq 0$ for all $x, y \in \dom(h)$ and $x_* \in \partial h(x), y_* \in \partial h(y)$.

\section{Useful Lemmas}
\label{sec:useful-lemmas}

The following lemmas are useful in the remaining proofs in Appendix~\ref{sec:proof}.
Their proofs appear in Appendices~\ref{sec:proof:ccr-property} and \ref{sec:proof:ccr-qccv}.

\begin{lemma}
  \label{lem:ccr-property}
  Let $\phi: \R \to \R_{\geq 0}$ be a margin-based loss function
  and $\calF = \calFlin$.
  \begin{enumerate}
    \item \label{lem:ccr-property:eta-symmetry}
    For all $f \in \calF$ and $x \in \calX$,
    $\calC_\phi(f, \eta, x)$ and $\Delta\calC_{\phi,\calF}(f, \eta, x)$ are symmetric about $\eta = \frac{1}{2}$,
    i.e., $\calC_\phi(f, \eta, x) = \calC_\phi(-f, 1 - \eta, x)$ and $\Delta\calC_{\phi,\calF}(f, \eta, x) = \Delta\calC_{\phi,\calF}(-f, 1 - \eta, x)$ for all $\eta \in [0, 1]$.

    \item \label{lem:ccr-property:alpha-symmetry}
    Fix $x \in \calX$.
    When $\eta = \frac{1}{2}$, we have
    \begin{align*}
      \inf_{f \in \calF: |f(x)| \leq \gamma} \Delta\calC_{\phi,\calF}\left(f, \tfrac{1}{2}, x\right)
      &= \inf_{f \in \calF: 0 \leq f(x) \leq \gamma} \Delta\calC_{\phi,\calF}\left(f, \tfrac{1}{2}, x\right) \\
      &= \inf_{f \in \calF: 0 \leq f(x) \leq \gamma} \calC_\phi\left(f, \tfrac{1}{2}, x\right) - \inf_{f \in \calF: f(x) \geq 0} \calC_\phi\left(f, \tfrac{1}{2}, x\right)
      .
    \end{align*}

    \item \label{lem:ccr-property:expanded-calibration-condition}
    A surrogate loss $\phi$ is calibrated wrt ($\phi_\gamma$,$\calF$) if and only if
    \begin{align*}
      \inf_{f \in \calF: |f(x)| \leq \gamma} \calC_\phi\left(f, \tfrac{1}{2}, x\right) &> \inf_{f \in \calF} \calC_\phi\left(f, \tfrac{1}{2}, x\right), \text{ and}\\
      \inf_{f \in \calF: f(x) \leq \gamma} \calC_\phi(f, \eta, x) &> \inf_{f \in \calF} \calC_\phi(f, \eta, x),
    \end{align*}
    for all $\eta \in \left(\tfrac{1}{2}, 1\right]$ and $x \in \calX$ such that $\xnorm > \gamma$.

    \item \label{lem:ccr-property:expanded-calibration-condition:0/1}
    A surrogate loss $\phi$ is calibrated wrt ($\ell_{01}$,$\calF$) if and only if
    \begin{align*}
      \inf_{f \in \calF: f(x) \leq 0} \calC_\phi(f, \eta, x) > \inf_{f \in \calF} \calC_\phi(f, \eta, x),
    \end{align*}
    for all $\eta \in \left(\frac{1}{2}, 1\right]$ and $x \in \calX \setminus \{0\}$.
  \end{enumerate}
\end{lemma}

Note that part~\ref{lem:ccr-property:expanded-calibration-condition:0/1} of Lemma~\ref{lem:ccr-property} can be regarded
as an equivalent condition of classification calibration~\citep[Definition~1]{Bartlett:2006} for $\calFlin$,
while \citet{Bartlett:2006} provides for $\calFall$.

\begin{lemma}
  \label{lem:ccr-qccv}
  Let $\phi: \R \to \R_{\geq 0}$ be a margin-based loss function.
  Let $\bar\calC_\phi(\alpha, \eta) \define \eta\phi(\alpha) + (1-\eta)\phi(-\alpha)$.
  If $\phi$ is bounded, continuous, non-increasing, and satisfies Assumption~\ref{assump:qccv-ccr},
  then
  \begin{enumerate}
    \item \label{lem:ccr-qccv:ccr-decreasing}
    for all $\eta \in \left(\frac{1}{2}, 1\right]$,
    $\bar\calC_\phi(\alpha, \eta)$ is nonincreasing in $\alpha$ when $\alpha \geq 0$.

    \item \label{lem:ccr-qccv:ccr-limit}
    for all $\eta \in \left(\frac{1}{2}, 1\right]$ and $\alpha > 0$,
    $\bar\calC_\phi(-\alpha, \eta) > \bar\calC_\phi(\alpha, \eta)$ if $\phi(-\alpha) > \phi(\alpha)$.

    \item \label{lem:ccr-qccv:even-decreasing}
    $\phi(\alpha) + \phi(-\alpha)$ is nonincreasing in $\alpha$ when $\alpha \geq 0$.

    \item \label{lem:ccr-qccv:ccr-inf}
    for $l, u \in \R$ ($l \leq u$),
    $\inf_{\alpha \in [l, u]} \bar\calC_\phi(\alpha, \eta) = \min\{\bar\calC_\phi(l, \eta), \bar\calC_\phi(u, \eta)\}$
    for all $\eta \in [0, 1]$.
  \end{enumerate}
\end{lemma}

\section{Deferred Proofs}
\label{sec:proof}

\subsection{Proof of Proposition~\ref{prop:margin-error-rate}}
\label{sec:proof:margin-error-rate}

\begin{proof}
  Fix $(x, y) \in \calX \times \calY$ and $f \in \calFlin$ associated with parameter $\theta \in \R^d$.
  Since we can prove the case $y = -1$ in the same manner,
  assume $y = +1$ below without loss of generality.

  We will check the existence of $\Delta_x \in B_2(\gamma)$
  such that $\theta^\top(x + \Delta_x) \leq 0$ and $x + \Delta_x \in \calX$,
  depending on the value $\theta^\top x$.
  If $\theta^\top x \leq 0$, the trivial choice $\Delta_x = 0$ satisfies $\theta^\top(x + \Delta_x) \leq 0$.

  If $0 < \theta^\top x \leq \gamma$, the choice $\Delta_x \define -(\theta^\top x)\theta$ satisfies them:
  $\|\Delta_x\|_2 = \theta^\top x \leq \gamma$ implies $\Delta_x \in B_2(\gamma)$,
  $\theta^\top(x + \Delta_x) = \theta^\top x - \theta^\top x = 0$,
  and $\|x + \Delta_x\|_2^2 = \|x\|^2 - \left(\theta^\top x\right)^2 \leq \xnorm^2 \leq 1$ implies $x + \Delta_x \in \calX$.

  If $\theta^\top x > \gamma$, we can check $\theta^\top(x + \Delta_x) > 0$ for any $\Delta_x \in B_2(\gamma)$.
  We consider the convex optimization problem
  $\min_{\Delta_x \in B_2(\gamma)} \theta^\top(x + \Delta_x)$.
  Consider the Lagrangian
  \begin{align*}
    \calL(\Delta_x, \mu) \define \theta^\top(x + \Delta_x) + \mu(\|\Delta_x\|_2 - \gamma),
  \end{align*}
  where $\mu \in \R$ is a KKT multiplier.
  Its KKT conditions are
  \begin{align*}
    \begin{cases}
      -\theta = \mu\frac{\Delta_x}{\|\Delta_x\|_2}, \\
      \|\Delta_x\|_2 \leq \gamma, \\
      \mu \geq 0, \\
      \mu(\|\Delta_x\|_2 - \gamma) = 0.
    \end{cases}
  \end{align*}
  The objective is minimized when the constraint $\|\Delta_x\|_2 \leq \gamma$ shall be activated,
  where the multiplier $\mu > 0$ and $\Delta_x = -\frac{\gamma}{\mu}\theta$,
  meaning that $\Delta_x$ is parallel to $\theta$ in the opposite direction.
  Hence, $\Delta_x = -\gamma\theta$ is the minimizer.
  We have $\theta^\top(x + \Delta_x) = \theta^\top x - \gamma > 0$
  with this minimizer $\Delta_x$.

  By combining the three cases, we have $\ell_\gamma(+1, x, f) = \ind{\theta^\top x \leq \gamma}$.
  \end{proof}

\subsection{Proof of Lemma~\ref{lem:calibration-iff-condition}}
\label{sec:proof:calibration-iff-condition}

\begin{proof}
  By Proposition~\ref{prop:calibration-condition},
  we need to show the following conditions are equivalent.
  \begin{enumerate}
    \renewcommand{\theenumi}{(\roman{enumi})}
    \renewcommand{\labelenumi}{\theenumi}
    \item \label{enum:calibration-iff:calibration-condition}
    For all $\epsilon > 0$, $\eta \in [0, 1]$, and $x \in \calX$, $\bar\delta(\epsilon, \eta, x) > 0$.

    \item \label{enum:calibration-iff:modified-condition}
    For all $\epsilon > 0$ and $\rho \in (0, 1-\gamma)$, $\delta_\rho(\epsilon) > 0$.
  \end{enumerate}

  From \eqref{eq:robust-01-inner-excess-risk} in the proof of Lemma~\ref{lem:mer-calibration-function},
  we have $\Delta\calC_{\phi_\gamma,\calFlin}(f, \eta, x) = 0$ for $x$ with $\xnorm \leq \gamma$.
  This means that the constraint $\Delta\calC_{\phi_\gamma,\calFlin}(f, \eta, x) \geq \epsilon$ in $\bar\delta$ would never be satisfied for $\epsilon > 0$,
  where the infimum value of $\bar\delta(\epsilon, \eta, x) = \infty$ for all $\epsilon > 0$, $\eta \in [0, 1]$.
  Note that
  \begin{align*}
    \delta_\rho(\epsilon)
    &= \inf_{\eta \in [0, 1]} \inf_{\xnorm \geq \gamma + \rho} \inf_{f \in \calFlin} \Delta\calC_{\phi,\calFlin}(f, \eta, x)
    \quad \text{s.t.} \quad
    \Delta\calC_{\phi_\gamma,\calFlin}(f, \eta, x) \geq \epsilon \\
    &= \inf_{\eta \in [0, 1]} \inf_{\xnorm \geq \gamma + \rho} \bar\delta(\epsilon, \eta, x).
  \end{align*}

  For \ref{enum:calibration-iff:calibration-condition} $\Rightarrow$ \ref{enum:calibration-iff:modified-condition},
  let $\tcalX_\rho \define \calX \setminus B_2^\circ(\gamma + \rho) = \{x \in \calX \mid \xnorm \geq \gamma + \rho\}$.
  For a fixed $\epsilon > 0$,
  the extreme value theorem states that $\delta_\rho(\epsilon) = \bar\delta(\epsilon, \eta_\epsilon, x_\epsilon)$
  for some $(\eta_\epsilon, x_\epsilon) \in [0,1] \times \tcalX_\rho$,
  by noting that $\bar\delta(\epsilon, \cdot, \cdot): [0, 1] \times \tcalX_\rho \to \R_{\geq 0}$ is continuous and its domain $[0, 1] \times \tcalX_\rho$ is compact.
  Indeed, $\bar\delta(\epsilon, \cdot, \cdot)$ is continuous
  because it is the infimum function of a continuous function over a compact set (see \eqref{eq:inner-calibration-function} in Lemma~\ref{lem:mer-calibration-function}).
  Eventually, we have $\delta_\rho(\epsilon) \geq \bar\delta(\epsilon, \eta_\epsilon, x_\epsilon) > 0$ by using \ref{enum:calibration-iff:calibration-condition}.

  Subsequently, we check \ref{enum:calibration-iff:modified-condition} $\Rightarrow$ \ref{enum:calibration-iff:calibration-condition}.
  The condition \ref{enum:calibration-iff:modified-condition} implies that
  $\bar\delta(\epsilon, \eta, x) \geq \delta_\rho(\epsilon) > 0$ for all $\epsilon > 0$, $\eta \in [0, 1]$, and $x \in \calX$ with $\xnorm > \gamma$.
  Together with $\bar\delta(\epsilon, \eta, x) = \infty$ for all $\epsilon > 0$, $\eta \in [0, 1]$, and $x \in \calX$ with $\xnorm \leq \gamma$,
  \ref{enum:calibration-iff:calibration-condition} is assured.
\end{proof}

\subsection{Proof of Lemma~\ref{lem:mer-calibration-function}}
\label{sec:proof:mer-calibration-function}

\begin{proof}
  We first simplify the constraint in the calibration function~\eqref{eq:calibration-function}.
  The $\phi_\gamma$-CCR for $f \in \calFlin$ at $x$ is
  \begin{align}
    \calC_{\phi_\gamma}(f, \eta, x)
    = \eta\ind{f(x) \leq \gamma} + (1 - \eta)\ind{f(x) \geq -\gamma}
    = \begin{cases}
      1 & \text{if $|f(x)| \leq \gamma$}, \\
      1 - \eta & \text{if $\gamma < f(x)$}, \\
      \eta & \text{if $f(x) < -\gamma$}.
    \end{cases}
    \label{eq:robust-01-ccr}
  \end{align}
  To compute the minimal ($\phi_\gamma$,$\calFlin$)-CCR,
  we divide into two cases.
  If $\xnorm \leq \gamma$, $\calC_{\phi_\gamma}(f, \eta, x) = 1$ for any $f \in \calFlin$
  because $|f(x)| \leq \gamma$.
  Thus, we have $\calC_{\phi_\gamma,\calFlin}^*(\eta, x) = 1$
  and $\Delta\calC_{\phi_\gamma,\calFlin}(f, \eta, x) = 0$.
  If $\xnorm > \gamma$, there exists $f \in \calFlin$ such that $\calC_{\phi_\gamma}(f, \eta, x) = \min\{\eta, 1 - \eta\}$.
  Thus, we have $\calC_{\phi_\gamma,\calFlin}^*(\eta, x) = \min\{\eta, 1 - \eta\}$.
  This implies that
  \begin{align*}
    \Delta\calC_{\phi_\gamma,\calFlin}(f, \eta, x) = \begin{cases}
      \max\{\eta, 1 - \eta\} & \text{if $|f(x)| \leq \gamma$}, \\
      |2\eta - 1| \cdot \ind{(2\eta - 1)f(x) \leq 0} & \text{if $\gamma < |f(x)|$}.
    \end{cases}
  \end{align*}
  Note that the latter case is obtained in the same manner as \citet[Proof of Theorem 3]{Bartlett:2006}.
  To sum it up, we obtain the expression of $\Delta\calC_{\phi_\gamma,\calFlin}$ as
  \begin{align}
    \Delta\calC_{\phi_\gamma,\calFlin}(f, \eta, x) = \begin{cases}
      0 & \text{if $\xnorm \leq \gamma$}, \\
      \max\{\eta, 1 - \eta\} & \text{if $\xnorm > \gamma$ and $|f(x)| \leq \gamma$}, \\
      |2\eta - 1| \cdot \ind{(2\eta-1)f(x) \leq 0} & \text{if $\xnorm > \gamma$ and $\gamma < |f(x)|$}.
    \end{cases}
    \label{eq:robust-01-inner-excess-risk}
  \end{align}

  Next, we simplify the infimum on $f$, $\inf_{f \in \calFlin}\{\Delta\calC_{\phi,\calFlin}(f, \eta, x) \mid \Delta\calC_{\phi_\gamma,\calFlin}(f, \eta, x) \geq \epsilon\} = \bar\delta(\epsilon, \eta, x)$ in \eqref{eq:calibration-function},
  for a fixed $\eta \in [0, 1]$ and $x \in \calX$.
  If $\xnorm \leq \gamma$ or $\epsilon > \max\{\eta, 1 - \eta\}$,
  no $f \in \calFlin$ achieves $\Delta\calC_{\phi_\gamma,\calFlin}(f, \eta, x) \geq \epsilon$,
  meaning that $\bar\delta(\epsilon, \eta, x) = \infty$.
  If $\xnorm > \gamma$ and $|2\eta - 1| < \epsilon \leq \max\{\eta, 1 - \eta\}$, $\Delta\calC_{\phi_\gamma,\calFlin}(f, \eta, x) \geq \epsilon$ is achieved when $|f(x)| \leq \gamma$.
  Hence, $\bar\delta(\epsilon, \eta, x) = \inf_f\{\Delta\calC_{\phi,\calFlin}(f, \eta, x) \mid |f(x)| \leq \gamma\}$.
  Note that $|2\eta - 1| \leq \max\{\eta, 1 - \eta\} = \frac{1 + |2\eta - 1|}{2}$ for all $\eta \in [0, 1]$.
  If $\xnorm > \gamma$ and $\epsilon \leq |2\eta - 1|$, $\Delta\calC_{\phi,\calFlin}(f, \eta, x) \geq \epsilon$ is achieved if either $|f(x)| \leq \gamma$ or $(2\eta - 1)f(x) \leq 0$ holds.
  Hence, $\bar\delta(\epsilon, \eta, x) = \inf_f\{\Delta\calC_{\phi,\calFlin}(f, \eta, x) \mid |f(x)| \leq \gamma \text{ or } (2\eta - 1)f(x) \leq 0\}$.
  These verify the statement of this lemma.
\end{proof}

\subsection{Proof of Theorem~\ref{thm:no-convex-calibrated-surrogate}}
\label{sec:proof:no-convex-calibrated-surrogate}

\begin{proof}
  Part~\ref{lem:ccr-property:expanded-calibration-condition} of Lemma~\ref{lem:ccr-property} states
  that $\phi$ is calibrated wrt ($\phi_\gamma$,$\calFlin$) if and only if
  \begin{align*}
    \inf_{f \in \calFlin: 0 \leq f(x) \leq \gamma} \calC_\phi\left(f, \tfrac{1}{2}, x\right) &> \inf_{f \in \calFlin: f(x) \geq 0} \calC_\phi\left(f, \tfrac{1}{2}, x\right)
    \quad \text{and} \\
    \inf_{f \in \calFlin: f(x) \leq \gamma} \calC_\phi(f, \eta, x) &> \inf_{f \in \calFlin: f(x) \geq 0} \calC_\phi(f, \eta, x)
    \quad \text{for any $\eta \in \left(\tfrac{1}{2}, 1\right]$},
  \end{align*}
  for all $x \in \calX$ such that $\xnorm > \gamma$.
  In order to show $\phi$ is not calibrated wrt ($\phi_\gamma$,$\calFlin$),
  it is sufficient to show the existence of $x \in \calX$ such that $\xnorm > \gamma$ and
  \begin{align*}
    & \inf_{f \in \calFlin: 0 \leq f(x) \leq \gamma} \calC_\phi\left(f, \tfrac{1}{2}, x\right) = \inf_{f \in \calFlin: f(x) \geq 0} \calC_\phi\left(f, \tfrac{1}{2}, x\right),
  \end{align*}
  which is equivalent to
  \begin{align}
    \inf_{f \in \calFlin: 0 \leq f(x) \leq \gamma} \phi(f(x)) + \phi(-f(x)) = \inf_{f \in \calFlin: f(x) \geq 0} \phi(f(x)) + \phi(-f(x)).
    \label{eq:sufficient-condition-for-no-convex-surrogate}
  \end{align}
  Since $\bar\phi(\alpha) \define \phi(\alpha) + \phi(-\alpha)$ is a convex even function,
  we have $\bar\phi(0) \leq \bar\phi(\alpha)$ for all $\alpha \in \R$.
  To see this, assume that there exists $\alpha_* \in \R$ such that $\alpha_* \ne 0$ and $\bar\phi(0) > \bar\phi(\alpha_*)$.
  Then, we also have $\bar\phi(-\alpha_*) < \bar\phi(0)$ since $\bar\phi$ is even.
  It follows that $\frac{1}{2}\{\bar\phi(-\alpha_*) + \bar\phi(\alpha_*)\} < \bar\phi(0)$.
  However, we have $\frac{1}{2}\{\bar\phi(-\alpha_*) + \bar\phi(\alpha_*)\} \geq \bar\phi\left(\frac{-\alpha_*+\alpha_*}{2}\right) = \bar\phi(0)$
  because of the convexity of $\bar\phi$.
  Hence, we see $\bar\phi(0) \leq \bar\phi(\alpha)$ for all $\alpha \in \R$.
  This means that $\inf_{0 \leq \alpha \leq \gamma} \bar\phi(\alpha) = \inf_{\alpha \in \calA: 0 \leq \alpha} \bar\phi(\alpha) = \bar\phi(0)$,
  where $\calA$ is any subset of the real line containing $0$.
  Note that for $x \in \calX$ such that $\xnorm > \gamma$,
  $f(x)$ ranges $[-\xnorm, \xnorm] \supseteq [-\gamma, \gamma]$ with $f \in \calFlin$.
  Therefore, for any choice of a fixed $x \in \calX$ such that $\xnorm > \gamma$,
  we have $\inf_{f \in \calFlin: 0 \leq f(x) \leq \gamma} \bar\phi(f(x)) = \inf_{0 \leq \alpha \leq \gamma} \bar\phi(\alpha)$
  and $\inf_{f \in \calFlin: f(x) \geq 0} \bar\phi(f(x)) = \inf_{0 \leq \alpha \leq \xnorm} \bar\phi(\alpha)$.
  This implies that the sufficient condition \eqref{eq:sufficient-condition-for-no-convex-surrogate} for nonexistence of convex surrogate losses holds.
\end{proof}

\subsection{Proof of Theorem~\ref{thm:calibration-condition}}
\label{sec:proof:calibration-condition}

Let $\bar\calC_\phi(\alpha, \eta) \define \eta\phi(\alpha) + (1-\eta)\phi(-\alpha)$.

\begin{proof}\textbf{of part~\ref{thm:calibration-condition:0/1}}
  By part~\ref{lem:ccr-property:expanded-calibration-condition:0/1} of Lemma~\ref{lem:ccr-property},
  $(\ell_{01},\calFlin)$-calibration is equivalent to
  \begin{align}
    \inf_{f \in \calFlin: f(x) \leq 0} \calC_\phi(f, \eta, x) > \inf_{f \in \calFlin} \calC_\phi(f, \eta, x)
    \quad \text{for all $\eta \in \left(\tfrac{1}{2}, 1\right]$ and $x \in \calX \setminus \{0\} $}.
    \label{eq:expanded-0/1-calibration-condition}
  \end{align}

  Fix an arbitrary $\eta$ such that $\frac{1}{2} < \eta \leq 1$ and $x \in \calX \setminus \{0\}$.
  We observe with part~\ref{lem:ccr-qccv:ccr-inf} of Lemma~\ref{lem:ccr-qccv} that
  \begin{align*}
    \inf_{f \in \calFlin: f(x) \leq 0} \calC_\phi(f, \eta, x)
    &= \inf_{\alpha \in [-\xnorm, 0]} \bar\calC_\phi(\alpha, \eta) \\
    &= \min\{\bar\calC_\phi(-\xnorm, \eta), \bar\calC_\phi(0, \eta)\}
    && \text{(part~\ref{lem:ccr-qccv:ccr-inf} of Lemma~\ref{lem:ccr-qccv})} \\
    &= \min\{\eta\overline{B} + (1-\eta)\underline{B}, \phi(0)\},
  \end{align*}
  and
  \begin{align*}
    \inf_{f \in \calFlin} \calC_\phi(f, \eta, x)
    &= \inf_{\alpha \in [-\xnorm, \xnorm]} \bar\calC_\phi(\alpha, \eta) \\
    &= \min\{\bar\calC_\phi(-\xnorm, \eta), \bar\calC_\phi(\xnorm, \eta)\}
    && \text{(part~\ref{lem:ccr-qccv:ccr-inf} of Lemma~\ref{lem:ccr-qccv})} \\
    &= \bar\calC_\phi(\xnorm, \eta)
    && \text{(part~\ref{lem:ccr-qccv:ccr-limit} of Lemma~\ref{lem:ccr-qccv})} \\
    &= \eta\underline{B} + (1 - \eta)\overline{B},
  \end{align*}
  where $\underline{B} \define \phi(\xnorm)$ and $\overline{B} \define \phi(-\xnorm)$.
  Note that $\overline{B} > \underline{B}$ since $\xnorm > 0$.
  Here,
  \begin{align*}
    \bar\calC_\phi(-\xnorm, \eta) - \bar\calC_\phi(\xnorm, \eta)
    &= (\overline{B} - \underline{B})(2\eta-1)
    > 0,
    \\
    \bar\calC_\phi(0, \eta) - \bar\calC_\phi(\xnorm, \eta)
    &= \phi(0)-\overline{B}+\eta(\overline{B}-\underline{B}) \\
    &\geq \frac{\overline{B} + \underline{B}}{2} - \overline{B} + \eta(\overline{B} - \underline{B}) \\
    &> \frac{\overline{B} + \underline{B}}{2} - \overline{B} + \frac{\overline{B} - \underline{B}}{2}
    && \text{($\overline{B} > \underline{B}$ and $\eta > \tfrac{1}{2}$)} \\
    &= 0,
  \end{align*}
  where the first inequality is shown by quasiconcavity of $\alpha \mapsto \phi(\alpha) + \phi(-\alpha)$.
  Indeed, by letting $F(\alpha) \define \phi(\alpha) + \phi(-\alpha)$,
  \begin{align*}
    2\phi(0) = F(0)
    = F\left(\frac{\xnorm}{2} + \frac{-\xnorm}{2}\right)
    \geq \min\{F(\xnorm), F(-\xnorm)\}
    = F(\xnorm)
    = \overline{B} + \underline{B}.
  \end{align*}
  Then, we have
  \begin{align*}
    \inf_{f \in \calFlin: f(x) \leq 0} & \calC_\phi(f, \eta, x) - \inf_{f \in \calFlin} \calC_\phi(f, \eta, x) \\
    &= \min\{
      \bar\calC_\phi(-\xnorm, \eta) - \bar\calC_\phi(\xnorm, \eta),
      \bar\calC_\phi(0, \eta) - \bar\calC_\phi(\xnorm, \eta)
    \} \\
    &> 0.
  \end{align*}
  This verifies the condition \eqref{eq:expanded-0/1-calibration-condition}.
\end{proof}

\begin{proof}\textbf{of part~\ref{thm:calibration-condition:robust-0/1}}
  $\phi$ is calibrated wrt ($\phi_\gamma$,$\calFlin$)
  if and only if
  \begin{align}
    \begin{aligned}
      & \text{(i)} \quad
      \inf\limits_{f \in \calFlin: |f(x)| \leq \gamma} \calC_\phi\left(f, \tfrac{1}{2}, x\right) &&> \inf\limits_{f \in \calFlin} \calC_\phi\left(f, \tfrac{1}{2}, x\right), \quad \text{and} \\
      & \text{(ii)} \quad
      \inf\limits_{f \in \calFlin: f(x) \leq \gamma} \calC_\phi(f, \eta, x) &&> \inf\limits_{f \in \calFlin} \calC_\phi(f, \eta, x)
      \quad \text{for all $\eta \in \left(\tfrac{1}{2}, 1\right]$}
    \end{aligned}
    \label{eq:expanded-calibration-condition}
  \end{align}
  for any $x \in \calX$ such that $\xnorm > \gamma$,
  by part~\ref{lem:ccr-property:expanded-calibration-condition} of Lemma~\ref{lem:ccr-property}.
  Now we show $\phi(\gamma) + \phi(-\gamma) > \phi(\alpha) + \phi(-\alpha)$ for any $\alpha \in (\gamma, 1]$,
  assuming (i) and (ii).
  For an arbitrary $\alpha \in (\gamma, 1]$,
  pick an $x$ such that $\xnorm = \alpha$,
  then $\{f(x) \mid f \in \calFlin\}$ ranges $[-\alpha, \alpha]$.
  \begin{align*}
    \phi(\gamma) + \phi(-\gamma)
    &= \inf_{0 \leq \alpha' \leq \gamma} \phi(\alpha') + \phi(-\alpha')
    && \text{(part~\ref{lem:ccr-qccv:even-decreasing} of Lemma~\ref{lem:ccr-qccv})} \\
    &= \inf_{f \in \calFlin: 0 \leq f(x) \leq \gamma} \phi(f(x)) + \phi(-f(x)) \\
    &= \inf_{f \in \calFlin: |f(x)| \leq \gamma} \phi(f(x)) + \phi(-f(x))
    && \text{($\phi(\alpha) + \phi(-\alpha)$ is even)} \\
    &> \inf_{f \in \calFlin} \phi(f(x)) + \phi(-f(x))
    && \text{((i) is used)} \\
    &= \inf_{-\xnorm \leq \alpha' \leq \xnorm} \phi(\alpha') + \phi(-\alpha') \\
    &= \inf_{0 \leq \alpha' \leq \alpha} \phi(\alpha') + \phi(-\alpha')
    && \text{($\phi(\alpha) + \phi(-\alpha)$ is even)} \\
    &= \phi(\alpha) + \phi(-\alpha).
    && \text{(part~\ref{lem:ccr-qccv:even-decreasing} of Lemma~\ref{lem:ccr-qccv})} \\
  \end{align*}

  Conversely, assume $\phi(\gamma) + \phi(-\gamma) > \phi(\alpha) + \phi(-\alpha)$ for any $\alpha \in (\gamma, 1]$.
  We will show (i) and (ii) in \eqref{eq:expanded-calibration-condition}.
  Fix an $x \in \calX$ such that $\xnorm > \gamma$ arbitrarily,
  then $\{f(x) \mid f \in \calFlin\}$ ranges $[-\xnorm, \xnorm] \supseteq [-\gamma, \gamma]$.
  Since $\phi(\alpha) + \phi(-\alpha)$ is nonincreasing in $\alpha \geq 0$ (part~\ref{lem:ccr-qccv:even-decreasing} of Lemma~\ref{lem:ccr-qccv}),
  we have
  \begin{align*}
    2\inf_{f \in \calFlin: |f(x)| \leq \gamma} \calC_\phi\left(f, \tfrac{1}{2}, x\right)
    &= \inf_{|\alpha| \leq \gamma} \phi(\alpha) + \phi(-\alpha)
    && \text{($f(x) \in [-\xnorm, \xnorm]$)} \\
    &= \inf_{0 \leq \alpha \leq \gamma} \phi(\alpha) + \phi(-\alpha)
    && \text{($\phi(\alpha) + \phi(-\alpha)$ is even)} \\
    &= \phi(\gamma) + \phi(-\gamma)
    && \text{(part~\ref{lem:ccr-qccv:even-decreasing} of Lemma~\ref{lem:ccr-qccv})} \\
    &> \phi(\xnorm) + \phi(-\xnorm)
    && \text{(by assumption)}\\
    &= \inf_{0 \leq \alpha \leq \xnorm} \phi(\alpha) + \phi(-\alpha),
    && \text{(part~\ref{lem:ccr-qccv:even-decreasing} of Lemma~\ref{lem:ccr-qccv})} \\
    &= \inf_{-\xnorm \leq \alpha \leq \xnorm} \phi(\alpha) + \phi(-\alpha)
    && \text{($\phi(\alpha) + \phi(-\alpha)$ is even)} \\
    &= 2\inf_{f \in \calFlin} \calC_\phi\left(f, \tfrac{1}{2}, x\right),
  \end{align*}
  which is equivalent to (i).
  For (ii), fix an $\eta$ such that $\frac{1}{2} < \eta \leq 1$.
  We first observe with parts~\ref{lem:ccr-qccv:ccr-limit} and \ref{lem:ccr-qccv:ccr-inf} of Lemma~\ref{lem:ccr-qccv} that
  \begin{align*}
    \inf_{-\xnorm \leq \alpha \leq \gamma} \bar\calC_\phi(\alpha, \eta)
    &= \min\left\{\bar\calC_\phi(-\xnorm, \eta), \bar\calC_\phi(\gamma, \eta)\right\}, \\
    \inf_{-\xnorm \leq \alpha \leq \xnorm} \bar\calC_\phi(\alpha, \eta)
    &= \min\left\{\bar\calC_\phi(-\xnorm, \eta), \bar\calC_\phi(\xnorm, \eta)\right\} = \bar\calC_\phi(\xnorm, \eta).
  \end{align*}
  Here, we have
  \begin{align*}
    \bar\calC_\phi(\xnorm, \eta) &= (\underline{B} - \overline{B})\eta + \overline{B}, \\
    \bar\calC_\phi(\gamma, \eta) &= (\phi(\gamma) - \phi(-\gamma))\eta + \phi(-\gamma),
  \end{align*}
  where $\underline{B} \define \phi(\xnorm)$ and $\overline{B} \define \phi(-\xnorm)$.
  Then, for all $\eta \in \left(\frac{1}{2}, 1\right]$,
  \begin{align*}
    \bar\calC_\phi(\gamma, \eta) - \bar\calC_\phi(\xnorm, \eta)
    &= (\phi(\gamma) - \phi(-\gamma) + \overline{B} - \underline{B})\eta + (\phi(-\gamma) - \overline{B}) \\
    &\geq (\phi(\gamma) - \phi(-\gamma) + \overline{B} - \underline{B}) \frac{1}{2} + \phi(-\gamma) - \overline{B} \\
    &= \frac{\{\phi(\gamma) + \phi(-\gamma)\} - \{\phi(\xnorm) + \phi(-\xnorm)\}}{2} \\
    &> 0,
  \end{align*}
  where the first inequality holds since $(\phi(\gamma) - \phi(-\gamma) + \overline{B} - \underline{B}) > 0$ and $\eta > \frac{1}{2}$,
  and the second inequality holds because of the assumption $\phi(\gamma) + \phi(-\gamma) > \phi(\alpha) + \phi(-\alpha)$ for any $\alpha \in (\gamma, 1]$.
  In addition, we have $\bar\calC_\phi(-\xnorm, \eta) > \bar\calC_\phi(\xnorm, \eta)$ for $\eta > \frac{1}{2}$
  by part~\ref{lem:ccr-qccv:ccr-limit} of Lemma~\ref{lem:ccr-qccv}.
  Therefore,
  \begin{align*}
    \inf_{f \in \calFlin: f(x) \leq \gamma} & \calC_\phi(f, \eta, x) - \inf_{f \in \calFlin} \calC_\phi(f, \eta, x) \\
    &= \inf_{-\xnorm \leq \alpha \leq \gamma} \bar\calC_\phi(\alpha, \eta) - \inf_{-\xnorm \leq \alpha \leq \xnorm} \bar\calC_\phi(\alpha, \eta) \\
    &= \min\{\bar\calC_\phi(-\xnorm,\eta) - \bar\calC_\phi(\xnorm,\eta), \bar\calC_\phi(\gamma,\eta) - \bar\calC_\phi(\xnorm,\eta)\} \\
    &> 0
  \end{align*}
  holds for all $\eta$ such that $\frac{1}{2} < \eta \leq 1$,
  and this verifies (ii).
\end{proof}

\subsection{Proof of Theorem~\ref{thm:convex-loss-calibration}}
\label{sec:proof:convex-loss-calibration}

\begin{proof}
  First, we derive the necessary and sufficient condition for $(\phi_\gamma,\calFlin)$-calibration under $\xi$-Massart condition.
  Let us introduce
  \begin{align*}
    \delta_{\rho,\xi}^{\mathrm{Massart}}(\epsilon)
    \define \inf_{\substack{\eta \in [0,1]\\|2\eta - 1| \geq \xi}} \inf_{x \in \tcalX_\rho} \inf_{f \in \calF} \calC_{\phi}(f, \eta, x) - \calC_{\phi,\calF}^*(\eta, x)
    \quad \text{s.t.} \quad
    \calC_{\phi_\gamma}(f, \eta, x) - \calC_{\phi_\gamma,\calF}^*(\eta, x) \geq \epsilon,
  \end{align*}
  where $\tcalX_\rho \define \calX \setminus B_2^\circ(\gamma + \rho)$.
  A loss function $\phi$ is $(\phi_\gamma,\calFlin)$-calibrated if and only if $\delta_{\xi}^{\mathrm{Massart}}(\epsilon) > 0$ for all $\epsilon > 0$.
  It is easy to see that $\phi$ is $(\phi_\gamma,\calFlin)$-calibrated if and only if $\delta_{\rho,\xi}^{\mathrm{Massart}}(\epsilon) > 0$ for all $\epsilon > 0$ and $\rho \in (0, 1-\gamma)$
  by the same argument as the proof of Lemma~\ref{lem:calibration-iff-condition}.
  By following the same argument as the proof of Lemma~\ref{lem:ccr-property} (part~\ref{lem:ccr-property:expanded-calibration-condition}),
  we claim that a surrogate $\phi$ is calibrated if and only if
  \begin{align*}
    \inf_{f \in \calFlin: f(x) \leq \gamma} \calC_{\phi}(f, \eta, x) > \inf_{f \in \calFlin} \calC_{\phi}(f, \eta, x)
  \end{align*}
  for all $\eta$ and $x$ such that $\eta \geq \frac{1+\xi}{2}$ and $\xnorm > \gamma$.\footnote{
    The proof of this argument is a routine given the proof of Lemma~\ref{lem:ccr-property} (part~\ref{lem:ccr-property:expanded-calibration-condition}),
    which is omitted.
  }
  Denote $\bar\calC_{\phi}(f(x), \eta) = \calC_{\phi}(f, \eta, x) = \eta\phi(f(x)) + (1-\eta)\phi(-f(x))$.
  Then, it is equivalent to
  \begin{align*}
    H(\eta, x) \define \inf_{\alpha \in [-\xnorm, \gamma]} \bar\calC_{\phi}(\alpha, \eta) - \inf_{\alpha \in [-\xnorm, \xnorm]} \bar\calC_{\phi}(\alpha, \eta) > 0
    \quad \text{for all $\eta \geq \tfrac{1+\xi}{2}$,}
  \end{align*}
  by noting that $f(x)$ spans $[-\xnorm, \xnorm]$ for $f \in \calFlin$.

  Next, we will check each loss function.

  \paragraph{Shifted hinge loss.}
  Since
  \begin{align*}
    \bar\calC_{\phi}(\alpha, \eta)
    = \begin{cases}
      -\eta\alpha + \eta(1+\beta) & \text{if $\alpha < -(1+\beta)$,} \\
      (1-2\eta)\alpha + (1+\beta) & \text{if $-(1+\beta) \leq \alpha < 1+\beta$,} \\
      (1-\eta)\alpha + (1-\eta)(1+\beta) & \text{if $1+\beta \leq \alpha$,}
    \end{cases}
  \end{align*}
  we have
  \begin{align*}
    \inf_{\alpha \in [-\xnorm, \gamma]} \bar\calC_{\phi}(\alpha, \eta) = \bar\calC_{\phi}(\gamma, \eta),
    \quad
    \inf_{\alpha \in [-\xnorm, \xnorm]} \bar\calC_{\phi}(\alpha, \eta) = \bar\calC_{\phi}(\xnorm, \eta).
  \end{align*}
  Hence,
  $H(\eta, x) = (1-2\eta)(\gamma-\xnorm) \geq \xi(\xnorm-\gamma) > 0$,
  implying that the hinge loss is $(\phi_\gamma,\calFlin)$-calibrated under any $\xi > 0$.

  \paragraph{Logistic loss.}
  The minimizer of
  \begin{align*}
    \bar\calC_{\phi}(\alpha, \eta)
    = \eta\log(1+e^{-\alpha}) + (1-\eta)\log(1+e^{\alpha})
  \end{align*}
  in $\alpha \in \R$ is $\alpha^*(\eta) = \ln\left(\frac{\eta}{1-\eta}\right)$.
  When $\eta \geq \frac{1+\xi}{2}$ with $\xi > \tanh\left(\frac{\gamma}{2}\right)$,
  we have $\alpha^*(\eta) > \gamma$.
  Since $\bar\calC_{\phi}(\alpha, \eta)$ is convex in $\alpha$,
  it is decreasing for $\alpha \leq \alpha^*(\eta)$.
  Hence,
  \begin{itemize}
    \item
    \underline{when $\alpha^*(\eta) \leq \xnorm$},
    $H(\eta, x) = \bar\calC_{\phi}(\gamma, \eta) - \bar\calC_{\phi}(\alpha^*(\eta), \eta) > 0$,
    and

    \item
    \underline{when $\gamma < \xnorm < \alpha^*(\eta)$},
    $H(\eta, x) = \bar\calC_{\phi}(\gamma, \eta) - \bar\calC_{\phi}(\xnorm, \eta) > 0$
    since $\gamma < \xnorm$.
  \end{itemize}
  Therefore, the logistic loss is $(\phi_\gamma,\calFlin)$-calibrated under $\xi > \tanh\left(\frac{\gamma}{2}\right)$.
\end{proof}

\subsection{Proof of Lemma~\ref{lem:ccr-property}}
\label{sec:proof:ccr-property}

\begin{proof}
  \textbf{Parts~\ref{lem:ccr-property:eta-symmetry} and \ref{lem:ccr-property:alpha-symmetry}} are obvious from the definition of the class-conditional $\phi$-risk.

  \header{Part~\ref{lem:ccr-property:expanded-calibration-condition}}
  Let $\bar\delta: \R_{\geq 0} \times [0, 1] \times \calX \to \R_{\geq 0}$ be the $(\phi_\gamma,\calFlin)$-calibration function,
  whose expression is given in Proposition~\ref{lem:mer-calibration-function}:
  for $x \in \calX$ such that $\xnorm > \gamma$,
  \begin{align*}
    \bar\delta(\epsilon, \eta, x) = \begin{cases}
      \infty & \text{if $\epsilon > \max\{\eta, 1 - \eta\}$}, \\
      \inf\limits_{f \in \calFlin: |f(x)| \leq \gamma} \Delta\calC_{\phi,\calFlin}(f, \eta, x) & \text{if $|2\eta - 1| < \epsilon \leq \max\{\eta, 1 - \eta\}$}, \\
      \inf\limits_{f \in \calFlin: |f(x)| \leq \gamma \text{ or } (2\eta - 1)f(x) \leq 0} \Delta\calC_{\phi,\calFlin}(f, \eta, x) & \text{if $\epsilon \leq |2\eta - 1|$},
    \end{cases}
  \end{align*}
  and $\bar\delta(\epsilon, \eta, x) = \infty$ for $\xnorm \leq \gamma$.
  Proposition~\ref{prop:calibration-condition} states that $\phi$ is $(\phi_\gamma,\calFlin)$-calibrated
  if and only if $\bar\delta(\epsilon, \eta, x) > 0$
  for all $\epsilon > 0$, $\eta \in [0,1]$, and $x \in \calX$ with $\xnorm > \gamma$.
  We subsequently fix $\xnorm > \gamma$
  and simplify the third expression $\inf_f \{ \Delta\calC_{\phi,\calFlin}(f, \eta, x) \mid f \in \calFlin, |f(x)| \leq \gamma \text{ or } (2\eta - 1)f(x) \leq 0 \}$ first.
  Using part~\ref{lem:ccr-property:eta-symmetry} of Lemma~\ref{lem:ccr-property} and the symmetry of $\calFlin$,
  since we have for $\eta \leq \frac{1}{2}$,
  \begin{align*}
    \inf_{f \in \calFlin: |f(x)|\leq \gamma \text{ or } (2\eta-1)f(x) \leq 0} \calC_\phi(f, \eta, x)
    &= \inf_{f \in \calFlin: f(x) \geq -\gamma} \calC_\phi(f, \eta, x) \\
    &= \inf_{f \in \calFlin: f(x) \geq -\gamma} \calC_\phi(-f, 1-\eta, x)
    && \text{(part~\ref{lem:ccr-property:eta-symmetry} of Lemma~\ref{lem:ccr-property})} \\
    &= \inf_{f \in \calFlin: f(x) \leq \gamma} \calC_\phi(f, 1-\eta, x),
    && \text{(replace $-f$ with $f$)}
  \end{align*}
  and for $\eta \geq \frac{1}{2}$,
  \begin{align*}
    \inf_{f \in \calFlin: |f(x)|\leq \gamma \text{ or } (2\eta-1)f(x) \leq 0} \calC_\phi(f, \eta, x)
    &= \inf_{f \in \calFlin: f(x) \leq \gamma} \calC_\phi(f, \eta, x).
  \end{align*}
  By combining these two, we see that $\inf_{f \in \calFlin: f(x) \leq \gamma} \Delta\calC_{\phi,\calFlin}(f, \eta, x) > 0$ holds for all $\eta \geq \frac{1}{2}$
  if and only if
  $\inf_f \{ \Delta\calC_{\phi,\calFlin}(f, \eta, x) > 0 \mid f \in \calFlin, |f(x)|\leq \gamma \text{ or } (2\eta-1)f(x) \leq 0 \}$
  holds for all $\eta \in [0, 1]$.
  Hence,
  \begin{align*}
    \inf_{f \in \calFlin: |f(x)|\leq \gamma \text{ or } (2\eta-1)f(x) \leq 0} \Delta\calC_{\phi,\calFlin}(f, \eta, x) > 0
  \end{align*}
  for $\epsilon > 0$ and $\eta \in [0, 1]$ such that $\epsilon \leq |2\eta-1|$
  if and only if
  \begin{align*}
    \inf_{f \in \calFlin: f(x) \leq \gamma} \Delta\calC_{\phi,\calFlin}(f, \eta, x) > 0
  \end{align*}
  for $\epsilon > 0$ and $\eta \in [\frac{1}{2}, 1]$ such that $\epsilon \leq 2\eta-1$.

  Note that the second expression $\inf_f \{ \Delta\calC_{\phi,\calFlin}(f, \eta, x) \mid f \in \calFlin, |f(x)| \leq \gamma \}$ can be simplified in the same way.
  Therefore, $\bar\delta(\epsilon, \eta, x) > 0$ for all $\epsilon > 0$, $\eta \in [0,1]$, and $x \in \calX$
  if and only if
  \begin{align*}
    \begin{cases}
      \inf\limits_{f \in \calFlin: |f(x)| \leq \gamma} \calC_\phi(f, \eta, x) > \inf\limits_{f \in \calFlin} \calC_\phi(f, \eta, x) & \text{for all $\eta \geq \frac{1}{2}$ such that $2\eta-1 < \epsilon \leq \eta$}, \\
      \inf\limits_{f \in \calFlin: f(x) \leq \gamma} \calC_\phi(f, \eta, x) > \inf\limits_{f \in \calFlin} \calC_\phi(f, \eta, x) & \text{for all $\eta \geq \frac{1}{2}$ such that $\epsilon \leq 2\eta-1$},
    \end{cases}
  \end{align*}
  for all $\epsilon > 0$, which is equivalent to
  \begin{align*}
    \begin{cases}
      \inf\limits_{f \in \calFlin: |f(x)| \leq \gamma} \calC_\phi(f, \eta, x) > \inf\limits_{f \in \calFlin} \calC_\phi(f, \eta, x) & \text{for all $\eta \geq \frac{1}{2}$ such that $\epsilon \leq \eta < \frac{1 + \epsilon}{2}$}, \\
      \inf\limits_{f \in \calFlin: f(x) \leq \gamma} \calC_\phi(f, \eta, x) > \inf\limits_{f \in \calFlin} \calC_\phi(f, \eta, x) & \text{for all $\eta \geq \frac{1}{2}$ such that $\frac{1 + \epsilon}{2} \leq \eta \leq 1$},
    \end{cases}
  \end{align*}
  for all $\epsilon > 0$.

  We immediately observe that
  \begin{align*}
    \bigcup_{\epsilon > 0} \Set{\eta \geq \frac{1}{2} | \varepsilon \leq \eta < \frac{1+\varepsilon}{2}}
    &= \Set{\frac{1}{2} \leq \eta \leq 1}, \; \text{and}\\
    \bigcup_{\epsilon > 0} \Set{\eta \geq \frac{1}{2} | \frac{1+\varepsilon}{2} \leq \eta \leq 1}
    &= \Set{\frac{1}{2} < \eta \leq 1}.
  \end{align*}
  Therefore, we reduce the above conditions as
  \begin{align*}
    \begin{cases}
      \inf\limits_{f \in \calFlin: |f(x)| \leq \gamma} \calC_\phi(f, \eta, x) > \inf\limits_{f \in \calFlin} \calC_\phi(f, \eta, x) & \text{if $\frac{1}{2} \leq \eta \leq 1$}, \\
      \inf\limits_{f \in \calFlin: f(x) \leq \gamma} \calC_\phi(f, \eta, x) > \inf\limits_{f \in \calFlin} \calC_\phi(f, \eta, x) & \text{if $\frac{1}{2} < \eta \leq 1$}.
    \end{cases}
  \end{align*}
  Note that $\inf_{|f(x)| \leq \gamma} \calC_\phi(f, \eta, x) \geq \inf_{f \in \calFlin: f(x) \leq \gamma} \calC_\phi(f, \eta, x)$ (inequality is not strict) always holds for all $\eta$.
  Since the first case is included in the second case except when $\eta = \frac{1}{2}$,
  this is equivalent to
  \begin{align*}
    \inf\limits_{f \in \calFlin: |f(x)| \leq \gamma} \calC_\phi\left(f, \tfrac{1}{2}, x\right) &> \inf\limits_{f \in \calFlin} \calC_\phi\left(f, \tfrac{1}{2}, x\right),
    \; \text{and} \\
    \inf\limits_{f \in \calFlin: f(x) \leq \gamma} \calC_\phi(f, \eta, x) &> \inf\limits_{f \in \calFlin} \calC_\phi(f, \eta, x)
    \; \text{for $\eta \in \left(\tfrac{1}{2},1\right]$}.
  \end{align*}

  \header{Part~\ref{lem:ccr-property:expanded-calibration-condition:0/1}}
  First, we obtain the calibration function $\bar\delta$ wrt $(\ell_{01},\calFlin)$ as in Lemma~\ref{lem:mer-calibration-function}.
  The $\ell_{01}$-CCR for $f \in \calFlin$ at $x$ is
  \begin{align*}
    \calC_{\ell_{01}}(f, \eta, x)
    = \eta\ind{\sign(f(x)) = +1} + (1-\eta)\ind{\sign(f(x))=-1}
    = \begin{cases}
      \eta & \text{if $f(x) \geq 0$,} \\
      1 - \eta & \text{if $f(x) < 0$.}
    \end{cases}
  \end{align*}
  To compute $\Delta\calC_{\ell_{01},\calFlin}(f, \eta, x)$,
  note that given $x$, $f(x)$ ranges $[-\xnorm, \xnorm]$ for $f \in \calFlin$.
  When $\xnorm = 0$, $\calC_{\ell_{01}}(f, \eta, x) = \eta$ for all $f \in \calFlin$,
  implying $\Delta\calC_{\ell_{01},\calFlin}(f, \eta, x) = 0$.
  When $\xnorm > 0$, we have $\Delta\calC_{\ell_{01},\calFlin}(f, \eta, x) = |2\eta-1|\cdot\ind{(2\eta-1)f(x) \leq 0}$
  as in the same way as \citet[Lemma~4.1]{Steinwart:2007}.
  Hence,
  \begin{align*}
    \Delta\calC_{\ell_{01},\calFlin}(f, \eta, x) = \begin{cases}
      0 & \text{if $\xnorm = 0$,} \\
      |2\eta-1|\cdot\ind{(2\eta-1)f(x) \leq 0} & \text{if $\xnorm > 0$.}
    \end{cases}
  \end{align*}
  Now, if $\xnorm = 0$ or $\epsilon > |2\eta-1|$, we always have $\Delta\calC_{\ell_{01},\calFlin}(f, \eta, x) < \epsilon$
  hence $\bar\delta(\epsilon, \eta, x) = \infty$.
  If $\xnorm > 0$ and $\epsilon \leq |2\eta - 1|$,
  we have $\Delta\calC_{\ell_{01},\calFlin}(f, \eta, x) < \epsilon$ if and only if $(2\eta-1)f(x) > 0$.
  Therefore, we have
  \begin{align*}
    \bar\delta(\epsilon, \eta, x) = \begin{cases}
      \infty & \text{if $\xnorm = 0$ or $\epsilon > |2\eta - 1|$,} \\
      \inf_{f \in \calFlin: (2\eta-1)f(x) \leq 0} \Delta\calC_{\phi,\calFlin}(f, \eta, x) & \text{if $\xnorm > 0$ and $\epsilon \leq |2\eta - 1|$.}
    \end{cases}
  \end{align*}

  Next, by Proposition~\ref{prop:calibration-condition}, $\phi$ is $(\ell_{01},\calFlin)$-calibrated
  if and only if $\bar\delta(\epsilon, \eta, x) > 0$
  for all $\epsilon > 0$, $\eta \in [0, 1]$, and $x \in \calX$.
  In the same way as part~\ref{lem:ccr-property:expanded-calibration-condition} of Lemma~\ref{lem:ccr-property},
  this is equivalent to
  \begin{align*}
    \inf\limits_{f \in \calFlin: f(x) \leq 0} \calC_\phi(f, \eta, x) > \inf\limits_{f \in \calFlin} \calC_\phi(f, \eta, x)
    \quad \text{for all $\eta \geq \tfrac{1}{2}$ such that $\tfrac{1 + \epsilon}{2} \leq \eta \leq 1$},
  \end{align*}
  for all $\epsilon > 0$ and $x \in \calX \setminus \{0\}$, by using part~\ref{lem:ccr-property:eta-symmetry} of Lemma~\ref{lem:ccr-property} and symmetry of $\calFlin$.
  This is equivalent to the lemma statement.
\end{proof}

\subsection{Proof of Lemma~\ref{lem:ccr-qccv}}
\label{sec:proof:ccr-qccv}

Denote $\bar\phi(\alpha) \define \phi(\alpha) + \phi(-\alpha)$.

\begin{proof} \textit{(of Lemma~\ref{lem:ccr-qccv})}
  \header{Part~\ref{lem:ccr-qccv:ccr-decreasing}}
  Fix an $\eta \in \left(\tfrac{1}{2}, 1\right]$
  and $\alpha_1, \alpha_2 \geq 0$ such that $\alpha_1 < \alpha_2$.
  By the fact that $\phi$ is nonincreasing, we have
  \begin{align*}
    \phi(\alpha_1) - \phi(-\alpha_1) - \phi(\alpha_2) + \phi(-\alpha_2)
    &= (\phi(\alpha_1) - \phi(\alpha_2)) + (\phi(-\alpha_2) - \phi(-\alpha_1)) \\
    &\geq 0.
  \end{align*}
  Then,
  \begin{align*}
    \bar\calC_\phi(\alpha_1, \eta) - \bar\calC_\phi(\alpha_2, \eta)
    &= (\phi(\alpha_1) - \phi(-\alpha_1) - \phi(\alpha_2) + \phi(-\alpha_2))\eta + \phi(-\alpha_1) - \phi(-\alpha_2) \\
    &\geq (\phi(\alpha_1) - \phi(-\alpha_1) - \phi(\alpha_2) + \phi(-\alpha_2))\frac{1}{2} + \phi(-\alpha_1) - \phi(-\alpha_2) \\
    &= \frac{\phi(\alpha_1) + \phi(-\alpha_1) - \phi(\alpha_2) - \phi(-\alpha_2)}{2} \\
    &\geq 0,
  \end{align*}
  where the last inequality holds because $\phi(\alpha) + \phi(-\alpha)$ is nonincreasing
  when $\alpha \geq 0$ by part~\ref{lem:ccr-qccv:even-decreasing}.
  Therefore, $\bar\calC_\phi(\alpha, \eta)$ is nonincreasing in $\alpha \geq 0$.

  \header{Part~\ref{lem:ccr-qccv:ccr-limit}}
  Fix an $\eta \in \left(\tfrac{1}{2},1\right]$. Then,
  \begin{align*}
    \bar\calC_\phi(-\alpha, \eta) - \bar\calC_\phi(\alpha, \eta)
    = (2\eta-1)(\phi(-\alpha) - \phi(\alpha))
    > 0.
  \end{align*}

  \header{Part~\ref{lem:ccr-qccv:even-decreasing}}
  $\bar\phi$ is an even function, so it is symmetric in $\alpha = 0$.
  $\bar\phi$ is continuous because of continuity of $\phi$.
  Every quasiconcave continuous function is nondecreasing, or nonincreasing,
  or there is global maxima in its domain~\citep{Boyd:2004}.
  If $\bar\phi$ is either nondecreasing or nonincreasing in $\alpha$,
  it is a constant function in $\alpha$ and clearly nonincreasing in $\alpha \geq 0$.
  If $\bar\phi$ has global maxima,
  i.e., there is a point $\alpha_* \in \dom(\bar\phi)$ such that $\bar\phi$ is nondecreasing for $\alpha \leq \alpha_*$ and nonincreasing for $\alpha \geq \alpha_*$,
  it is still nonincreasing in $\alpha \geq 0$.
  This is clear when $\alpha_* \leq 0$.
  When $\alpha_* > 0$, $\bar\phi$ may only be a constant function in $\alpha \in [0, \alpha_*]$
  otherwise we have a point $\tilde\alpha \in [0, \alpha_*)$ such that $\bar\phi(\tilde\alpha) < \bar\phi(\alpha_*)$;
  hence $\bar\phi(\alpha_*) = \bar\phi(-\alpha_*)$ ($\define \bar\phi_*$) by the symmetry and $\bar\phi_0 \define \bar\phi(\tilde\alpha) < \bar\phi_*$,
  which means there is no convex superlevel sets for $\bar\phi$ within the range $(\bar\phi_0, \bar\phi_*)$.
  For example, pick $t \in (\bar\phi_0, \bar\phi_*)$ and consider $t$-superlevel set of $\bar\phi$.
  If $t$-superlevel set is convex, it must contain every point in $[-\alpha_*, \alpha_*]$ since $t < \bar\phi_* = \bar\phi(-\alpha_*) = \bar\phi(\alpha_*)$.
  However, $t$-superlevel set would not contain $\tilde\alpha \in [-\alpha_*, \alpha_*]$ since $t > \bar\phi_0 = \bar\phi(\tilde\alpha)$.
  This contradicts the quasiconcavity of $\bar\phi$.
  In any cases, $\bar\phi$ is nonincreasing in $\alpha \geq 0$.

  \header{Part~\ref{lem:ccr-qccv:ccr-inf}}
  This is an immediate consequence of the quasiconcavity and continuity of $\bar\calC_\phi(\alpha, \eta)$ (Assumption~\ref{assump:qccv-ccr}).
\end{proof}

\section{Derivation of Calibration Functions}
\label{sec:calibration-function}

In this section, we derive closed-forms of $(\phi_\gamma,\calFlin)$-calibration functions for several surrogate losses $\phi$
by minimizing $\bar\delta(\epsilon, \eta, x)$ in \eqref{eq:inner-calibration-function} wrt $\eta \in [0, 1]$ and $x \in \tcalX_\rho$,
or $\xnorm \geq \gamma + \rho$ in other words.
Let $\bar\calC_\phi(\alpha, \eta) \define \eta\phi(\alpha) + (1-\eta)\phi(-\alpha)$ to simplify notation.

\subsection{Ramp Loss}
\label{sec:calibration-function:ramp}

\begin{figure}[t]
  \centering
  \scriptsize
  \def\ylim{1.1}
  \def\vareta{0.75}
  \def\vargamma{0.3}
  \def\changepoints{{-1-\varbeta},{-1+\varbeta},{1-\varbeta},{1+\varbeta}}
  \subfigure[$0 \leq \beta < 1 - \gamma$][c]{
    \def\varbeta{0.3}
    \includegraphics{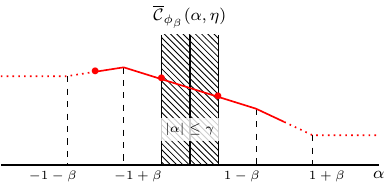}
  } \hfill
  \subfigure[$1 - \gamma \leq \beta < 1 + \gamma$][c]{
    \def\varbeta{0.8}
    \includegraphics{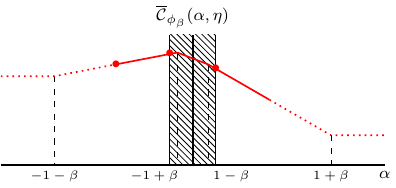}
  } \hfill
  \subfigure[$1 + \gamma \leq \beta < 2$][c]{
    \def\varbeta{1.5}
    \includegraphics{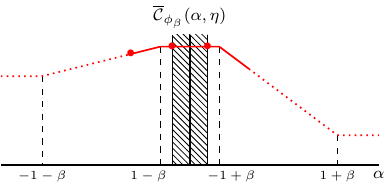}
  } \hfill
  \subfigure[$2 \leq \beta$][c]{
    \def\varbeta{2.2}
    \includegraphics{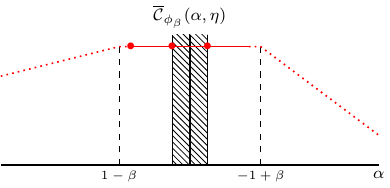}
  }
  \caption{The class-conditional risk for the ramp loss.}
  \label{fig:ccr-ramp-loss}
\end{figure}

The ramp loss is $\phi(\alpha) = \min\left\{1, \max\left\{0, \frac{1-\alpha}{2}\right\}\right\}$.
We consider the shifted ramp loss: $\phi_\beta(\alpha) = \phi(\alpha - \beta)$:
\begin{align}
  \phi_\beta(\alpha) = \begin{cases}
    1 & \text{if $\alpha \leq -1 + \beta$}, \\
    \frac{1 - \alpha + \beta}{2} & \text{if $-1 + \beta < \alpha \leq 1 + \beta$}, \\
    0 & \text{if $1 + \beta < \alpha$}.
  \end{cases}
  \nonumber
\end{align}
The $\phi_\beta$-CCR is plotted in Figure~\ref{fig:ccr-ramp-loss}.
We can confirm that $\bar\calC_{\phi_\beta}$ is quasiconcave with each $\beta \geq 0$.

\subsubsection{Minimal Inner Risk}

By part~\ref{lem:ccr-qccv:ccr-inf} of Lemma~\ref{lem:ccr-qccv}, it is easy to check
\begin{align*}
  \calC_{\phi_\beta,\calFlin}^*(\eta, x)
  &= \inf_{f \in \calFlin} \calC_{\phi_\beta}(f, \eta, x)
  = \inf_{\alpha \in [-\xnorm, \xnorm]} \bar\calC_{\phi_\beta}(\alpha, \eta) \\
  &= \min\{\bar\calC_{\phi_\beta}(-\xnorm, \eta), \bar\calC_{\phi_\beta}(\xnorm, \eta)\}.
\end{align*}

\subsubsection{Calibration Function}

We analyze $\phi_\beta$-CCR $\calC_{\phi_\beta}(f, \eta, x) = \eta\phi_\beta(f(x)) + (1 - \eta)\phi_\beta(-f(x)) = \bar\calC_{\phi_\beta}(f(x), \eta)$,
and restrict $\eta > \frac{1}{2}$ by virtue of the symmetry of $\calC_{\phi_\beta}$ (part~\ref{lem:ccr-property:eta-symmetry} in Lemma~\ref{lem:ccr-property}).
It is easy to see $\calC_{\phi_\beta,\calFlin}^*(\eta, x) = \bar\calC_{\phi_\beta}(\xnorm, \eta)$.
Subsequently, we divide into cases depending on the relationship among
$\bar\calC_{\phi_\beta}(-\xnorm, \eta)$, $\bar\calC_{\phi_\beta}(\gamma, \eta)$, and $\bar\calC_{\phi_\beta}(-\gamma, \eta)$.

\header{(A) When $0 \leq \beta < 1 - \gamma$}
\begin{align*}
  \bar\calC_{\phi_\beta}(-\xnorm, \eta) &= \begin{cases}
    \eta + \frac{1 - \xnorm + \beta}{2}(1-\eta) & \text{if $\gamma+\rho \leq \xnorm < 1-\beta$,} \\
    \frac{1 + \xnorm + \beta}{2}\eta + \frac{1 - \xnorm + \beta}{2}(1 - \eta) & \text{otherwise,}
  \end{cases} \\
  \bar\calC_{\phi_\beta}(\gamma, \eta) &= \frac{1-\gamma+\beta}{2}\eta + \frac{1+\gamma+\beta}{2}(1 - \eta), \\
  \bar\calC_{\phi_\beta}(-\gamma, \eta) &= \frac{1+\gamma+\beta}{2}\eta + \frac{1-\gamma+\beta}{2}(1 - \eta),
\end{align*}
from which it follows that $\bar\calC_{\phi_\beta}(-\gamma, \eta) - \bar\calC_{\phi_\beta}(\gamma, \eta) = \frac{\gamma}{2}(2\eta - 1) > 0$,
that is, $\bar\calC_{\phi_\beta}(-\gamma, \eta) > \bar\calC_{\phi_\beta}(\gamma, \eta)$ for all $\eta > \frac{1}{2}$.
In addition, since when $\gamma+\rho \leq \xnorm < 1-\beta$,
\begin{align*}
  \bar\calC_{\phi_\beta}(-\xnorm, \eta) - \bar\calC_{\phi_\beta}(\gamma, \eta)
  = (\gamma + \xnorm)\left(\eta - \frac{1}{2}\right),
\end{align*}
we have $\calC_{\phi_\beta}(\gamma, \eta) < \calC_{\phi_\beta}(-\xnorm, \eta)$ for such $\beta$ and $\xnorm$.
By part~\ref{lem:ccr-qccv:ccr-inf} in Lemma~\ref{lem:ccr-qccv},
it follows that
\begin{align*}
  \inf_{f \in \calFlin: |f(x)| \leq \gamma} \calC_{\phi_\beta}(f, \eta, x)
  &= \inf_{|\alpha| \leq \gamma} \bar\calC_{\phi_\beta}(\alpha, \eta)
  = \bar\calC_{\phi_\beta}(\gamma, \eta), \\
  \inf_{f \in \calFlin: f(x) \leq \gamma} \calC_{\phi_\beta}(f, \eta, x)
  &= \inf_{\alpha \in [-\xnorm, \gamma]} \bar\calC_{\phi_\beta}(\alpha, \eta)
  = \min\{\bar\calC_{\phi_\beta}(\gamma, \eta), \bar\calC_{\phi_\beta}(-\xnorm, \eta)\}.
\end{align*}
Thus, by Lemma~\ref{lem:mer-calibration-function},
\begin{align*}
  \bar\delta(\epsilon, \eta, x)
  &= \begin{cases}
    \infty & \text{if $\eta < \epsilon$,} \\
    \bar\calC_{\phi_\beta}(\gamma, \eta) - \bar\calC_{\phi_\beta}(\xnorm, \eta) & \text{if $\epsilon \leq \eta$ and $\bar\calC_{\phi_\beta}(-\xnorm, \eta) \geq \bar\calC_{\phi_\beta}(\gamma, \eta)$,} \\
    \bar\calC_{\phi_\beta}(-\xnorm, \eta) - \bar\calC_{\phi_\beta}(\xnorm, \eta) & \text{if $\epsilon \leq \eta$ and $\bar\calC_{\phi_\beta}(-\xnorm, \eta) < \bar\calC_{\phi_\beta}(\gamma, \eta)$,}
  \end{cases}
  \\
  &= \begin{cases}
    \infty & \text{if $\eta < \epsilon$,} \\
    \bar\calC_{\phi_\beta}(\gamma, \eta) - \bar\calC_{\phi_\beta}(\xnorm, \eta) & \text{if $\epsilon \leq \eta$ and $\bar\calC_{\phi_\beta}(-\xnorm, \eta) \geq \bar\calC_{\phi_\beta}(\gamma, \eta)$,} \\
    \frac{1+\xnorm-\beta}{2}\left(\eta-\frac{1}{2}\right) & \text{if $\epsilon \leq \eta$ and $\bar\calC_{\phi_\beta}(-\xnorm, \eta) < \bar\calC_{\phi_\beta}(\gamma, \eta)$,}
  \end{cases}
\end{align*}
where the last identity holds because $\xnorm \geq 1-\beta$ when $\bar\calC_{\phi_\beta}(-\xnorm, \eta) < \bar\calC_{\phi_\beta}(\gamma, \eta)$.
Since $\bar\calC_{\phi_\beta}(\alpha, \eta)$ is nonincreasing in $\alpha \geq 0$ (part~\ref{lem:ccr-qccv:ccr-decreasing} of Lemma~\ref{lem:ccr-qccv}),
we know $\calC_{\phi_\beta}^*(\eta, x) = \bar\calC_{\phi_\beta}(\xnorm, \eta)$ is maximized at $\xnorm = \gamma+\rho$,
which implies
\begin{align*}
  \delta_\rho(\epsilon)
  &= \inf_{\eta \in \left(\frac{1}{2}, 1\right]} \inf_{x \in \calX: \xnorm \geq \gamma + \rho} \bar\delta(\epsilon, \eta, x) \\
  &= \inf_{\eta \in \left(\frac{1}{2}, 1\right]: \epsilon \leq \eta} \frac{1+\gamma+\rho-\beta}{2}\left(\eta-\frac{1}{2}\right) \\
  &= \frac{1+\gamma+\rho-\beta}{2}\left[\epsilon - \frac{1}{2}\right]_+.
\end{align*}

\header{(B) When $1 - \gamma \leq \beta < 1 + \gamma$}
\begin{align*}
  \bar\calC_{\phi_\beta}(-\xnorm, \eta) &= \eta + \frac{1 - \xnorm + \beta}{2}(1 - \eta), \\
  \bar\calC_{\phi_\beta}(\gamma, \eta) &= \frac{1-\gamma+\beta}{2}\eta + (1 - \eta), \\
  \bar\calC_{\phi_\beta}(-\gamma, \eta) &= \eta + \frac{1-\gamma+\beta}{2}(1 - \eta),
\end{align*}
from which it follows that $\bar\calC_{\phi_\beta}(-\gamma, \eta) - \bar\calC_{\phi_\beta}(\gamma, \eta) = \frac{1+\gamma-\beta}{2}(2\eta - 1) > 0$,
that is, $\bar\calC_{\phi_\beta}(-\gamma, \eta) > \bar\calC_{\phi_\beta}(\gamma, \eta)$ for all $\eta > \frac{1}{2}$.
In addition, since
\begin{align*}
  \bar\calC_{\phi_\beta}(\gamma, \eta) - \bar\calC_{\phi_\beta}(-\xnorm, \eta)
  = -\frac{2-2\beta+\gamma+\xnorm}{2}\left(\eta - \eta_0(x)\right),
  && \left(\eta_0(x) \define \frac{1}{1+\frac{1+\gamma-\beta}{1+\xnorm-\beta}} \right)
\end{align*}
and $2-2\beta+\gamma+\xnorm>0$,
we have $\bar\calC_{\phi_\beta}(\gamma, \eta) > \bar\calC_{\phi_\beta}(-\xnorm, \eta)$ if $\frac{1}{2} < \eta < \eta_0(x)$,
and $\bar\calC_{\phi_\beta}(\gamma, \eta) \leq \bar\calC_{\phi_\beta}(-\xnorm, \eta)$ if $\eta_0(x) \leq \eta$.
Note that $\frac{1}{2} < \eta_0(x) < 1$.
\begin{itemize}
  \item \underline{If $\frac{1}{2} < \eta < \eta_0(x)$:}
  By part~\ref{lem:ccr-qccv:ccr-inf} in Lemma~\ref{lem:ccr-qccv},
  it follows that
  \begin{align*}
    \inf_{f \in \calFlin: |f(x)| \leq \gamma} \calC_{\phi_\beta}(f, \eta, x)
    &= \inf_{|\alpha| \leq \gamma} \bar\calC_{\phi_\beta}(\alpha, x)
    = \bar\calC_{\phi_\beta}(\gamma, \eta), \\
    \inf_{f \in \calFlin: f(x) \leq \gamma} \calC_{\phi_\beta}(f, \eta, x)
    &= \inf_{\alpha \in [-\xnorm, \gamma]} \bar\calC_{\phi_\beta}(\alpha, x)
    = \bar\calC_{\phi_\beta}(-\xnorm, \eta).
  \end{align*}
  Thus, by Lemma~\ref{lem:mer-calibration-function},
  \begin{align*}
    \bar\delta(\epsilon, \eta, x)
    &= \begin{cases}
      \infty & \text{if $\eta < \epsilon$}, \\
      \bar\calC_{\phi_\beta}(\gamma, \eta) - \bar\calC_{\phi_\beta}(\xnorm, \eta)
      = \frac{\xnorm-\gamma}{2}\eta & \text{if $\epsilon \leq \eta < \frac{1+\epsilon}{2}$}, \\
      \bar\calC_{\phi_\beta}(-\xnorm, \eta) - \bar\calC_{\phi_\beta}(\xnorm, \eta) \\
      \quad = (1 + \xnorm - \beta)\left(\eta - \frac{1}{2}\right) & \text{if $\frac{1+\epsilon}{2} \leq \eta$},
    \end{cases}
  \end{align*}
  and
  \begin{align*}
    \inf_{\eta \in \left(\frac{1}{2}, \eta_0(x)\right)} \bar\delta(\epsilon, \eta, x)
    &= \min\left\{
      \frac{\xnorm - \gamma}{2}\max\left\{\epsilon, \frac{1}{2}\right\},
      (1 + \xnorm - \beta)\epsilon
    \right\} \\
    &= \begin{cases}
      (1 + \xnorm - \beta)\epsilon & \text{if $0 < \epsilon \leq \epsilon_0(x)$,} \\
      \frac{\xnorm - \gamma}{4} & \text{if $\epsilon_0(x) < \epsilon \leq \frac{1}{2}$,} \\
      \frac{\xnorm - \gamma}{2}\epsilon & \text{if $\frac{1}{2} < \epsilon$},
    \end{cases}
  \end{align*}
  where $\epsilon_0(x) \define \frac{\xnorm - \gamma}{4(1 + \xnorm - \beta)}$.
  Finally, by taking the infimum over $x \in \tcalX_\rho$,
  \begin{align*}
    \inf_{x \in \tcalX_\rho} \inf_{\eta \in \left(\frac{1}{2}, \eta_0(x)\right)} \bar\delta(\epsilon, \eta, x)
    &= \begin{cases}
      (1 + \gamma + \rho - \beta)\epsilon & \text{if $0 < \epsilon \leq \epsilon_0$,} \\
      \frac{\rho}{4} & \text{if $\epsilon_0 < \epsilon \leq \frac{1}{2}$,} \\
      \frac{\rho}{2}\epsilon & \text{if $\frac{1}{2}< \epsilon$},
    \end{cases}
  \end{align*}
  where $\epsilon_0 \define \frac{\rho}{4(1 + \gamma + \rho - \beta)}$.

  \item \underline{If $\eta_0(x) \leq \eta \leq 1$:}
  By part~\ref{lem:ccr-qccv:ccr-inf} in Lemma~\ref{lem:ccr-qccv},
  it follows that
  \begin{align*}
    \inf_{f \in \calFlin: |f(x)| \leq \gamma} \calC_{\phi_\beta}(f, \eta, x)
    &= \inf_{|\alpha| \leq \gamma} \bar\calC_{\phi_\beta}(\alpha, \eta)
    = \bar\calC_{\phi_\beta}(\gamma, \eta), \\
    \inf_{f \in \calFlin: f(x) \leq \gamma} \calC_{\phi_\beta}(f, \eta, x)
    &= \inf_{\alpha \in [-\xnorm, \gamma]} \bar\calC_{\phi_\beta}(\alpha, \eta)
    = \bar\calC_{\phi_\beta}(\gamma, \eta)
  \end{align*}
  Thus,
  \begin{align*}
    \bar\delta(\epsilon, \eta, x)
    = \begin{cases}
      \infty & \text{if $\eta < \epsilon$},
      \\
      \bar\calC_{\phi_\beta}(\gamma, \eta) - \bar\calC_{\phi_\beta}(\xnorm, \eta) = \frac{\xnorm - \gamma}{2}\eta
      & \text{if $\epsilon \leq \eta$},
    \end{cases}
  \end{align*}
  and
  \begin{align*}
    \inf_{\eta \in [\eta_0(x), 1]} \bar\delta(\epsilon, \eta, x)
    &= \frac{\xnorm - \gamma}{2}\max\left\{\eta_0(x), \epsilon\right\}.
  \end{align*}
  By taking the infimum over $x \in \tcalX_\rho$,
  \begin{align*}
    \inf_{x \in \tcalX_\rho} \inf_{\eta \in [\eta_0(x), 1]} \bar\delta(\epsilon, \eta, x)
    &= \begin{cases}
      \frac{\rho}{2}\epsilon_1 & \text{if $0 < \epsilon \leq \epsilon_1$,} \\
      \frac{\rho}{2}\epsilon & \text{if $\epsilon_1 < \epsilon$,}
    \end{cases}
  \end{align*}
  where $\epsilon_1 \define \eta_0(\gamma + \rho) = \frac{1+\gamma+\rho-\beta}{2+2\gamma-2\beta+\rho}$.
\end{itemize}
Finally, we obtain the calibration function by combining the above cases as follows.
\begin{align*}
  \delta_\rho(\epsilon) = \begin{cases}
    (1+\gamma+\rho-\beta)\epsilon & \text{if $0 < \epsilon \leq \epsilon_0$,} \\
    \frac{\rho}{4} & \text{if $\epsilon_0 < \epsilon \leq \frac{1}{2}$,} \\
    \frac{\rho}{2}\epsilon & \text{if $\frac{1}{2} < \epsilon$.}
  \end{cases}
\end{align*}

\header{(C) When $1 + \gamma \leq \beta < 2$}
It is easy to see
\begin{align*}
  \inf_{f \in \calFlin: |f(x)| \leq \gamma} \calC_{\phi_\beta}(f, \eta, x)
  &= \inf_{|\alpha| \leq \gamma} \bar\calC_{\phi_\beta}(\alpha, \eta) = 1,
  \\
  \inf_{f \in \calFlin: f(x) \leq \gamma} \calC_{\phi_\beta}(\alpha, \eta)
  &= \inf_{\alpha \in [-\xnorm, \gamma]} \bar\calC_{\phi_\beta}(\alpha, \eta) = \bar\calC_{\phi_\beta}(-\xnorm, \eta) \\
  &= \begin{cases}
    1 & \text{if $\gamma < \xnorm \leq -1+\beta$}, \\
    \eta + \frac{1-\xnorm+\beta}{2}(1-\eta) & \text{if $-1+\beta < \xnorm \leq 1$},
  \end{cases}
  \\
  \calC_{\phi_\beta,\calFlin}^*(\eta, x) &= \bar\calC_{\phi_\beta}(\xnorm, \eta) \\
  &= \begin{cases}
    1 & \text{if $\gamma < \xnorm \leq -1+\beta$}, \\
    \frac{1-\xnorm+\beta}{2}\eta + (1-\eta) & \text{if $-1+\beta < \xnorm \leq 1$}.
  \end{cases}
\end{align*}
Hence, by part~\ref{lem:ccr-qccv:ccr-inf} in Lemma~\ref{lem:ccr-qccv},
it follows that
\begin{align*}
  \bar\delta(\epsilon, \eta, x)
  &= \begin{cases}
    \infty & \text{if $\eta < \epsilon$}, \\
    1 - \bar\calC_{\phi_\beta}(\xnorm, \eta) & \text{if $\epsilon \leq \eta < \frac{1 + \epsilon}{2}$}, \\
    \bar\calC_{\phi_\beta}(-\xnorm, \eta) - \bar\calC_{\phi_\beta}(\xnorm, \eta) & \text{if $\frac{1 + \epsilon}{2} \leq \eta$}.
  \end{cases} \\
  &= \begin{cases}
    \infty & \text{if $\eta < \epsilon$}, \\
    0 & \text{if $\epsilon \leq \eta < \frac{1+\epsilon}{2}$ and $\gamma < \xnorm \leq -1+\beta$}, \\
    \frac{1+\xnorm-\beta}{2}\eta & \text{if $\epsilon \leq \eta < \frac{1+\epsilon}{2}$ and $-1+\beta < \xnorm \leq 1$}, \\
    0 & \text{if $\frac{1+\epsilon}{2} \leq \eta$ and $\gamma < \xnorm \leq -1+\beta$}, \\
    (1+\xnorm-\beta)\left(\eta-\frac{1}{2}\right) & \text{if $\frac{1+\epsilon}{2} \leq \eta$ and $-1+\beta < \xnorm \leq 1$}.
\end{cases}
\end{align*}
Thus, Lemma~\ref{lem:mer-calibration-function} implies $\delta_\rho(\epsilon) = \inf_{\eta \in \left(\frac{1}{2}, 1\right]} \inf_{\xnorm \in [\gamma + \rho, 1]} \bar\delta(\epsilon, \eta, x) = 0$
when $\gamma + \rho \leq -1 + \beta \iff 1 + \gamma + \rho \leq \beta$,
by setting $\xnorm \leq -1 + \beta$ and arbitrary $\eta$.
When $1 + \gamma + \rho > \beta$, $\delta_\rho(\epsilon) = \frac{1+\gamma+\rho-\beta}{2}\epsilon$
by setting $\xnorm = \gamma + \rho$.

\header{(D) When $2 \leq \beta$}
In this case, $\bar\calC_{\phi_\beta}(\alpha, \eta) = 1$ for all $\eta \in [0, 1]$ and $\alpha \in [-1, 1]$.
Hence, $\Delta\calC_{\phi_\beta,\calFlin}(f, \eta, x) = 0$ and $\delta_\rho(\epsilon) = 0$.

To sum up, the $(\phi_\gamma,\calFlin)$-calibration function and its Fenchel-Legendre biconjugate
of the ramp loss is as follows:
\begin{itemize}
  \item If $0 \leq \beta < 1 - \gamma$,
  $\delta_\rho(\epsilon) = \delta_\rho^{**}(\epsilon) = \frac{1+\gamma+\rho-\beta}{2}\left[\epsilon-\frac{1}{2}\right]_+$.

  \item If $1 - \gamma \leq \beta < 1 + \gamma$,
  \begin{align*}
    \delta_\rho(\epsilon) = \begin{cases}
      (1+\gamma+\rho-\beta)\epsilon & \text{if $0 < \epsilon \leq \epsilon_0$,} \\
      \frac{\rho}{4} & \text{if $\epsilon_0 < \epsilon \leq \frac{1}{2}$,} \\
      \frac{\rho}{2}\epsilon & \text{if $\frac{1}{2} < \epsilon$,}
    \end{cases}
    \quad \text{and} \quad
    \delta_\rho^{**}(\epsilon) = \frac{\rho}{2}\epsilon.
  \end{align*}

  \item If $1 + \gamma \leq \beta < 1 + \gamma + \rho$,
  $\delta_\rho(\epsilon) = \delta_\rho^{**}(\epsilon) = \frac{1+\gamma+\rho-\beta}{2}\epsilon$.

  \item If $1 + \gamma + \rho \leq \beta$,
  $\delta_\rho(\epsilon) = \delta_\rho^{**}(\epsilon) = 0$.
\end{itemize}
We see that the ramp loss is $(\phi_\rho,\calFlin)$-calibrated when $1 - \gamma \leq \beta < 1 + \gamma$.

\subsection{Sigmoid Loss}
\label{sec:calibration-function:sigmoid}

\begin{figure}[t]
  \centering
  \scriptsize
  \def\ylim{1.1}
  \def\vareta{0.55}
  \def\vargamma{0.3}
  \def\varbeta{0.5}
  \includegraphics{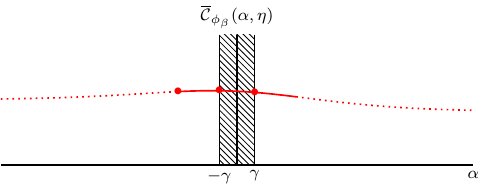}
  \caption{The class-conditional risk of the sigmoid loss.}
  \label{fig:ccr-sigmoid-loss}
\end{figure}

The sigmoid loss is $\phi(\alpha) = \frac{1}{1 + e^{\alpha}}$.
We consider the shifted sigmoid loss: $\phi_{\beta}(\alpha) = \frac{1}{1 + e^{\alpha - \beta}}$ for $\beta \geq 0$.
$\phi_{\beta}$-CCR is
\begin{align*}
  \calC_{\phi_\beta}(f, \eta, x)
  &= \bar\calC_{\phi_\beta}(f(x), \eta) \\
  &= \frac{\eta}{1 + e^{f(x) - \beta}} + \frac{1 - \eta}{1 + e^{-f(x) - \beta}}.
\end{align*}
$\bar\calC_{\phi_\beta}$ is plotted in Figure~\ref{fig:ccr-sigmoid-loss},
from which we can see $\bar\calC_{\phi_\beta}$ is quasiconcave when $\beta \geq 0$.

\subsubsection{Minimal Inner Risk}

By part~\ref{lem:ccr-qccv:ccr-inf} of Lemma~\ref{lem:ccr-qccv}, it is easy to check
\begin{align*}
  \calC_{\phi_\beta,\calFlin}^*(\eta, x)
  &= \inf_{f \in \calFlin} \calC_{\phi_\beta}(f, \eta, x) \\
  &= \inf_{\alpha \in [-\xnorm, \xnorm]} \bar\calC_{\phi_\beta}(\alpha, \eta) \\
  &= \min\{\bar\calC_{\phi_\beta}(-\xnorm, \eta), \bar\calC_{\phi_\beta}(\xnorm, \eta)\} \\
  &= \min\left\{
    \frac{\eta}{1+e^{\xnorm-\beta}} + \frac{1-\eta}{1+e^{-\xnorm-\beta}},
    \frac{\eta}{1+e^{-\xnorm-\beta}} + \frac{1-\eta}{1+e^{\xnorm-\beta}}
  \right\}.
\end{align*}

\subsubsection{Calibration Function}

We focus on the case $\eta > \frac{1}{2}$ due to the symmetry of $\calC_{\phi_\beta}$.
The minimal inner risk is
\begin{align*}
  \calC_{\phi_\beta,\calFlin}^*(\eta, x)
  = \bar\calC_{\phi_\beta}(\xnorm, \eta)
  = \frac{\eta}{1+e^{\xnorm-\beta}} + \frac{1-\eta}{1+e^{-\xnorm-\beta}}.
\end{align*}
We compute $\calC_{\phi_\beta}(f, \eta, x)$.
Since
\begin{align*}
  \bar\calC_{\phi_\beta}(-\gamma, \eta) - \bar\calC_{\phi_\beta}(\gamma, \eta)
  &= \left(\frac{\eta}{1 + e^{-\gamma-\beta}} + \frac{1 - \eta}{1 + e^{\gamma-\beta}}\right) - \left(\frac{\eta}{1 + e^{\gamma-\beta}} + \frac{1 - \eta}{1 + e^{-\gamma-\beta}}\right) \\
  &= (2\eta - 1)\left(\frac{1}{1 + e^{-\gamma-\beta}} - \frac{1}{1 + e^{\gamma-\beta}}\right) \\
  &> 0, \qquad \text{(since $-\gamma-\beta < \gamma-\beta$)}
\end{align*}
we have $\bar\calC_{\phi_\beta}(\gamma, \eta) < \bar\calC_{\phi_\beta}(-\gamma, \eta)$ for all $\eta > \frac{1}{2}$,
implying that $\inf_{|\alpha| \leq \gamma} \bar\calC_{\phi_\beta}(\alpha, \eta) = \bar\calC_{\phi_\beta}(\gamma, \eta)$
and $\inf_{f: |f(x)| \leq \gamma} \calC_{\phi_\beta}(f, \eta, x) = \bar\calC_{\phi_\beta}(\gamma, \eta)$.
Note that there exists $f \in \calFlin$ such that $f(x) = \gamma$ since we assume $\|x\|_2 > \gamma$.
On the other hand, we divide cases to compute $\inf_{f: f(x) \leq \gamma} \calC_{\phi_\beta}(f, \eta, x) = \inf_{\alpha \in [-\xnorm, \gamma]} \bar\calC_{\phi_\beta}(\alpha, \eta)$.
By part~\ref{lem:ccr-qccv:ccr-inf} of Lemma~\ref{lem:ccr-qccv},
\begin{align*}
  \inf_{f \in \calFlin: f(x) \leq \gamma} \calC_{\phi_\beta}(f, \eta, x)
  = \begin{cases}
    \bar\calC_{\phi_\beta}(-\xnorm, \eta) & \text{if $\bar\calC_{\phi_\beta}(-\xnorm) < \bar\calC_{\phi_\beta}(\gamma, \eta)$}, \\
    \bar\calC_{\phi_\beta}(\gamma, \eta) & \text{if $\bar\calC_{\phi_\beta}(-\xnorm) \geq \bar\calC_{\phi_\beta}(\gamma, \eta)$}.
  \end{cases}
\end{align*}
Thus, by Lemma~\ref{lem:mer-calibration-function},
we can compute $\delta$ by evaluating $\bar\delta$
by dividing the cases regarding $\eta$ and $x$.
If $\bar\calC_{\phi_\beta}(-\xnorm) < \bar\calC_{\phi_\beta}(\gamma, \eta)$ and $\eta \geq \frac{1+\epsilon}{2}$,
\begin{align*}
  \bar\delta(\epsilon, \eta, x)
  &= \bar\calC_{\phi_\beta}(-\xnorm, \eta) - \bar\calC_{\phi_\beta}(\xnorm, \eta) \\
  &= (2\eta-1)\left\{\phi_\beta(-\xnorm) - \phi_\beta(\xnorm)\right\}
  ,
\end{align*}
which is minimized at $\eta = \frac{1+\epsilon}{2}$ and $\xnorm = \gamma + \rho$
because $\phi_\beta(-\xnorm) - \phi_\beta(\xnorm) > 0$ is increasing in $\xnorm$
and the constraint
\begin{align*}
  \left\{\bar\calC_{\phi_\beta}(-\xnorm) < \bar\calC_{\phi_\beta}(\gamma, \eta)\right\} \wedge \left\{\eta \geq \frac{1+\epsilon}{2}\right\}
  \iff F(\xnorm) := \frac{\phi_\beta(-\gamma) - \phi_\beta(\xnorm)}{\phi_\beta(-\xnorm) - \phi_\beta(\gamma)} < \frac{1+\epsilon}{1-\epsilon}
\end{align*}
is always satisfied for any $\epsilon > 0$ and $x$ such that $\xnorm > \gamma$.
Note that $F(\xnorm)$ is increasing in $\xnorm$ thereby maximized at $\xnorm = 1$,
where $F(1) < 1 < \frac{1+\epsilon}{1-\epsilon}$.
Under the choice of the minimizers,
\begin{align*}
  \bar\delta\left(\epsilon, \frac{1+\epsilon}{2}, x\right)
  = \epsilon\left\{\phi_\beta(-\gamma-\rho) - \phi_\beta(\gamma+\rho)\right\}
  \define A_0\epsilon,
\end{align*}
where $A_0 \define \phi_\beta(-\gamma-\rho) - \phi_\beta(\gamma+\rho)$.

If $\bar\calC_{\phi_\beta}(-\xnorm) \geq \bar\calC_{\phi_\beta}(\gamma, \eta)$ and $\epsilon \leq \eta$,
we have $F(\xnorm) \geq \frac{\epsilon}{2-\epsilon}$.
This constraint with $\xnorm \geq \gamma + \rho$ is always satisfied
because $F$ is increasing and $F(\xnorm) > F(\gamma) = 1 \geq \frac{\epsilon}{2-\epsilon}$ for all $\epsilon \in (0, 1)$.
Consequently,
\begin{align*}
  \bar\delta(\epsilon, \eta, x)
  &= \bar\calC_{\phi_\beta}(\gamma, \eta) - \bar\calC_{\phi_\beta}(\xnorm, \eta) \\
  &= \left\{\phi_\beta(\gamma) -\phi_\beta(-\gamma) - \phi_\beta(\xnorm) + \phi_\beta(-\xnorm)\right\}\eta + \phi_\beta(-\gamma) - \phi_\beta(-\xnorm),
\end{align*}
which is minimized at $\eta = \max\left\{\epsilon, \frac{1}{2}\right\}$ and $\xnorm = \gamma + \rho$
because it is nondecreasing in both $\eta$ and $x$.
by noting that $-\bar\calC_{\phi_\beta}(\cdot, \eta)$ is nondecreasing (part~\ref{lem:ccr-qccv:ccr-decreasing} of Lemma~\ref{lem:ccr-qccv}).
Under the choice of the minimizers,
\begin{align*}
  \bar\delta\left(\epsilon, \max\left\{\epsilon, \frac{1}{2}\right\}, \gamma + \rho\right)
  &= A_1\left[\epsilon - \frac{1}{2}\right]_+ + \delta_0,
\end{align*}
where
\begin{align*}
  A_1 &\define \phi_\beta(\gamma) - \phi_\beta(-\gamma) - \phi_\beta(\gamma + \rho) + \phi_\beta(-\gamma - \rho), \\
  \delta_0 &\define \frac{\phi_\beta(\gamma) + \phi_\beta(-\gamma) - \phi_\beta(\gamma + \rho) - \phi_\beta(-\gamma - \rho)}{2}.
\end{align*}
Note that $\delta_0 > 0$ because $\phi_\beta(\alpha) + \phi_\beta(-\alpha)$ is nonincreasing in $\alpha \geq 0$
(part~\ref{lem:ccr-qccv:even-decreasing} of Lemma~\ref{lem:ccr-qccv}).

If $\bar\calC_{\phi_\beta}(-\xnorm) < \bar\calC_{\phi_\beta}(\gamma, \eta)$ and $\epsilon \leq \eta < \frac{1+\epsilon}{2}$,
which is equivalent to $\frac{1 + \epsilon}{1 - \epsilon} < F(\xnorm) \leq \frac{\epsilon}{2-\epsilon}$,
this is never satisfied because $\frac{1+\epsilon}{1-\epsilon} \geq \frac{\epsilon}{2-\epsilon}$ for all $\epsilon \in (0, 1)$.

By combining these cases, we have
\begin{align*}
  \delta_\rho(\epsilon)
  &= \inf_{\eta \in \left[\frac{1}{2}, 1\right]} \inf_{\xnorm > \gamma + \rho} \bar\delta(\epsilon, \eta, x)
  = \begin{cases}
    A_0\epsilon & \text{if $0 < \epsilon \leq \epsilon_0$}, \\
    \delta_0 & \text{if $\epsilon_0 < \epsilon \leq \frac{1}{2}$}, \\
    A_1\left(\epsilon-\frac{1}{2}\right) + \delta_0 & \text{if $\epsilon > \frac{1}{2}$},
  \end{cases} \\
  \delta_\rho^{**}(\epsilon) &= \begin{cases}
    A_0\epsilon & \text{if $0 < \epsilon \leq \frac{1}{2}$,} \\
    A_1\left(\epsilon-\frac{1}{2}\right) + \delta_0 & \text{if $\frac{1}{2} < \epsilon$,}
  \end{cases}
\end{align*}
where $\epsilon_0 \define \frac{\delta_0}{A_0}$.

Thus, the sigmoid loss is calibrated wrt $(\phi_\gamma,\calFlin)$ when $\delta_0 > 0$.
This always holds as long as $\beta > 0$.

\subsection{Modified Squared Loss}
\label{sec:calibration-function:modified-squared}

We design a bounded and nonincreasing surrogate loss by modifying the squared loss,
which we call modified squared loss here:
\begin{align}
  \phi(\alpha) = \begin{cases}
    1 & \text{if $\alpha \leq 0$}, \\
    (1 - \alpha)^2 & \text{if $0 < \alpha \leq 1$}, \\
    0 & \text{if $1 < \alpha$},
  \end{cases}
  \nonumber
\end{align}
and consider the shifted version $\phi_\beta(\alpha) \define \phi(\alpha - \beta)$:
\begin{align}
  \phi_\beta(\alpha) = \begin{cases}
    1 & \text{if $\alpha \leq \beta$}, \\
    (1 - \alpha + \beta)^2 & \text{if $\beta < \alpha \leq 1 + \beta$}, \\
    0 & \text{if $1 + \beta < \alpha$}.
  \end{cases}
  \nonumber
\end{align}
$\bar\calC_{\phi_\beta}$ is plotted in Figure~\ref{fig:ccr-modified-squared-loss},
from which we can see $\bar\calC_{\phi_\beta}$ is quasiconcave when $\beta \geq 0$.

\subsubsection{Calibration Function}

Now we consider $\phi_\beta$-CCR $\calC_{\phi_\beta}(f, \eta, x) = \bar\calC_{\phi_\beta}(\alpha, \eta) = \eta\phi(\alpha) + (1 - \eta)\phi(-\alpha)$,
where $\alpha = f(x)$,
and focus on the case $\eta > \frac{1}{2}$ due to the symmetry of $\bar\calC_{\phi_\beta}$ in $\eta$
(part~\ref{lem:ccr-property:eta-symmetry} of Lemma~\ref{lem:ccr-property}).
By part~\ref{lem:ccr-qccv:ccr-inf} of Lemma~\ref{lem:ccr-qccv},
it is easy to see
\begin{align*}
  \calC_{\phi_\beta,\calFlin}^*(\eta, x)
  = \min\{\bar\calC_{\phi_\beta}(-\xnorm, \eta), \bar\calC_{\phi_\beta}(\xnorm, \eta)\}
  = \bar\calC_{\phi_\beta}(\xnorm, \eta).
\end{align*}
We divide into three cases depending on the relationship among
$\bar\calC_{\phi_\beta}(-\xnorm, \eta)$, $\bar\calC_{\phi_\beta}(-\gamma, \eta)$, and $\bar\calC_{\phi_\beta}(\xnorm, \eta)$,

\header{(A) When $0 \leq \beta < \gamma$}
Since
\begin{align*}
  \bar\calC_{\phi_\beta}(-\gamma, \eta) - \bar\calC_{\phi_\beta}(\gamma, \eta)
  &= \left\{\eta\cdot 1 + (1 - \eta)(1 - \gamma + \beta)^2\right\} - \left\{\eta(1 - \gamma + \beta)^2 + (1 - \eta)\cdot 1\right\} \\
  &= (2\eta - 1)(\gamma - \beta)\left\{2 - (\gamma - \beta)\right\} \\
  &\geq 0,
\end{align*}
we have $\bar\calC_{\phi_\beta}(\gamma, \eta) < \bar\calC_{\phi_\beta}(-\gamma, \eta)$ for all $\eta > \frac{1}{2}$.
On the other hand, since
\begin{align*}
  \bar\calC_{\phi_\beta}(\gamma, \eta) - \bar\calC_{\phi_\beta}(-\xnorm, \eta)
  &= -\left\{(1-(1-\gamma+\beta)^2)+(1-(1-\xnorm+\beta)^2)\right\}(\eta - \eta_0(x)) \\
  & \text{where} \quad
  \eta_0(x) \define \frac{1-(1-\xnorm+\beta)^2}{(1-(1-\gamma+\beta)^2)+(1-(1-\xnorm+\beta)^2)}
\end{align*}
and $\frac{1}{2} < \eta_0(x) < 1$,
we have $\bar\calC_{\phi_\beta}(\gamma, \eta) \geq \bar\calC_{\phi_\beta}(-\xnorm, \eta)$ if $\frac{1}{2} < \eta \leq \eta_0(x)$
and $\bar\calC_{\phi_\beta}(\gamma, \eta) < \bar\calC_{\phi_\beta}(-\xnorm, \eta)$ if $\eta > \eta_0(x)$.
\begin{itemize}
  \item \underline{If $\frac{1}{2} < \eta \leq \eta_0(x)$:}
  By part~\ref{lem:ccr-qccv:ccr-inf} in Lemma~\ref{lem:ccr-qccv},
  \begin{align*}
    \inf_{f \in \calFlin: |f(x)| \leq \gamma \text{ or } (2\eta-1)f(x) \leq 0} \calC_{\phi_\beta}(f, \eta, x)
    &= \inf_{\alpha \in [-\xnorm, \gamma]} \bar\calC_{\phi_\beta}(\alpha, \eta)
    = \bar\calC_{\phi_\beta}(-\xnorm, \eta), \\
    \inf_{f \in \calFlin: |f(x)| \leq \gamma} \calC_{\phi_\beta}(f, \eta, x)
    &= \inf_{|\alpha| \leq \gamma} \bar\calC_{\phi_\beta}(\alpha, \eta)
    = \bar\calC_{\phi_\beta}(\gamma, \eta).
  \end{align*}
  Thus, by Lemma~\ref{lem:mer-calibration-function},
  \begin{align*}
    \bar\delta(\epsilon, \eta, x) = \begin{cases}
      \infty & \text{if $\eta < \epsilon$}, \\
      \bar\calC_{\phi_\beta}(\gamma, \eta) - \bar\calC_{\phi_\beta}(\xnorm, \eta) \\
      \quad = (\xnorm-\gamma)(2+2\beta-\gamma-\xnorm)\eta & \text{if $\epsilon \leq \eta < \frac{1 + \epsilon}{2}$}, \\
      \bar\calC_{\phi_\beta}(-\xnorm, \eta) - \bar\calC_{\phi_\beta}(\xnorm, \eta) \\
      \quad = (\xnorm-\beta)(2-\xnorm+\beta)(2\eta - 1) & \text{if $\frac{1 + \epsilon}{2} \leq \eta$}.
    \end{cases}
  \end{align*}
  Hence we obtain
  \begin{align*}
    \inf_{\eta \in \left(\frac{1}{2}, \eta_0(x)\right]} \inf_{\xnorm \geq \gamma + \rho} \bar\delta(\epsilon, \eta, x)
    = \begin{cases}
      A_0\epsilon & \text{if $0 < \epsilon \leq \epsilon_0$}, \\
      \delta_0 & \text{if $\epsilon_0 < \epsilon \leq \frac{1}{2}$}, \\
      A_1\epsilon & \text{if $\frac{1}{2} < \epsilon$},
    \end{cases}
  \end{align*}
  where
  $A_0 \define (\gamma+\rho-\beta)(2+\beta-\gamma-\rho)$,
  $A_1 \define \rho(2+2\beta-2\gamma-\rho)$,
  $\delta_0 \define \frac{A_1}{2}$, and
  $\epsilon_0 \define \frac{\delta_0}{A_0}$.
  Note that the second case would not degenerate ($\delta_0 > 0$).

  \item \underline{If $\eta_0(x) < \eta \leq 1$:}
  By part~\ref{lem:ccr-qccv:ccr-inf} in Lemma~\ref{lem:ccr-qccv},
  it follows that
  \begin{align*}
    \inf_{f \in \calFlin: |f(x)| \leq \gamma \text{ or } (2\eta-1)f(x) \leq 0} \calC_{\phi_\beta}(f, \eta, x)
    &= \inf_{\alpha \in [-\xnorm, \gamma]} \bar\calC_{\phi_\beta}(\alpha, \eta)
    = \bar\calC_{\phi_\beta}(\gamma, \eta), \\
    \inf_{f \in \calFlin: |f(x)| \leq \gamma} \calC_{\phi_\beta}(f, \eta, x)
    &= \inf_{|\alpha| \leq \gamma} \bar\calC_{\phi_\beta}(\alpha, \eta)
    = \bar\calC_{\phi_\beta}(\gamma, \eta).
  \end{align*}
  Thus, by Lemma~\ref{lem:mer-calibration-function},
  \begin{align*}
    \bar\delta(\epsilon, \eta, x) = \begin{cases}
      \infty & \text{if $\eta < \epsilon$,} \\
      \bar\calC_{\phi_\beta}(\gamma, \eta) - \bar\calC_{\phi_\beta}(\xnorm, \eta)
      = (\xnorm-\gamma)(2+2\beta-\gamma-\xnorm)\eta
      & \text{if $\epsilon \leq \eta$}.
    \end{cases}
  \end{align*}
  Hence we obtain
  \begin{align*}
    \inf_{\eta \in \left(\eta_0(x), 1\right]} \inf_{\xnorm \geq \gamma + \rho} \bar\delta(\epsilon, \eta)
    = \begin{cases}
      \eta_0(\gamma+\rho)\epsilon & \text{if $ 0 < \epsilon \leq \eta_0(\gamma+\rho)$}, \\
      A_1\epsilon & \text{if $\eta_0(\gamma+\rho) < \epsilon$}.
    \end{cases}
  \end{align*}
\end{itemize}
Note that $\eta_0(\gamma+\rho) > \frac{1}{2}$.
Combining the above, we obtain the $(\phi_\gamma,\calFlin)$-calibration function from Lemma~\ref{lem:mer-calibration-function}:
\begin{align*}
  \delta_\rho(\epsilon)
  = \begin{cases}
    A_0\epsilon & \text{if $0 < \epsilon \leq \epsilon_0$}, \\
    \delta_0 & \text{if $\epsilon_0 < \epsilon \leq \frac{1}{2}$}, \\
    A_1\epsilon & \text{if $\frac{1}{2} < \epsilon$},
  \end{cases}
\end{align*}
where
$A_0 = (\gamma+\rho-\beta)(2+\beta-\gamma-\rho)$,
$A_1 = \rho(2+2\beta-2\gamma-\rho)$,
$\delta_0 = \frac{A_1}{2}$, and
$\epsilon_0 = \frac{\delta_0}{A_0}$.

\header{(B) When $\gamma \leq \beta < 1$}
It is easy to see
\begin{align*}
  \inf_{f \in \calFlin: |f(x)| \leq \gamma} \calC_{\phi_\beta}(f, \eta, x)
  &= \inf_{|\alpha| \leq \gamma} \bar\calC_{\phi_\beta}(\alpha, \eta)
  = 1, \\
  \inf_{f \in \calFlin: |f(x)| \leq \gamma \text{ or } (2\eta-1)f(x) \leq 0} \calC_{\phi_\beta}(f, \eta, x)
  &= \inf_{\alpha \in [-\xnorm, \gamma]} \bar\calC_{\phi_\beta}(\alpha, \eta)
  = \bar\calC_{\phi_\beta}(-\xnorm, \eta).
\end{align*}
Hence, by noting that $\bar\calC_{\phi_\beta}(\xnorm, \eta) = \bar\calC_{\phi_\beta}(-\xnorm, \eta) = 1$ for $\xnorm \leq \beta$,
it follows that
\begin{align*}
  \bar\delta(\epsilon, \eta, x) = \begin{cases}
    \infty & \text{if $\eta < \epsilon$}, \\
    1 - \bar\calC_{\phi_\beta}(\xnorm, \eta)
    = (\xnorm - \beta)(2 - \xnorm + \beta)\eta
    & \text{if $\epsilon \leq \eta < \frac{1 + \epsilon}{2}$ and $\xnorm > \beta$}, \\
    \bar\calC_{\phi_\beta}(-\xnorm, \eta) - \bar\calC_{\phi_\beta}(\xnorm, \eta) \\
    \qquad = (\xnorm - \beta)(2 - \xnorm + \beta)(2\eta-1)
    & \text{if $\frac{1 + \epsilon}{2} \leq \eta$ and $\xnorm > \beta$} \\
    0 & \text{if $\epsilon \leq \eta$ and $\xnorm \leq \beta$}.
  \end{cases}
\end{align*}
Thus, by Lemma~\ref{lem:mer-calibration-function},
we have $\delta_\rho(\epsilon) = 0$ when $\beta \geq \gamma + \rho$,
and
\begin{align*}
  \delta_\rho(\epsilon)
  = \inf_{\eta \in \left(\frac{1}{2}, 1\right]} \inf_{\xnorm \geq \gamma + \rho} \bar\delta(\epsilon, \eta, x)
  &= A_0\epsilon
\end{align*}
when $\beta < \gamma + \rho$,
where $A_0 = (\gamma + \rho - \beta)(2 + \beta - \gamma - \rho)$.

\header{(C) When $1 \leq \beta$}
In this case, $\bar\calC_{\phi_\beta}(\alpha, \eta) = 1$ for all $\alpha \in [-1, 1]$.
Hence, $\Delta\calC_{\phi_\beta,\calFlin}(f, \eta, x) = 0$ and $\delta_\rho(\epsilon) = 0$.

To sum up, the $(\phi_\gamma,\calFlin)$-calibration function and its Fenchel-Legendre biconjugate
of the modified squared loss are as follows:
\begin{itemize}
  \item If $0 \leq \beta < \gamma$,
  \begin{align*}
    \delta_\rho(\epsilon)
    = \begin{cases}
      A_0\epsilon & \text{if $0 < \epsilon \leq \epsilon_0$}, \\
      \delta_0 & \text{if $\epsilon_0 < \epsilon \leq \frac{1}{2}$}, \\
      A_1\epsilon & \text{if $\frac{1}{2} < \epsilon$},
    \end{cases}
    \quad \text{and} \quad
    \delta_\rho^{**}(\epsilon) = A_1\epsilon,
  \end{align*}
  where
  $A_0 = (\gamma+\rho-\beta)(2+\beta-\gamma-\rho)$,
  $A_1 = \rho(2+2\beta-2\gamma-\rho)$,
  $\delta_0 = \frac{A_1}{2}$, and
  $\epsilon_0 = \frac{\delta_0}{A_0}$.

  \item If $\gamma \leq \beta < \gamma + \rho$, $\delta_\rho(\epsilon) = \delta_\rho^{**}(\epsilon) = A_0\epsilon$.

  \item If $\gamma + \rho \leq \beta$, $\delta(\epsilon) = \delta^{**}(\epsilon) = 0$.
\end{itemize}
We deduce that the modified squared loss is calibrated wrt $(\phi_\gamma,\calFlin)$ if $0 \leq \beta \leq \gamma$.

\begin{figure}[t]
  \centering
  \scriptsize
  \def\ylim{1.1}
  \def\vareta{0.75}
  \def\vargamma{0.3}
  \def\changepoints{{-1-\varbeta},{-\varbeta},{\varbeta},{1+\varbeta}}
  \subfigure[$0 \leq \beta < \gamma$][c]{
    \def\varbeta{0.2}
    \includegraphics{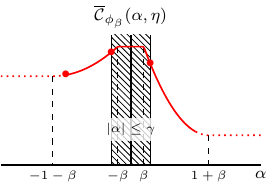}
  } \hfill
  \subfigure[$\gamma \leq \beta < 1$][c]{
    \def\varbeta{0.5}
    \includegraphics{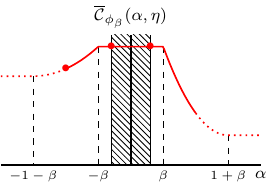}
  } \hfill
  \subfigure[$1 \leq \beta$][c]{
    \def\varbeta{1.2}
    \includegraphics{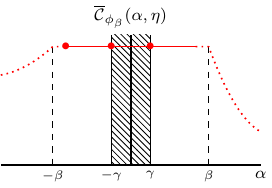}
  }
  \caption{The class-conditional risk of the modified squared loss.}
  \label{fig:ccr-modified-squared-loss}
\end{figure}

\subsubsection{When \texorpdfstring{$\beta < 0$}{beta < 0}}

\begin{figure}[t]
  \centering
  \scriptsize
  \def\ylim{0.9}
  \def\vargamma{0.3}
  \def\changepoints{{-1-\varbeta},{-\varbeta},{\varbeta},{1+\varbeta}}
  \def\varbeta{-0.2}
  \subfigure[$\eta = 0.7$][c]{
    \includegraphics{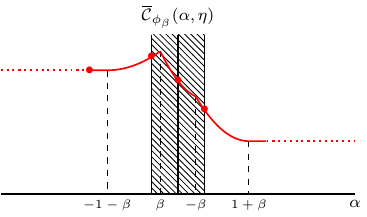}
  }
  \subfigure[$\eta = 0.5$][c]{
    \includegraphics{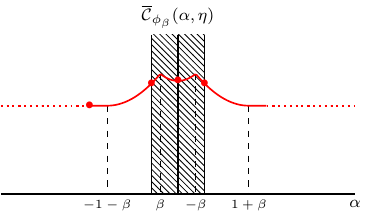}
  }
  \caption{The class-conditional risk of the modified squared loss when $\gamma < \frac{2}{5}$ and $-1 - \gamma + \sqrt{1+2\gamma^2} < \beta < 0$.}
  \label{fig:ccr-modified-squared-loss:non-qccv}
\end{figure}

In this case, the CCR of the modified squared loss is no longer quasiconcave (see Figure~\ref{fig:ccr-modified-squared-loss:non-qccv} (b)).
However, $\phi_\beta$ is still $(\phi_\gamma,\calFlin)$-calibrated
under some $\gamma$ and $\beta < 0$.
Here, we show an example.

Assume that $0 < \gamma < \frac{2}{5}$ and $(-\gamma<)-1-\gamma+\sqrt{1+2\gamma^2} < \beta < 0$.
We focus on $\eta > \frac{1}{2}$ due to the symmetry of $\bar\calC_{\phi_\beta}(\alpha, \eta)$ in $\eta$
(part~\ref{lem:ccr-property:eta-symmetry} of Lemma~\ref{lem:ccr-property}).
Since we still have $\eta_0(x) > \frac{1}{2}$.
we can confirm in the same way as the case (A) that
$\bar\calC_{\phi_\beta}(-\gamma, \eta) > \bar\calC_{\phi_\beta}(\gamma, \eta)$,
$\bar\calC_{\phi_\beta}(\gamma, \eta) \geq \bar\calC_{\phi_\beta}(-\xnorm, \eta)$ if $\frac{1}{2} < \eta \leq \eta_0(x)$,
and $\bar\calC_{\phi_\beta}(\gamma, \eta) < \bar\calC_{\phi_\beta}(-\xnorm, \eta)$ if $\eta_0(x) < \eta$.
In addition, we see that
\begin{align*}
  \bar\calC_{\phi_\beta}(-\xnorm, \eta) - \bar\calC_{\phi_\beta}(0, \eta)
  &= \left\{\eta + (1-\eta)(1-\xnorm+\beta)^2\right\} - (1+\beta)^2 \\
  &= \eta(1-(1-\xnorm+\beta^2)) -\xnorm(2+2\beta-\xnorm) \\
  &> \frac{1}{2}(1-(1-\xnorm+\beta^2)) - \xnorm(2+2\beta-\xnorm) \\
  &> \frac{1}{2}(1-(1-\xnorm)^2) - \xnorm(2-\xnorm) \\
  & \qquad \text{(nonincreasing in $-1-\gamma+\sqrt{1+2\gamma^2}<\beta<0$)} \\
  &= \frac{1}{2}\xnorm(\xnorm - 2) \\
  &< 0, \\
  \bar\calC_{\phi_\beta}(0, \eta) - \bar\calC_{\phi_\beta}(\gamma, \eta)
  &= (1+\beta)^2 - \left\{\eta(1-\gamma+\beta)^2+(1-\eta)\right\} \\
  &= (1+\beta)^2 - 1 + \eta(1-(1-\gamma+\beta)^2) \\
  &> (1+\beta)^2 - 1 + \frac{1}{2}(1-(1-\gamma+\beta)^2) \\
  &> (1+\beta)^2 - 1 + \frac{1}{2}(1-(1-\gamma)^2) \\
  &= (1+\beta)^2 + (1+\gamma)^2 - \frac{1}{2} \\
  &> 0, \\
  \bar\calC_{\phi_\beta}(\gamma) - \bar\calC_{\phi_\beta}(\xnorm)
  &> 0.
\end{align*}
Then, we have
$\bar\calC_{\phi_\beta}(0, \eta) > \bar\calC_{\phi_\beta}(-\xnorm, \eta)$ and
$\bar\calC_{\phi_\beta}(0, \eta) > \bar\calC_{\phi_\beta}(\gamma, \eta) > \bar\calC_{\phi_\beta}(\xnorm, \eta)$.
Figure~\ref{fig:ccr-modified-squared-loss:non-qccv} and the above comparisons give us
\begin{align*}
  \calC_{\phi_\beta,\calFlin}^*(\eta, x)
  &= \inf_{\alpha \in [-\xnorm, \xnorm]} \bar\calC_{\phi_\beta}(\alpha, \eta)
  = \calC_{\phi_\beta}(\xnorm, \eta), \\
  \inf_{f \in \calF: |f(x)| \leq \gamma} \calC_{\phi_\beta}(f, \eta, x)
  &= \inf_{|\alpha| \leq \gamma} \bar\calC_{\phi_\beta}(\alpha, \eta)
  = \bar\calC_{\phi_\beta}(\gamma, \eta), \\
  \inf_{f \in \calF: f(x) \leq \gamma} \calC_{\phi_\beta}(f, \eta, x)
  &= \inf_{\alpha \in [-\xnorm, \gamma]} \bar\calC_{\phi_\beta}(\alpha, \eta)
  = \min\{\bar\calC_{\phi_\beta}(-\xnorm, \eta), \bar\calC_{\phi_\beta}(\gamma, \eta)\}. \\
\end{align*}
By Lemma~\ref{lem:mer-calibration-function},
when $\epsilon \leq \eta < \frac{1+\epsilon}{2}$,
\begin{align*}
  \bar\delta(\epsilon, \eta, x)
  &= \inf_{f \in \calFlin: |f(x)| \leq \gamma} \Delta\calC_{\phi_\beta}(f, \eta, x) \\
  &= \bar\calC_{\phi_\beta}(\gamma, \eta) - \bar\calC_{\phi_\beta}(\xnorm, \eta) \\
  &= \left\{\eta(1-\gamma+\beta)^2+(1-\eta)\right\} - \left\{\eta(1-\xnorm+\beta)^2+(1-\eta)\right\} \\
  &= (\xnorm-\gamma)(2+2\beta-\gamma-\xnorm)\eta,
\end{align*}
and
\begin{align*}
  \inf_{\eta \in \left[\epsilon, \frac{1+\epsilon}{2}\right] \cap \left(\frac{1}{2}, 1\right]} \inf_{\xnorm \geq \gamma+\rho} \bar\delta(\epsilon, \eta, x)
  &= \inf_{\xnorm \geq \gamma+\rho} (\xnorm-\gamma)(2+2\beta-\gamma-\xnorm)\max\left\{\epsilon,\frac{1}{2}\right\} \\
  &= \begin{cases}
    \delta_0 & \text{if $0 < \epsilon \leq \frac{1}{2}$,} \\
    A_1\epsilon & \text{if $\frac{1}{2} < \epsilon$.}
  \end{cases}
\end{align*}
When $\frac{1+\epsilon}{2} \leq \eta$,
\begin{align*}
  \bar\delta(\epsilon, \eta)
  &= \inf_{f \in \calFlin: f(x) \leq \gamma} \Delta\calC_{\phi_\beta}(f, \eta, x) \\
  &= \min\left\{
    \bar\calC_{\phi_\beta}(-\xnorm, \eta) - \bar\calC_{\phi_\beta}(\xnorm, \eta),
    \bar\calC_{\phi_\beta}(\gamma, \eta) - \bar\calC_{\phi_\beta}(\xnorm, \eta)
  \right\} \\
  &= \min\left\{
    (1-(1-\xnorm+\beta))^2(2\eta-1),
    (\xnorm-\gamma)(2+2\beta-\gamma-\xnorm)\eta
  \right\},
\end{align*}
and
\begin{align*}
  & \inf_{\eta \in \left[\frac{1+\epsilon}{2}, 1\right] \cap \left(\frac{1}{2}, 1\right]} \inf_{\xnorm \geq \gamma+\rho} \bar\delta(\epsilon, \eta, x) \\
  &\quad = \inf_{\xnorm \geq \gamma+\rho} \min\left\{
    (1-(1-\xnorm+\beta)^2)\epsilon,
    (\xnorm-\gamma)(2+2\beta-\gamma-\xnorm)\frac{1+\epsilon}{2}
  \right\} \\
  &= \min\left\{
    A_0\epsilon,
    A_1\frac{1+\epsilon}{2}
  \right\}.
\end{align*}

Hence, $(\phi_\gamma,\calFlin)$-calibration function of $\phi_\beta$ is
\begin{align*}
  \delta_\rho(\epsilon)
  &= \begin{cases}
    A_0\epsilon & \text{if $0 < \epsilon \leq \epsilon_0$,} \\
    \delta_0 & \text{if $\epsilon_0 < \epsilon \leq \frac{1}{2}$,} \\
    A_1\epsilon & \text{if $\frac{1}{2} < \epsilon$,}
  \end{cases}
\end{align*}
where
$A_0 = (\gamma+\rho-\beta)(2+\beta-\gamma-\rho)$,
$A_1 = \rho(2+2\beta-2\gamma-\rho)$,
$\delta_0 = \frac{A_1}{2}$, and
$\epsilon_0 = \frac{\delta_0}{A_0}$.
We can see that the second case would not degenerate (i.e., $\delta_0 > 0$)
under the range $-1-\gamma+\sqrt{1+2\gamma^2} < \beta < 0$ and $0 < \gamma \leq \frac{2}{5}$.

\subsection{Hinge Loss}
\label{sec:calibration-function:hinge}

The $\phi_\beta$-CCR is $\calC_{\phi_\beta}(f, \eta, x) = \bar\calC_{\phi_\beta}(f(x), \eta)$, where
\begin{align*}
  \bar\calC_{\phi_\beta}(\alpha, \eta) = \begin{cases}
    -\eta\alpha + \eta(1 + \beta) & \text{if $\alpha < -(1 + \beta)$}, \\
    (1 - 2\eta)\alpha + (1 + \beta) & \text{if $-(1 + \beta) \leq \alpha < 1 + \beta$}, \\
    (1 - \eta)\alpha + (1 - \eta)(1 + \beta) & \text{if $1 + \beta < \alpha$}.
  \end{cases}
\end{align*}

\subsubsection{Minimal Inner Risk}

When $\eta > \frac{1}{2}$, $\bar\calC_{\phi_\beta}(\alpha, \eta)$ is minimized at $\alpha = \xnorm$,
and when $\eta \leq \frac{1}{2}$, $\bar\calC_{\phi_\beta}(\alpha, \eta)$ is minimized at $\alpha = -\xnorm$.
Hence,
\begin{align*}
  \calC_{\phi_\beta,\calFlin}^*(\eta, x)
  &= \inf_{\alpha \in [-\xnorm, \xnorm]} \bar\calC_{\phi_\beta}(\alpha, \eta) \\
  &= \begin{cases}
    \bar\calC_{\phi_\beta}(\xnorm, \eta) & \text{if $\eta > \frac{1}{2}$} \\
    \bar\calC_{\phi_\beta}(-\xnorm, \eta) & \text{if $\eta \leq \frac{1}{2}$}
  \end{cases} \\
  &= -|1 - 2\eta|\cdot\xnorm + 1 + \beta.
\end{align*}

\subsubsection{Calibration Function}

We restrict the range of $\eta$ to $\eta > \frac{1}{2}$
by virtue of part~\ref{lem:ccr-property:eta-symmetry} of Lemma~\ref{lem:ccr-property}.
Then, $\calC_{\phi_\beta,\calFlin}^*(\eta, x) = \bar\calC_{\phi_\beta}(\xnorm, \eta)$.
$\bar\calC_{\phi_\beta}(\alpha, \eta)$ is plotted in Figure~\ref{fig:ccr-hinge-loss} in case of $\eta > \frac{1}{2}$.
From the figure, we can see that
\begin{align*}
  \inf_{f \in \calFlin: f(x) \leq \gamma} \calC_{\phi_\beta}(f, \eta, x)
  &= \inf_{\substack{\alpha \in [-\xnorm, \xnorm]: |\alpha| \leq \gamma \text{ or } \\ (2\eta - 1)\alpha \leq 0}} \bar\calC_{\phi_\beta}(\alpha, \eta) \\
  &= \bar\calC_{\phi_\beta}(\gamma, \eta) \\
  &= \inf_{|\alpha| \leq \gamma} \bar\calC_{\phi_\beta}(\alpha, \eta) \\
  &= \inf_{f \in \calFlin: |f(x)| \leq \gamma} \calC_{\phi_\beta}(f, \eta, x),
\end{align*}
by noting that $\xnorm > \gamma$ is assumed.
Hence, by Lemma~\ref{lem:mer-calibration-function},
\begin{align*}
  \bar\delta(\epsilon, \eta, x) = \begin{cases}
    \infty & \text{if $\xnorm \leq \gamma$ or $\eta < \epsilon$}, \\
    \bar\calC_{\phi_\beta}(\gamma, \eta) - \bar\calC_{\phi_\beta,\calFlin}^*(\xnorm, \eta) = (\xnorm-\gamma)(2\eta-1) & \text{if $\xnorm > \gamma$ and $\epsilon \leq \eta$},
  \end{cases}
\end{align*}
for $\eta > \frac{1}{2}$,
and
\begin{align*}
  \delta_\rho(\epsilon)
  = \inf_{\eta \in \left[\frac{1}{2}, 1\right]} \inf_{x \in \tcalX_\rho} \bar\delta(\epsilon, \eta, x)
  = \begin{cases}
    0 & \text{if $0 < \epsilon \leq \frac{1}{2}$}, \\
    2\rho\left(\epsilon - \frac{1}{2}\right) & \text{if $\frac{1}{2} < \epsilon$},
  \end{cases}
\end{align*}
and $\delta_\rho^{**}(\epsilon) = \delta_\rho(\epsilon)$.

\begin{figure}[t]
  \begin{minipage}[c]{0.49\columnwidth}
    \centering
    \scriptsize
    \def\ylim{2.5}
    \def\vareta{0.75}
    \def\vargamma{0.3}
    \def\varbeta{0.3}
    \includegraphics{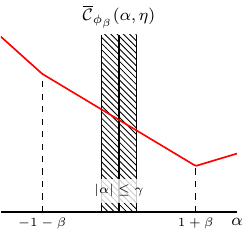}
    \caption{The class-conditional risk of the hinge loss.}
    \label{fig:ccr-hinge-loss}
  \end{minipage}
  \hfill
  \begin{minipage}[c]{0.49\columnwidth}
    \centering
    \scriptsize
    \def\ylim{3.1}
    \def\vareta{0.75}
    \def\vargamma{0.3}
    \def\varbeta{0.2}
    \includegraphics{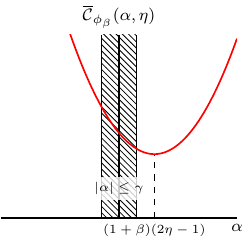}
    \caption{The class-conditional risk of the squared loss.}
    \label{fig:ccr-squared-loss}
  \end{minipage}
\end{figure}

\subsection{Squared Loss}
\label{sec:calibration-function:squared}

The $\phi_\beta$-CCR is $\calC_{\phi_\beta}(f, \eta, x) = \bar\calC_{\phi_\beta}(f(x), \eta)$, where
\begin{align*}
  \bar\calC_{\phi_\beta}(\alpha, \eta)
  &= \eta(1 - \alpha + \beta)^2 + (1 - \eta)(1 + \alpha + \beta)^2 \nonumber \\
  &= \{\alpha - (1 + \beta)(2\eta - 1)\}^2 + 4(1 + \beta)^2\eta(1 - \eta).
\end{align*}
Let $\alpha_* \define (1+\beta)(2\eta-1)$.

\subsubsection{Minimal Inner Risk}

When $\eta > \frac{1}{2}$, $\bar\calC_{\phi_\beta}(\alpha, \eta)$ is minimized at $\alpha = \xnorm$ if $\xnorm \geq \alpha_*$, and at $\alpha = \alpha_*$ if $\xnorm < \alpha_*$.
When $\eta \leq \frac{1}{2}$, $\bar\calC_{\phi_\beta}(\alpha, \eta$) is minimized at $\alpha = -\xnorm$ if $\xnorm \leq -\alpha_*$, and at $\alpha = \alpha_*$ if $\xnorm > -\alpha_*$.
Hence,
\begin{align*}
  \calC_{\phi_\beta,\calFlin}^*(\eta, x)
  &= \inf_{\alpha \in [-\xnorm, \xnorm]} \bar\calC_{\phi_\beta}(\alpha, \eta) \\
  &= \begin{cases}
    \bar\calC_{\phi_\beta}(\xnorm, \eta) & \text{if $\eta > \frac{1}{2}$ and $\xnorm \geq \alpha_*$}, \\
    \bar\calC_{\phi_\beta}(-\xnorm, \eta) & \text{if $\eta \leq \frac{1}{2}$ and $\xnorm \leq -\alpha_*$}, \\
    \bar\calC_{\phi_\beta}(\alpha_*, \eta) & \text{otherwise}.
  \end{cases}
\end{align*}

\subsubsection{Calibration Function}

We restrict the range of $\eta$ to $\eta > \frac{1}{2}$
by virtue of part~\ref{lem:ccr-property:eta-symmetry} of Lemma~\ref{lem:ccr-property}.
$\bar\calC_{\phi_\beta}(\alpha, \eta)$ is plotted in Figure~\ref{fig:ccr-squared-loss} in case of $\eta > \frac{1}{2}$.
By comparing $\alpha_*$ and $\xnorm$, we have
\begin{align*}
  \calC_{\phi_\beta,\calFlin}^*(\eta, x) = \begin{cases}
    \bar\calC_{\phi_\beta}(\alpha_*, \eta) & \text{if $\alpha_* < \xnorm$}, \\
    \bar\calC_{\phi_\beta}(\xnorm, \eta) & \text{if $\alpha_* \geq \xnorm$}.
  \end{cases}
\end{align*}
From the figure, we can see that
\begin{align*}
  \inf_{f \in \calFlin: f(x) \leq \gamma} \calC_{\phi_\beta}(f, \eta, x)
  &= \inf_{\substack{\alpha \in [-\xnorm, \xnorm]: |\alpha| \leq \gamma \text{ or } \\ (2\eta - 1)\alpha \leq 0}} \bar\calC_{\phi_\beta}(\alpha, \eta) \\
  &= \inf_{|\alpha| \leq \gamma} \bar\calC_{\phi_\beta}(\alpha, \eta) \\
  &= \inf_{f \in \calFlin} \calC_{\phi_\beta}(f, \eta, x) \\
  &= \begin{cases}
    \bar\calC_{\phi_\beta}(\alpha_*, \eta) & \text{if $\gamma > \alpha_*$}, \\
    \bar\calC_{\phi_\beta}(\gamma, \eta) & \text{if $\gamma \leq \alpha_*$},
  \end{cases}
\end{align*}
by noting that $\xnorm > \gamma$ is assumed.
Hence, by Lemma~\ref{lem:mer-calibration-function},
\begin{align*}
  \bar\delta(\epsilon, \eta, x)
  &= \begin{cases}
    \infty & \text{if $\xnorm \leq \gamma$ or $\eta < \epsilon$}, \\
    \bar\calC_{\phi_\beta}(\alpha_*, \eta) - \bar\calC_{\phi_\beta}(\alpha_*, \eta) & \text{if $\xnorm > \gamma$ and $\epsilon \leq \eta$ and $\alpha_* < \gamma$}, \\
    \bar\calC_{\phi_\beta}(\gamma, \eta) - \bar\calC_{\phi_\beta}(\alpha_*, \eta) & \text{if $\xnorm > \gamma$ and $\epsilon \leq \eta$ and $\gamma \leq \alpha_* < \xnorm$}, \\
    \bar\calC_{\phi_\beta}(\gamma, \eta) - \bar\calC_{\phi_\beta}(\xnorm, \eta) & \text{if $\xnorm > \gamma$ and $\epsilon \leq \eta$ and $\xnorm \leq \alpha_*$}. \\
  \end{cases}
\end{align*}
By taking the infimum over $\eta$ and $x$,
\begin{align*}
  \delta_\rho(\epsilon)
  = \inf_{\eta \in \left[\frac{1}{2}, 1\right]} \inf_{x \in \tcalX_\rho} \bar\delta(\epsilon, \eta, x)
  = \begin{cases}
    0 & \text{if $0 < \epsilon \leq \epsilon_0$,} \\
    4(1+\beta)^2(\epsilon-\epsilon_0)^2 & \text{if $\epsilon_0 < \epsilon$,}
  \end{cases}
\end{align*}
where $\epsilon_0 \define \frac{1+\beta+\gamma}{2(1+\beta)}$,
and $\delta_\rho^{**}(\epsilon) = \delta_\rho(\epsilon)$.

\section{Simulation Results}
\label{sec:simulation-results}

\subsection{Detail of Numerical Approximation of Bayes Risks}
\label{sec:numerical-approximation-of-bayes-risk}

In order to compute the Bayes $(\phi,\calFlin)$-risk for a loss $\phi$,
we substitute the Bayes $(\phi_\gamma,\calFlin)$-classifier $f_\gamma^* \in \calFlin$ into
\begin{align*}
  R_\phi(f_\gamma^*)
  &= \E[\phi(Yf_\gamma^*(X))] \\
  &= \int_X \Big\{ \phi(f_\gamma^*(X)) \P(Y=+1|X) + \phi(-f_\gamma^*(X)) \P(Y=-1|X) \Big\} \d\P(X)
\end{align*}
and apply numerical integration.
The partitioning quadrature method was used with grid size $\num{0.05}$.
The Bayes $(\phi_\gamma,\calFlin)$-classifier is $f_\gamma^*(x) = (x_1 + x_2)/\sqrt{2}$ for both twonorm and advnorm datasets.

To perform numerical integration, $\P(Y=+1|X)$ needs to be estimated.
Note that $\P(X)$ can be estimated given $\P(Y=+1|X)$.
For the advnorm dataset, we estimate $\P(X|Y=+1)$ with kernel density estimator and then compute $\P(Y=+1|X)$ and $\P(X)$.
Subsequently, we focus on the twonorm dataset and derive the closed-form expression of $\P(Y=+1|x)$.
Let $q_+$ and $q_-$ be probability density functions of $\calN([0.3\;0.3]^\top, 0.1^2I_2)$ and $\calN([-0.3\;-0.3], 0.1^2I_2)$, respectively.
Then,
\begin{align*}
  \P(Y=+1|X)
  \!=\! \frac{\P(Y=+1) \P(X|Y=+1)}{\P(Y=+1) \P(X|Y=+1) + \P(Y=-1) \P(X|Y=-1)}
  \!=\! \frac{\frac{1}{2}q_+(X)}{\frac{1}{2}q_+(X) + \frac{1}{2}q_-(X)}.
\end{align*}
The approximated Bayes risks are listed in Table~\ref{tab:approximated-bayes-risk}.

\begin{table}[t]
  \footnotesize
  \centering
  \caption{
    The approximated Bayes risks.
    For advnorm, we used kernel density estimator with RBF kernel (bandwidth: $\num{0.25}$) to estimate $\P(X|Y=+1)$.
  }
  \label{tab:approximated-bayes-risk}
  \begin{tabular}{lll}
    \toprule
    Loss & twonorm & advnorm \\
    \midrule
    Robust 0-1 & $\num{0.012}$ & $\num{0.067}$ \\
    Ramp       & $\num{0.389}$ & $\num{0.550}$ \\
    Sigmoid    & $\num{0.445}$ & $\num{0.525}$ \\
    Hinge      & $\num{0.778}$ & $\num{1.100}$ \\
    Logistic   & $\num{0.590}$ & $\num{0.750}$ \\
    \bottomrule
  \end{tabular}
\end{table}

\subsection{Full Simulation Results of Benchmark Dataset}
\label{sec:simulation-results-benchmark}

We show the full simulation results of MNIST in Tables~\ref{tab:full-target} and \ref{tab:full-0/1-loss}.
Simulation details are as follows.
\begin{itemize}
  \item Dataset: MNIST extracted with two digits ($\num{7000}$ instances for each digit).
  \item Preprocessing: Reduced to $2$-dimension with the principal component analysis.
  \item Train-test split: $\num{14000}$ instances are randomly split into training and test data with the ratio $4$ to $1$.
  \item Model: Linear models $f(x) = \theta^\top x + \theta_0$ ($\theta$ and $\theta_0$ are learnable parameters)
  \item Surrogate loss: The ramp, sigmoid, hinge, and logistic losses with shift $\beta = +0.5$.
  \item Target loss: the $\gamma$-adversarially robust 0-1 loss with $\gamma = 0.1$.
  \item Optimization: Batch gradient descent with $\num{1000}$ iterations.
\end{itemize}

\begin{table}[h]
\centering
\footnotesize
\caption{
  The simulation results of the $\gamma$-adversarially robust 0-1 loss with $\gamma = 0.1$ and $\beta = 0.5$.
  50 trials are conducted for each pair of a method and dataset.
  Standard errors (multiplied by $10^4$) are shown in parentheses.
  Bold-faces indicate outperforming methods, chosen by one-sided t-test with the significant level 5\%.
}
\label{tab:full-target}
\begin{tabular}{lllll}
  \toprule
  {} &                 Ramp &             Sigmoid &                Hinge &    Logistic \\
  \midrule
  0 vs 1 &            0.034 (3) &  \textbf{0.017 (2)} &           0.087 (12) &  0.321 (19) \\
  0 vs 2 &   \textbf{0.111 (7)} &          0.133 (10) &   \textbf{0.109 (8)} &  0.281 (19) \\
  0 vs 3 &   \textbf{0.107 (7)} &           0.126 (8) &            0.120 (9) &  0.307 (18) \\
  0 vs 4 &   \textbf{0.069 (6)} &          0.093 (12) &            0.072 (7) &  0.269 (21) \\
  0 vs 5 &  \textbf{0.233 (21)} &          0.340 (25) &  \textbf{0.233 (21)} &  0.269 (16) \\
  0 vs 6 &   \textbf{0.129 (8)} &          0.167 (13) &   \textbf{0.127 (8)} &  0.287 (22) \\
  0 vs 7 &   \textbf{0.067 (6)} &           0.073 (6) &            0.090 (9) &  0.302 (18) \\
  0 vs 8 &   \textbf{0.096 (7)} &          0.123 (12) &            0.100 (9) &  0.263 (20) \\
  0 vs 9 &   \textbf{0.082 (6)} &           0.101 (8) &            0.092 (8) &  0.279 (22) \\
  \bottomrule
\end{tabular}
\end{table}

\begin{table}[h]
\centering
\footnotesize
\caption{
  The simulation results of the 0-1 loss with $\beta = 0.5$.
  50 trials are conducted for each pair of a method and dataset.
  Standard errors (multiplied by $10^4$) are shown in parentheses.
  Bold-faces indicate outperforming methods, chosen by one-sided t-test with the significant level 5\%.
}
\label{tab:full-0/1-loss}
\begin{tabular}{lllll}
  \toprule
  {} &                 Ramp &             Sigmoid &                Hinge &    Logistic \\
  \midrule
  0 vs 1 &            0.012 (2) &  \textbf{0.005 (1)} &            0.038 (7) &  0.228 (18) \\
  0 vs 2 &   \textbf{0.050 (5)} &           0.059 (7) &            0.058 (7) &  0.206 (18) \\
  0 vs 3 &   \textbf{0.047 (4)} &           0.054 (6) &            0.064 (8) &  0.229 (15) \\
  0 vs 4 &   \textbf{0.028 (4)} &  \textbf{0.029 (4)} &            0.032 (6) &  0.184 (18) \\
  0 vs 5 &  \textbf{0.117 (11)} &          0.185 (20) &  \textbf{0.117 (11)} &  0.193 (15) \\
  0 vs 6 &   \textbf{0.060 (5)} &           0.080 (8) &            0.063 (6) &  0.206 (18) \\
  0 vs 7 &   \textbf{0.027 (3)} &  \textbf{0.027 (4)} &            0.045 (6) &  0.214 (18) \\
  0 vs 8 &   \textbf{0.050 (6)} &           0.054 (6) &            0.054 (7) &  0.186 (18) \\
  0 vs 9 &   \textbf{0.040 (4)} &           0.044 (5) &            0.046 (6) &  0.192 (20) \\
  \bottomrule
\end{tabular}
\end{table}

\section{Additional Plots}
\label{sec:additional-plots}

In this section, we put additional plots of the counterexample in Section~\ref{sec:calibrated-surrogate}.
The class-conditional risks of
\begin{align*}
  \phi(\alpha) = \frac{e^{-\alpha^2} + 1}{2} + \mathrm{clip}_{[-5,5]}(-\alpha) + \num{4.5}
\end{align*}
are plotted in Figure~\ref{fig:counterexample}.
From this figure, it is easy to see that $\calC_{\phi}(\cdot, \num{0.6})$ is not quasiconcave
while its even part $\calC_{\phi}(\cdot, \num{0.5})$ is quasiconcave.

\begin{figure}[h]
  \centering
  \includegraphics{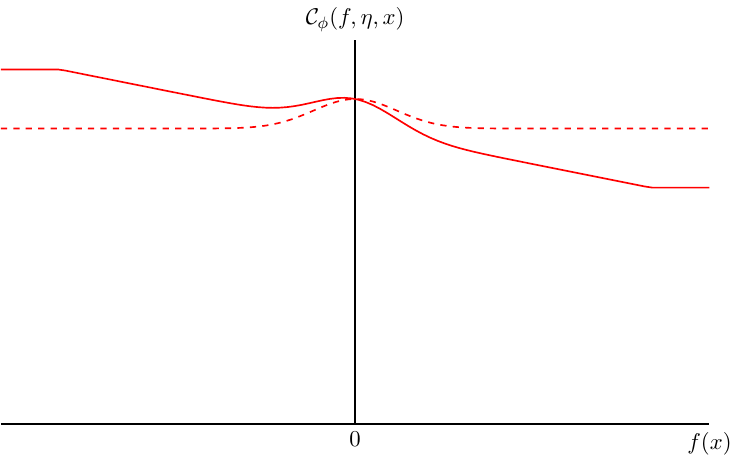}
  \caption{
    The class-conditional risks of $\phi(\alpha) = \frac{e^{-\alpha^2}+1}{2} + \mathrm{clip}_{[-5,5]}(-\alpha) + \num{4.5}$.
    The solid line is $\calC_{\phi}(\cdot, \num{0.6})$ and the dashed line is $\calC_{\phi}(\cdot, \num{0.5})$.
  }
  \label{fig:counterexample}
\end{figure}

\end{document}

%% file: preamble.tex
\usepackage{bbm}
\usepackage{bm}
\usepackage{booktabs}
\usepackage{braket}
\usepackage[font=footnotesize,labelfont=bf]{caption}
\usepackage{chngcntr}
\usepackage{color}
\usepackage{enumitem}
\usepackage[T1]{fontenc}
\usepackage{mathtools}
\usepackage{mathrsfs}
\usepackage{mdframed}
\usepackage{multirow}
\usepackage{pifont}
\usepackage{rotating}
\usepackage[group-separator={,},group-minimum-digits={3}]{siunitx}
\usepackage{stmaryrd}
\usepackage{tikz}
\usepackage{wrapfig}

\usetikzlibrary{external}
\usetikzlibrary{patterns}

\newtheorem{assumption}{Assumption}
\newtheorem*{assumption*}{Assumption}

\DontPrintSemicolon
\SetArgSty{textnormal}
\SetKwInOut{Input}{Input}
\SetKwInOut{Output}{Output}
\SetKwComment{Comment}{$\triangleright$\ }{}
\SetKwRepeat{Do}{do}{while}

\renewcommand{\tilde}{\widetilde}

\renewcommand{\bar}{\overline}
\def\E{\mathbb{E}}
\def\P{\mathbb{P}}
\def\R{\mathbb{R}}
\def\N{\mathbb{N}}

\def\calA{\mathcal{A}}

\def\calC{\mathcal{C}}

\def\calF{\mathcal{F}}

\def\calL{\mathcal{L}}

\def\calN{\mathcal{N}}
\def\calR{\mathcal{R}}

\def\calU{\mathcal{U}}

\def\calX{\mathcal{X}}
\def\calY{\mathcal{Y}}

\newcommand{\define}{\stackrel{\scriptscriptstyle \mathrm{def}}{=}}

\newcommand{\sign}{\mathrm{sign}}

\newcommand{\ind}[1]{\mathbbm{1}_{\left\{#1\right\}}}

\renewcommand{\d}{\mathrm{d}}

\newcommand{\dom}{\boldsymbol{\mathrm{dom}}}

\renewcommand{\epsilon}{\varepsilon}
\def\calFlin{\calF_{\mathrm{lin}}}

\def\calFall{\calF_{\mathrm{all}}}
\newcommand{\header}[1]{\noindent \textbf{#1:}}
\newcommand{\epi}{\mathop{\mathrm{epi}}}
\newcommand{\conv}{\mathop{\mathrm{co}}}
\newcommand{\closconv}{\mathop{\overline{\mathrm{co}}}}
\def\xnorm{\|x\|_2}
\def\tcalX{\tilde\calX}

\SetSymbolFont{stmry}{bold}{U}{stmry}{m}{n}